\documentclass[journal]{IEEEtran}

\usepackage{algorithm,algorithmic,amsmath,amssymb,amsthm,amsfonts}
\usepackage{bm,color,subfig,dsfont,pifont,hyperref}
\usepackage{graphicx}
\usepackage{array}
\usepackage{comment}
\usepackage{enumitem}

\providecommand{\Vector}[1]{\boldsymbol{#1}}
\providecommand{\Matrix}[1]{\boldsymbol{\mathbf{#1}}}
\providecommand{\fnorm}[1]{\left\Vert#1\right\Vert_{\mathrm{F}}}
\providecommand{\twoinfnorm}[1]{\left\|#1\right\|_{2,\infty}}
\providecommand{\twonorm}[1]{\left\|#1\right\|_{2}}
\providecommand{\opnorm}[1]{\left\|#1\right\|_{\mathrm{op}}}
\providecommand{\freg}[0]{f_{\text{reg}}}
\providecommand{\fbal}[0]{f_{\text{bal}}}
\providecommand{\dist}[2]{\mathrm{dist}(#1, #2)}

\newtheorem{assumption}{Assumption}
\newtheorem{hypothesis}{Hypothesis}

\newtheorem{lemma}{Lemma}
\newtheorem{theorem}{Theorem}

\newcommand{\norm}[1]{\left\lVert#1\right\rVert}

\begin{document}

\title{Global Convergence Analysis of Vanilla Gradient Descent for Asymmetric Matrix Completion}

\author{Xu Zhang,~\IEEEmembership{Member,~IEEE,} Shuo Chen, Jinsheng Li, Xiangying Pang, Maoguo Gong,~\IEEEmembership{Fellow,~IEEE}
\thanks{This work was supported by the Postdoctoral Fellowship Program of CPSF under Grant No. GZC20232038 and the China Postdoctoral Science Foundation under Grant No. 2024M762521. \textit{(Corresponding author: Xiangying Pang.)}}
\thanks{X.~Zhang is with School of Artificial Intelligence, Xidian University, Xi'an 710126, China (e-mail: zhang.xu@xidian.edu.cn).}
\thanks{S.~Chen is with Department of Architecture and Design, Huawei Cloud, Hangzhou 310051, China (e-mail: chenshuo51@huawei.com).}
\thanks{J.~Li is with the Future Technology Research Center, China Telecom Research
 Institute, Beijing 102209, China (e-mail: lijs45@chinatelecom.cn). }
\thanks{X.~Pang is with Department of Mathematics, The Chinese University of Hong Kong, Hong Kong SAR of China (e-mail: xypang@math.cuhk.edu.hk).}
\thanks{M.~Gong is with the Key Laboratory of Collaborative Intelligence Systems, Ministry of Education, School of Electronic Engineering, Xidian University, Xi'an, China, and the Academy of Artificial Intelligence, College of Mathematics Science, Inner Mongolia Normal University, Hohhot, China (e-mail: mggong@mail.xidian.edu.cn).}

}
\markboth{Journal of \LaTeX\ Class Files,~Vol.~14, No.~8, August~2021}%
{Shell \MakeLowercase{\textit{et al.}}: A Sample Article Using IEEEtran.cls for IEEE Journals}


\maketitle

\begin{abstract}
This paper investigates the asymmetric low-rank matrix completion problem, which can be formulated as an unconstrained non-convex optimization problem with a nonlinear least-squares objective function, and is solved via gradient descent methods. Previous gradient descent approaches typically incorporate regularization terms into the objective function to guarantee convergence. However, numerical experiments and theoretical analysis of the gradient flow both demonstrate that the elimination of regularization terms in gradient descent algorithms does not adversely affect convergence performance. By introducing the leave-one-out technique, we inductively prove that the vanilla gradient descent with spectral initialization achieves a linear convergence rate with high probability. Besides, we demonstrate that the balancing regularization term exhibits a small norm during iterations, which reveals the implicit regularization property of gradient descent. Empirical results show that our algorithm has a lower computational cost while maintaining comparable completion performance compared to other gradient descent algorithms.
\end{abstract}


\section{Introduction}

Low-rank matrix completion focuses on how to recover the remaining unknown elements of a matrix based on its partial elements under the low-rank assumption \cite{candes2010power,candes2012exact}, which is widely used in applications such as recommender systems \cite{ramlatchan2018survey, chen2022review}, image inpainting \cite{xue2017depth,cai2024restoration}, and network localization \cite{xiao2017noise, kim2021deep}. Specifically, given a target matrix $\Vector{M}_{\star}\in\mathbb{R}^{d_{1}\times d_{2}}$ with rank $r$, only partial elements $\mathcal{P}_{\Omega}(\Vector{M}_{\star})$ are observed, where $r\ll \min\{d_{1}, d_{2}\}$, $\Omega\subset [d_{1}]\times[d_{2}]$ denote the set of observable elements and $\mathcal{P}_{\Omega}(\cdot)$ is a projection operator defined as
\begin{equation}
\label{eqn-projection-defn}
\left[\mathcal{P}_{\Omega}\left(\Vector{M}_{\star}\right)\right]_{ij} \triangleq \begin{cases}
\left[\Vector{M}_{\star}\right]_{ij}, & (i, j)\in\Omega, \\
0, & (i, j)\not\in\Omega.
\end{cases}
\end{equation}
The goal of matrix completion is to recover $\Vector{M}_{\star}$ from the partial measurements $\mathcal{P}_{\Omega}\left(\Vector{M}_{\star}\right)$. 

Suppose that the rank of the target matrix $\Vector{M}_{\star}$ is known beforehand, then $\Vector{M}_{\star}$ can be decomposed into the product of two low-rank matrices, and can be modeled as a non-linear least-squares problem
\begin{equation}
\label{eqn-matrix-completion}
\min_{\Vector{X},\Vector{Y}} 
f(\Vector{X}, \Vector{Y})\triangleq \frac{1}{2p}\norm{\mathcal{P}_{\Omega}\left(\Vector{X}\Vector{Y}^\top - \Vector{M}_{\star}\right)}^{2}_{\mathrm{F}},
\end{equation}
where $\Vector{X}\in\mathbb{R}^{d_{1}\times r}$, $\Vector{Y}\in\mathbb{R}^{d_{2}\times r}$, and $p$ denotes the sampling probability. Considering $r\ll\min\{d_{1}, d_{2}\}$, this model significantly alleviates the computational difficulty by reducing the number of variables from $d_{1}\times d_{2}$ to $r\times(d_{1} + d_{2})$. 

The non-convexity of the model prevents us from guaranteeing that the iterative sequence $\left\{\Vector{X}_{k}\Vector{Y}_{k}^\top\right\}_{k=0}^{+\infty}$ converges to $\Vector{M}_{\star}$. During the iterative process, there might be an ill-conditioned situation where the magnitudes of $\Vector{X}_{k}$ and $\Vector{Y}_{k}$ are asymmetric, i.e., the norm of one is too large while the norm of the other is too small. This asymmetry might harm the convergence of the algorithm. To ensure convergence, regularization terms are introduced to prevent $\Vector{X}$ and $\Vector{Y}$ from differing significantly in the sense of norms \cite{raghunandan2010matrix}. A common regularization term is $\fnorm{\Vector{X}}^{2}+\fnorm{\Vector{Y}}^{2}$ \cite{koren2009matrix,chen2015fast, sun2016guaranteed, chen2020noisy}, and the related problem becomes
\begin{multline}
\label{eqn-mc-regularization}
\min_{\Vector{X}, \Vector{Y}}~\freg(\Vector{X}, \Vector{Y}) = \frac{1}{2p}\fnorm{\mathcal{P}_{\Omega}\left(\Vector{X}\Vector{Y}^\top- \Vector{M}_{\star}\right)}^{2} \\ + \frac{\lambda}{2} \big(\fnorm{\Vector{X}}^{2} + \fnorm{\Vector{Y}}^{2}\big),
\end{multline}
where $\lambda>0$ is a regularization parameter.

Another common regularization term is the balancing term $f_{\mathrm{diff}}(\Vector{X}, \Vector{Y}) = \fnorm{\Vector{X}^\top\Vector{X} - \Vector{Y}^\top\Vector{Y}}^{2}$  \cite{chen2020nonconvex}. The idea is also very intuitive: when the norms of $\Vector{X}$ and $\Vector{Y}$ differ significantly, the value of the balancing term will increase, thus acting as a penalty function. After introducing the balancing term, the problem becomes
\begin{multline}
    \label{eqn-mc-balancing}
\min_{\Vector{X}, \Vector{Y}} \fbal(\Vector{X}, \Vector{Y}) \triangleq \frac{1}{2p}\fnorm{\mathcal{P}_{\Omega}\left(\Vector{X}\Vector{Y}^\top- \Vector{M}_{\star}\right)}^{2} \\ + \frac{1}{8}\fnorm{\Vector{X}^\top\Vector{X} - \Vector{Y}^\top\Vector{Y}}^{2}.
\end{multline}

\subsection{Motivations}

The incorporation of regularization terms inherently increases the computational cost of gradient computation while simultaneously introducing additional hyperparameters that require careful tuning. However, numerical experiments in Fig. \ref{fig: matrix-completion-via-gd} show that the elimination of regularization terms does not adversely affect the convergence speed of the gradient descent (GD) algorithm under spectral initialization. In particular, we compare the convergence rates of vanilla GD (VGD) for problem \eqref{eqn-matrix-completion},  regularized GD (RGD) for problem \eqref{eqn-mc-regularization}, and balancing GD (BGD) for problem \eqref{eqn-mc-balancing} in Fig. \ref{fig: matrix-completion-via-gd}. Two randomly generated target matrices $\Vector{M}_{\star}\in\mathbb{R}^{1200\times 800}$ have a rank of $10$, and the condition number $\kappa$ is 1 and 3, respectively. The sampling probability is $p=0.2$, the step size is $s=0.5$, and $\lambda$ in problem \eqref{eqn-mc-regularization} is chosen in $\{10^{-3}, 10^{-6}, 10^{-10}\}$.  It can be observed that VGD and BGD converge almost identically, with linear convergence rates. As for RGD, the convergence curves settle into some fixed errors, and the smaller the parameter $\lambda$, the lower the fixed error. This also confirms that the regularization term is not necessary for asymmetric matrix completion.

\begin{figure}[htb]
\centering
 \subfloat[$\kappa=1$]{
     \includegraphics[scale=0.32]{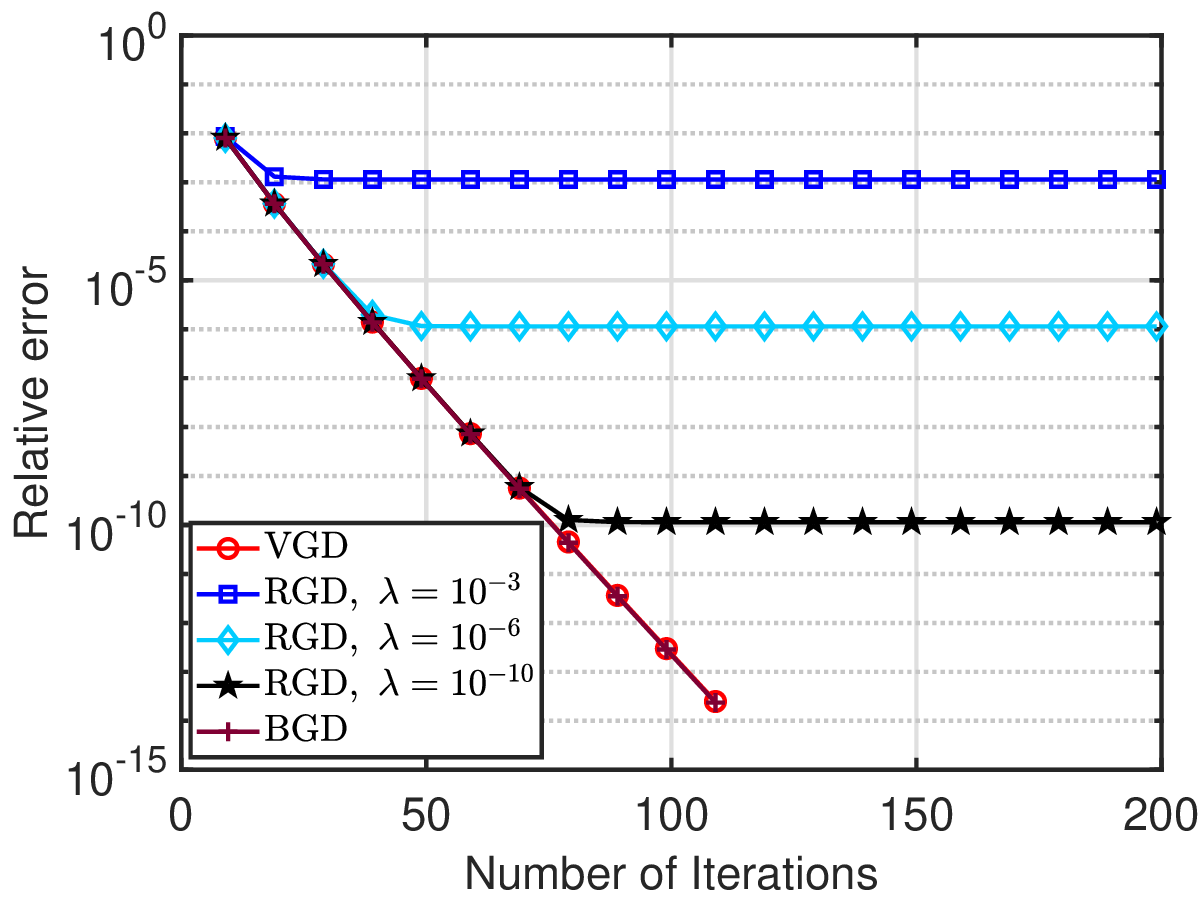}}
    \hfill
\subfloat[$\kappa=3$]{
    \includegraphics[scale=0.32]{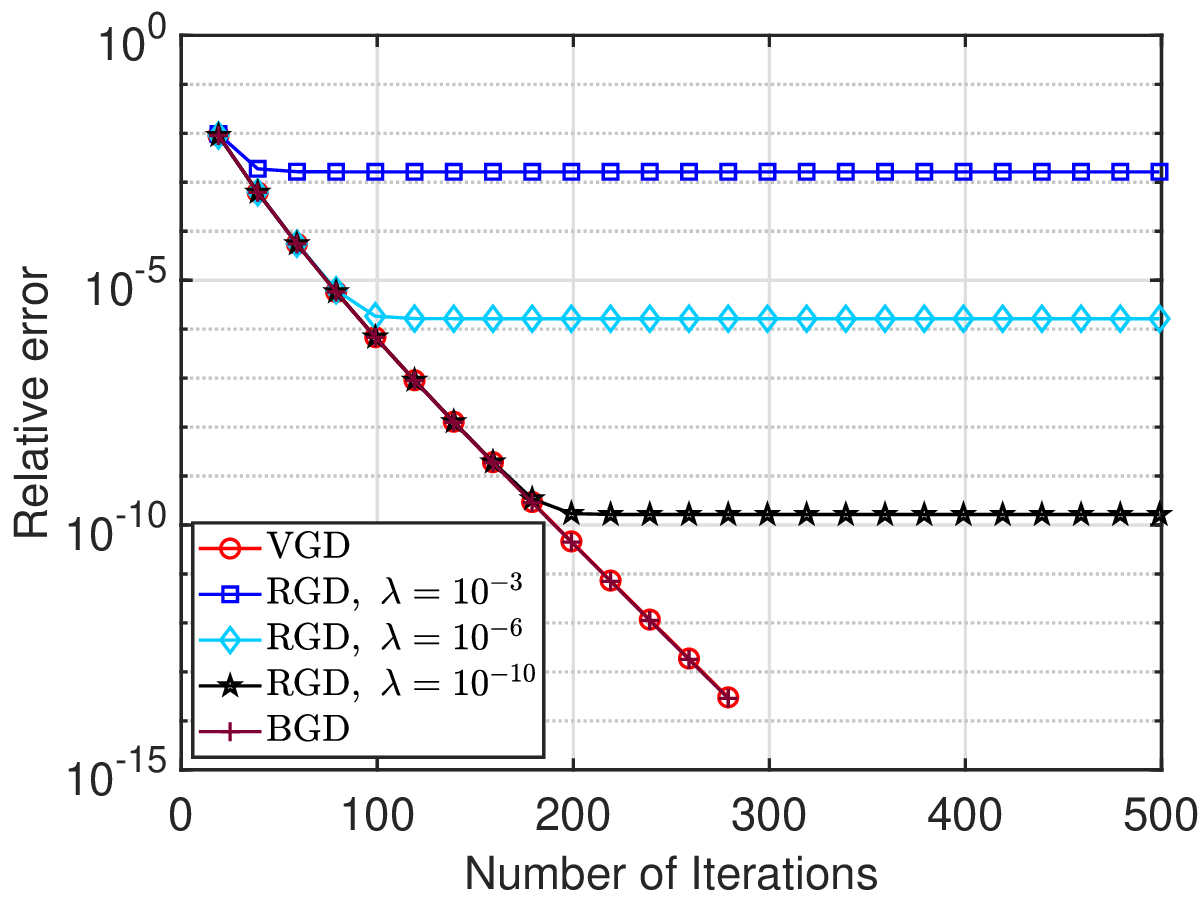}}
\caption{Convergence results of VGD for \eqref{eqn-matrix-completion}, RGD for \eqref{eqn-mc-regularization} and BGD for \eqref{eqn-mc-balancing} under $d_1=1200$, $d_2=800$, $r=10$ and $p=0.2$.}
\label{fig: matrix-completion-via-gd}
\end{figure}

The above numerical results demonstrate that eliminating the balancing term preserves convergence performance. Next, we further validate this finding through differential equation analysis. The gradient flow corresponding to the gradient method of problem \eqref{eqn-matrix-completion} (c.f. \eqref{eqn-gd-mc_1} and \eqref{eqn-gd-mc_2}) is
\begin{equation}
\label{eqn-gradient-flow-mc}
\begin{cases}
\Dot{\Vector{X}}(t) = - \frac{1}{p}\mathcal{P}_{\Omega}\left(\Vector{X}(t)\Vector{Y}(t)^\top- \Vector{M}_{\star}\right)\Vector{Y}(t), \\
\Dot{\Vector{Y}}(t) = - \frac{1}{p}\mathcal{P}_{\Omega}\left(\Vector{X}(t)\Vector{Y}(t)^\top- \Vector{M}_{\star}\right)^\top\Vector{X}(t).
\end{cases}
\end{equation}

Denote the solutions of Eq. \eqref{eqn-gradient-flow-mc} as $\Vector{X} = \Vector{X}(t)$, $\Vector{Y} = \Vector{Y}(t)$. Then we calculate the derivative of the balancing term $f_{\text{diff}}(\Vector{X}, \Vector{Y})$ with respect to time $t$
\begin{multline}
\frac{\text{d}}{\text{d}t}f_{\text{diff}}(\Vector{X}, \Vector{Y}) =  \frac{1}{2}\left\langle\Vector{X}\left(\Vector{X}^\top\Vector{X} -  \Vector{Y}^\top\Vector{Y}\right), \Dot{\Vector{X}}\right\rangle \\- \frac{1}{2} \left\langle\Vector{Y}\left(\Vector{X}^\top\Vector{X} - \Vector{Y}^\top\Vector{Y}\right), \Dot{\Vector{Y}}\right\rangle.
\end{multline}
Notice that
\begin{align*}
&\left\langle\Vector{X}\left(\Vector{X}^\top\Vector{X} - \Vector{Y}^\top\Vector{Y}\right), \Dot{\Vector{X}}\right\rangle \\
= & \left\langle \Vector{X}\left(\Vector{X}^\top\Vector{X} - \Vector{Y}^\top\Vector{Y}\right), - \frac{1}{p}\mathcal{P}_{\Omega}\left(\Vector{X}\Vector{Y}^\top- \Vector{M}_{\star}\right)\Vector{Y}\right\rangle \\
= & \left\langle \Vector{X}\left(\Vector{X}^\top\Vector{X} - \Vector{Y}^\top\Vector{Y}\right)\Vector{Y}^\top, - \frac{1}{p}\mathcal{P}_{\Omega}\left(\Vector{X}\Vector{Y}^\top- \Vector{M}_{\star}\right)\right\rangle, 
\end{align*}
and
\begin{align*}
&\left\langle\Vector{Y}\left(\Vector{X}^\top\Vector{X} - \Vector{Y}^\top\Vector{Y}\right), \Dot{\Vector{Y}}\right\rangle \\
= & \left\langle \Vector{Y}\left(\Vector{X}^\top\Vector{X} - \Vector{Y}^\top\Vector{Y}\right)\Vector{X}^\top, - \frac{1}{p}\mathcal{P}_{\Omega}\left(\Vector{X}\Vector{Y}^\top- \Vector{M}_{\star}\right)^\top\right\rangle \\
= & \left\langle \Vector{X}\left(\Vector{X}^\top\Vector{X} - \Vector{Y}^\top\Vector{Y}\right)\Vector{Y}^\top, - \frac{1}{p}\mathcal{P}_{\Omega}\left(\Vector{X}\Vector{Y}^\top- \Vector{M}_{\star}\right)\right\rangle,
\end{align*}
where means 
\begin{equation}
\frac{\text{d}}{\text{d}t}f_{\text{diff}}(\Vector{X}, \Vector{Y}) =  0.
\end{equation}
This indicates that in the continuous sense, the balancing term is a constant, and thus it does not affect the convergence of the solution.

\subsection{Contributions}
This paper studies vanilla gradient descent for low-rank asymmetric matrix completion. Our contributions are twofold:
\begin{itemize}
    \item[1)] This paper establishes the theoretical analysis for the linear convergence rate of the vanilla gradient descent method based on spectral initialization. This result provides the first convergence rate result for the asymmetric matrix completion problem without regularization terms, which concludes the theoretical framework of the equivalence between regularized and non-regularized matrix recovery problems.
    \item[2)] This paper reveals the implicit regularization property of the vanilla gradient descent with spectral initialization. By introducing an auxiliary leave-one-out completion problem and its corresponding sequence, theoretical analysis demonstrates that the norm of the balancing term remains small during the iterative process, thereby demonstrating that gradient descent exhibits implicit regularization properties.
\end{itemize}

\subsection{Related Work}
Matrix completion is a fundamental subclass of matrix recovery problems \cite{davenport2016overview}, which has been widely studied over the past two decades due to its ability to exploit low-dimensional structure in high-dimensional data. The seminal work of Candès and Recht established nuclear norm minimization (NNM) as a convex surrogate for rank minimization, which guarantees exact recovery under uniform sampling and incoherence conditions \cite{candes2012exact,recht2011simpler}. Despite its theoretical elegance, NNM suffers from computational intractability in large-scale applications, rendering it impractical for modern datasets with millions of rows and columns. To overcome these limitations, researchers turned to non-convex matrix factorization methods, which reduce storage and enable gradient-based optimization.

Early non-convex approaches relied on explicit regularizers to ensure identifiability and control parameter norms, e.g., the regularization term in problem \eqref{eqn-mc-regularization} and the balancing term in problem \eqref{eqn-mc-balancing}. Jain et al. \cite{jain2013low} provided convergence guarantees for alternating minimization with a penalty on $\ell_2$ row norm. Sun and Luo \cite{sun2016guaranteed} demonstrated that RGD for regularized objectives in problem \eqref{eqn-mc-regularization} avoids spurious local minima, and Chen et al. \cite{chen2020noisy} analyzed the statistical guarantees for RGD of problem \eqref{eqn-mc-regularization} in the noisy case.  Nie et al.\cite{nie2018matrix} employed a parameter-free logarithmic regularizer and proposed an efficient reweighted optimization algorithm with a convergence guarantee. Chen et al. \cite{chen2020nonconvex} established the sampling rate requirements for problem \eqref{eqn-mc-balancing} by using BGD with spectral initialization. 

A growing body of research questions the necessity of explicit regularization in matrix recovery problems. For symmetric positive semidefinite matrix completion, Ma et al. \cite{ma2020implicit} demonstrated that VGD with spectral initialization converges to the global optimality without regularization, while Ma and Fattahi \cite{ma2024convergence} proved that VGD with small initialization converges globally without any explicit regularization, even in overparameterized cases. For asymmetric matrices, global convergence without regularization terms was established only in mantrix factorization with fully observed settings or matrix sensing with restricted isometry property (RIP) measurements. In particular, Ye and Du \cite{ye2021global} presented that VGD with small initialization converges globally for asymmetric low-rank mantrix factorization without regularization terms on a fully observed matrix. Ma et al. \cite{ma20121beyond} showed that VGD with spectral initialization converges linearly to the optimality in matrix sensing with RIP assumptions. Soltanolkotabi et al. \cite{soltanolkotabi2025implicit} establish linear convergence for implicit balancing and regularization in overparameterized asymmetric matrix sensing.
However, asymmetric matrix completion without regularization terms remains challenging. The sparse sampling operator $\mathcal{P}_{\Omega}$ violates RIP, which weakens concentration bounds and necessitates incoherence condition. Besides, the norms of $\Vector{X}$ and $\Vector{Y}$ can diverge without regularization, and the sparse sampling might exacerbate the imbalance. 

\subsection{Organization}
The remainder of this paper is organized as follows. Section \ref{sec:algorithm} presents a vanilla gradient descent algorithm tailored for asymmetric matrix completion. Section \ref{sec:convergence} establishes global convergence guarantees for the proposed algorithm and provides a proof roadmap to elucidate key technical insights. Section \ref{sec:simulations} makes simulations to validate our theoretical results and Section \ref{sec:conclusion} provides the conclusion.

\section{Algorithms} \label{sec:algorithm}
This section introduces the gradient descent algorithm for the asymmetric matrix completion problem \eqref{eqn-matrix-completion}.

First of all, we leverage the spectral initialization method to initialize the iteration sequence. Denote the truncated rank-$r$ singular value decomposition (SVD) of $\frac{1}{p}\mathcal{P}_{\Omega}(\Vector{M}_{\star})$ as 
\begin{equation}
    \mathcal{T}_r\left(\frac{1}{p}\mathcal{P}_{\Omega}(\Vector{M}_{\star})\right)=\Vector{U}_{0}\Vector{\Sigma}_{0}\Vector{V}_{0}^\top,
\end{equation}
where $\Vector{U}_{0}\in\mathbb{R}^{d_{1}\times r}$ and $\Vector{V}_{0}\in\mathbb{R}^{d_{2}\times r}$ are orthonormal matrices, and $\Vector{\Sigma}_{0}\in\mathbb{R}^{r\times r}$ is a diagonal matrix. We initialize the iteration sequence as follows
\begin{equation}
    \Vector{X}_{0} = \Vector{U}_{0}\Vector{\Sigma}_{0}^{1/2}, \quad \Vector{Y}_{0} = \Vector{V}_{0}\Vector{\Sigma}_{0}^{1/2},
\end{equation}

Next. we explore the use of the gradient descent method to solve this problem in a parallel manner. The gradient of $f(\Vector{X}, \Vector{Y})$ is
\begin{align}
\nabla_{\Vector{X}}f(\Vector{X}, \Vector{Y}) = & \frac{1}{p}\mathcal{P}_{\Omega}\left(\Vector{X}\Vector{Y}^\top - \Vector{M}_{\star}\right)\Vector{Y}, \\
\nabla_{\Vector{Y}}f(\Vector{X}, \Vector{Y}) = & \frac{1}{p}\mathcal{P}_{\Omega}\left(\Vector{X}\Vector{Y}^\top - \Vector{M}_{\star}\right)^\top\Vector{X}.
\end{align}
Therefore, the update rule of the gradient descent method is
\begin{align}
\label{eqn-gd-mc_1}
\Vector{X}_{k+1} &= \Vector{X}_{k} - \frac{s}{p}\mathcal{P}_{\Omega}\left(\Vector{X}_{k}\Vector{Y}_{k}^\top - \Vector{M}_{\star}\right)\Vector{Y}_{k}, \\
\Vector{Y}_{k+1} &= \Vector{Y}_{k} - \frac{s}{p}\mathcal{P}_{\Omega}\left(\Vector{X}_{k}\Vector{Y}_{k}^\top - \Vector{M}_{\star}\right)^\top\Vector{X}_{k},
\label{eqn-gd-mc_2}
\end{align}
where $s>0$ denotes the step size. We summarize the above process in Algorithm \ref{alg:alg1}, where $K$ denotes the largest number of iterations.

\begin{algorithm}[H]
\caption{Vanilla Gradient Descent (VGD) for Asymmetric Matrix Completion} \label{alg:alg1}
\begin{algorithmic}
\STATE \textbf{Initialization:} $\Vector{U}_{0}\Vector{\Sigma}_{0}\Vector{V}_{0}^\top =\mathcal{T}_r (\frac{1}{p}\mathcal{P}_{\Omega}(\Vector{M}_{\star}))$, $\Vector{X}_{0} = \Vector{U}_{0}\Vector{\Sigma}_{0}^{1/2}, \Vector{Y}_{0} = \Vector{V}_{0}\Vector{\Sigma}_{0}^{1/2}$ 
\FOR{$k=0,\ldots,K-1$}
\STATE $\Vector{X}_{k+1} = \Vector{X}_{k} - \frac{s}{p}\mathcal{P}_{\Omega}\left(\Vector{X}_{k}\Vector{Y}_{k}^\top - \Vector{M}_{\star}\right)\Vector{Y}_{k}$
\STATE $\Vector{Y}_{k+1} = \Vector{Y}_{k} - \frac{s}{p}\mathcal{P}_{\Omega}\left(\Vector{X}_{k}\Vector{Y}_{k}^\top - \Vector{M}_{\star}\right)^\top\Vector{X}_{k}$
\ENDFOR
\STATE \textbf{Ouuput:} $\Vector{M}_K=\Vector{X}_{K}\Vector{Y}_{K}^\top$
\end{algorithmic}
\label{alg1}
\end{algorithm}

\section{Convergence Guarantees} \label{sec:convergence}
This section provides the convergence rate of Algorithm \ref{alg:alg1}. Before that, we first provide some important definitions and assumptions. 
Let $\sigma_{\max}$ be the largest singular value of $\Vector{M}_{\star}$ and $\sigma_{\min}$ be the smallest non-zero singular value. The condition number is denoted as $\kappa\triangleq\sigma_{\max}/\sigma_{\min}$. 

Assume that the sampling set $\Omega$ is generated by independent Bernoulli sampling.
\begin{assumption}[Bernoulli Sampling] \label{assumption-iid}
    For any $i\in[d_{1}]$ and $j\in [d_{2}]$, the element $[\Vector{M}_{\star}]_{ij}$ is observed with probability $p$, where and $0 < p \leq 1$.
\end{assumption}

To prevent the nonzero elements of $\Vector{M}_{\star}$ from being concentrated in a few positions, it is necessary to introduce the assumption of the $\mu$-incoherence property of $\Vector{M}_{\star}$.
\begin{assumption}[Incoherence Condition, \cite{candes2012exact}]
\label{assumption-incoherence}
Let the SVD of $\Vector{M}_{\star}$ be $\Vector{M}_{\star} = \Vector{U}_{\star}\Vector{\Sigma}_{\star}\Vector{V}_{\star}^\top$, where $\Vector{U}_{\star}\in\mathbb{R}^{d_{1}\times r}$ and $\Vector{V}_{\star}\in\mathbb{R}^{d_{2}\times r}$ are orthonormal matrices, and $\Vector{\Sigma}_{\star}\in\mathbb{R}^{r\times r}$ is a diagonal matrix. If $\Vector{U}_{\star}$ and $\Vector{V}_{\star}$ satisfy
\begin{equation}
\left\|\Vector{U}_{\star}\right\|_{2,\infty} \leq \sqrt{\frac{\mu r}{d_{1}}}~, \quad \left\|\Vector{V}_{\star}\right\|_{2,\infty} \leq \sqrt{\frac{\mu r}{d_{2}}}~,
\end{equation}
then $\Vector{M}_{\star}$ is $\mu$-incoherent, where
$\|\Vector{A}\|_{2,\infty}$ the largest $\ell_{2}$-norm of all the rows of $\Vector{A}$.
\end{assumption}

It is worth noting that, under Assumption \ref{assumption-incoherence}, we have $\mu\geq1$. Otherwise, we have
\begin{equation}
\fnorm{\Vector{U}_{\star}}^{2} \leq d_{1}\twoinfnorm{\Vector{U}_{\star}}^{2} \leq \mu r < r,
\end{equation}
which contradicts the fact that $\Vector{U}_{\star}$ is an orthogonal matrix.

In addition, if Assumptions \ref{assumption-iid} and \ref{assumption-incoherence} hold, the projection operator $p^{-1}\mathcal{P}_{\Omega}$ satisfies the RIP to some extent, that is, its behavior is close to that of the identity operator $\mathcal{I}$ from $\mathbb{R}^{d_{1}\times d_{2}}$ to $\mathbb{R}^{d_{1}\times d_{2}}$, which makes it possible to complete the matrix for undersampled elements. Please refer to Lemmas \ref{lemma-rip-subspace} and \ref{lemma-rip-all-space} in the Supplementary Material for more information.

Define $\Vector{F}_k=[\Vector{X}_k^\top, \Vector{Y}_k^\top]^\top$ and the optimal solution as
\begin{align}
    \Vector{F}_{\star} \triangleq \begin{bmatrix}
 \Vector{X}_{\star} \\
  \Vector{Y}_{\star}
\end{bmatrix} = \begin{bmatrix}
 \Vector{U}_{\star}\Vector{\Sigma}_{\star}^{1/2} \\
  \Vector{V}_{\star}\Vector{\Sigma}_{\star}^{1/2}
\end{bmatrix}\in\mathbb{R}^{(d_{1} + d_{2})\times r},
\end{align}
where $\Vector{X}_{\star}=\Vector{U}_{\star}\Vector{\Sigma}_{\star}^{1/2}$, $\Vector{Y}_{\star}=\Vector{V}_{\star}\Vector{\Sigma}_{\star}^{1/2}$. Note that the above term is an optimal solution to problem \eqref{eqn-matrix-completion}. However, due to the non-uniqueness of optimal solutions, we formally define the distance between $\Vector{F}_k$ and $\Vector{F}_{\star}$ as follows
\begin{multline}
\label{eqn-defn-distance}
\dist{\Vector{F}_{k}}{\Vector{F}_{\star}} \triangleq \\ \sqrt{\inf_{\Vector{Q}\in\text{GL}(r)}\left(\fnorm{\Vector{X}_{k}\Vector{Q} - \Vector{X}_{\star}}^{2} + \fnorm{\Vector{Y}_{k}\Vector{Q}^{-\rm{T}} - \Vector{Y}_{\star}}^{2}\right)},
\end{multline}
where $\text{GL}(r)=\{Q\in\mathbb{R}^{r\times r}: Q \text{ is invertible}\}$ is general linear group of $r$ degree.

Based on the distance metric \eqref{eqn-defn-distance}, we present the main theorem.
\begin{theorem}
\label{thm-linear-convergence}
Suppose that $\Vector{M}_{\star}$ is $\mu$-incoherent. If the sampling rate $p$ and the step size $s$ satisfy
\begin{equation}
\label{eqn-assumption-p-s}
\begin{split}
& p\geq \frac{C_{3}\mu^{3}r^{3}\kappa^{16}\max\{d_{1}, d_{2}\}\log \left(\max\{d_{1}, d_{2}\}\right)}{\min\{d_{1}, d_{2}\}^{2}} \\
& 0<s\leq\frac{\min\{d_{1}, d_{2}\}}{C_{4}\max\{d_{1}, d_{2}\}^{3/2}\sqrt{\mu r}\kappa^{4}\sigma_{\max}}
\end{split}
\end{equation}
for some constants $C_{3}, C_{4}>0$, then for $0\leq k\leq K\triangleq (d_{1} + d_{2})^{4}$, the iteration sequences  $\{\Vector{F}_{k}\}_{k=0}^{K}$ of Algorithm  \ref{alg:alg1} satisfy the following inequality with probability no less than $1 - (d_{1} + d_{2})^{-5}$:
\begin{equation}
\label{eqn-thm-linear-convergence}
\dist{\Vector{F}_{k}}{\Vector{F}_{\star}} \leq \left(1 - \frac{s\sigma_{\min}}{100}\right)^{k}\dist{\Vector{F}_{0}}{\Vector{F}_{\star}}.
\end{equation}
\end{theorem}

Theorem \ref{thm-linear-convergence} demonstrates that the gradient descent method with spectral initialization in Algorithm \ref{alg:alg1} for solving asymmetric matrix completion problems is linearly convergent with high probability. Notably, the step size condition reveals that the convergence rate becomes slower as the condition number $\kappa$ increases, which aligns with numerical results in Fig. \ref{fig: matrix-completion-via-gd}. To the best of our knowledge, this constitutes the first convergence rate result for the vanilla gradient descent algorithm of asymmetric matrix completion.

The theorem extends four key prior works in the following way:
\begin{itemize}
    \item[1)] Building upon the linear convergence results for vanilla gradient descent in symmetric matrix completion \cite{ma2020implicit} and asymmetric matrix sensing \cite{ma20121beyond}, we extend these theoretical guarantees to the asymmetric matrix completion setting. This generalization encompasses both rectangular matrix structures and structured sampling operators.

    \item[2)] The linear convergence guarantees for regularized gradient descent in the regularized model \eqref{eqn-mc-regularization} \cite{sun2016guaranteed,chen2020noisy} and the balancing model \eqref{eqn-mc-balancing} \cite{chen2020nonconvex} are further extended to the regularization-free model \eqref{eqn-matrix-completion}. Besides, we demonstrate the implicit regularization effect of VGD by rigorously establishing that the norm of the balancing term maintains a bounded magnitude throughout the iterative process.
\end{itemize}
This result finalizes the theoretical bridge between regularization-based and regularization-free formulations in low-rank matrix recovery.

\subsection{Proof Roadmap}

This subsection outlines the proof roadmap for Theorem \ref{thm-linear-convergence}, primarily employing the leave-one-out technique and mathematical induction. The full proof is delayed in the Appendices.

\textbf{Leave-one-out technique.} To employ the leave-one-out technique, we first define the following projection operators
\begin{itemize}
\item $\mathcal{P}_{\Omega_{-i,\cdot}}$ represents the projection operator that removes all elements in $\Omega$ whose row indices are $i$;
\item $\mathcal{P}_{i,\cdot}$ represents the projection operator that only preserves the elements in the $i$-th row of the matrix.
\end{itemize}

Building on these definitions, we define the leave-one-out matrix completion problem corresponding to problem \eqref{eqn-matrix-completion}. When $0\leq l\leq d_{1}$, the problem is 
\begin{align}
\label{eqn-leave-one-out-opt-row}
&\min_{\Vector{X}, \Vector{Y}} \fbal^{(l)}(\Vector{X}, \Vector{Y})\triangleq \frac{1}{2p}\fnorm{\left(\mathcal{P}_{\Omega_{-l,\cdot}}+p\mathcal{P}_{l,\cdot}\right)\left(\Vector{X}\Vector{Y}^\top- \Vector{M}_{\star}\right)}^{2}  \nonumber\\
& \qquad \qquad \qquad \qquad \qquad + \frac{1}{8}\fnorm{\Vector{X}^\top\Vector{X} - \Vector{Y}^\top\Vector{Y}}^{2}.
\end{align}
In problem \eqref{eqn-leave-one-out-opt-row}, it is assumed that all elements in the $l$-th row of the target $\Vector{M}_{\star}$ are observable, which eliminates the influence of the randomness of the observation operator on this row.

It is worth noting that the objective function of problem \eqref{eqn-leave-one-out-opt-row} is modified from $\fbal(\Vector{X}, \Vector{Y})$ rather than $f(\Vector{X}, \Vector{Y})$, since in the subsequent inductive proof, the inductive hypothesis of linear convergence can ensure that the balancing terms of the sequences $\{\Vector{X}_{k}\}$ and $\{\Vector{Y}_{k}\}$ corresponding to the original problem \eqref{eqn-matrix-completion} have a relatively small upper bound. 

Then we provide the update rule and the initialization method for  problem \eqref{eqn-leave-one-out-opt-row}. The update rule through gradient descent is
\begin{align}
\label{eqn-gd-mc-loo1}
\Vector{X}_{k+1}^{(l)} = &\Vector{X}_{k}^{(l)}  - \frac{s}{p}\mathcal{P}_{\Omega_{-l,\cdot}}\left(\Vector{X}_{k}^{(l)}\left(\Vector{Y}_{k}^{(l)}\right)^\top- \Vector{M}_{\star}\right)\Vector{Y}_{k}^{(l)} \nonumber\\
&  - s\mathcal{P}_{l,\cdot}\left(\Vector{X}_{k}^{(l)}\left(\Vector{Y}_{k}^{(l)}\right)^\top- \Vector{M}_{\star}\right)\Vector{Y}_{k}^{(l)} \nonumber\\
&- \frac{s}{2}\Vector{X}_{k}^{(l)}\left(\left(\Vector{X}_{k}^{(l)}\right)^\top\Vector{X}_{k}^{(l)} - \left(\Vector{Y}_{k}^{(l)}\right)^\top\Vector{Y}_{k}^{(l)}\right), 
\end{align}
and
\begin{align} \label{eqn-gd-mc-loo2}
\Vector{Y}_{k+1}^{(l)} =  &\Vector{Y}_{k}^{(l)}  - \frac{s}{p}\mathcal{P}_{\Omega_{-l,\cdot}}\left(\Vector{X}_{k}^{(l)}\left(\Vector{Y}_{k}^{(l)}\right)^\top- \Vector{M}_{\star}\right)^\top\Vector{X}_{k}^{(l)} \nonumber\\
&- s\mathcal{P}_{l,\cdot}\left(\Vector{X}_{k}^{(l)}\left(\Vector{Y}_{k}^{(l)}\right)^\top- \Vector{M}_{\star}\right)^\top\Vector{X}_{k}^{(l)} \nonumber\\
&- \frac{s}{2}\Vector{Y}_{k}^{(l)}\left(\left(\Vector{Y}_{k}^{(l)}\right)^\top\Vector{Y}_{k}^{(l)} - \left(\Vector{X}_{k}^{(l)}\right)^\top\Vector{X}_{k}^{(l)}\right).
\end{align}
Accordingly, we define  $\Vector{F}_k^{(l)}=\left[(\Vector{X}_k^{(l)})^\top, (\Vector{Y}_k^{(l)})^\top\right]^\top$. The initial point is generated by the spectral decomposition of the observed matrix
\begin{equation}
\Vector{M}_{0}^{(l)} \triangleq 
\left(\frac{1}{p}\mathcal{P}_{\Omega_{-l,\cdot}} + \mathcal{P}_{l,\cdot}\right)\left(\Vector{M}_{\star}\right).
\end{equation}
Similarly we can define the leave-one-out matrix completion problem for $d_{1} + 1 \leq l \leq d_{1} + d_{2}$.

\textbf{Mathematical induction.} To apply mathematical induction, we should make some hypotheses for the bounds of $\Vector{F}_k$, $\Vector{F}_k^{(l)}$, and $\Vector{F}^\star$. However, noting that we cannot guarantee the existence of the best alignment matrix $\Vector{Q}_{k}$ that takes the infimum in \eqref{eqn-defn-distance} for $\Vector{F}_k$ and $\Vector{F}^\star$, we need to introduce some well-defined best rotation matrices for matrices $\Vector{F}_k$, $\Vector{F}_k^{(l)}$, and $\Vector{F}^\star$:
\begin{align}
\Vector{O}_{k} \triangleq & \mathop{\arg\min}_{\Vector{O}\in\mathcal{O}_r}\fnorm{\Vector{F}_{k}\Vector{O} - \Vector{F}_{\star}}, \\
\Vector{O}_{k}^{(l)} \triangleq & \mathop{\arg\min}_{\Vector{O}\in\mathcal{O}_r}\fnorm{\Vector{F}_{k}^{(l)}\Vector{O} - \Vector{F}_{\star}}, 1\leq l\leq d_{1} + d_{2},\\
\Vector{R}_{k}^{(l)} \triangleq & \mathop{\arg\min}_{\Vector{O}\in\mathcal{O}_r}\fnorm{\Vector{F}_{k}\Vector{O}_{k} - \Vector{F}_{k}^{(l)}\Vector{O}},  1\leq l\leq d_{1} + d_{2}.
\end{align}
It can be shown that the existence of $\Vector{Q}_{k}$ can be derived from the above matrices under certain conditions. Moreover, the distance between $\Vector{Q}_{k}$ and $\Vector{O}_{k}$ is very close under the spectral norm. It should be noted that in prior works such as \cite{ma2020implicit, chen2020noisy}, which study the convergence of gradient methods for matrix completion, the conclusion is that $\Vector{F}_{k}$ converges linearly to $\Vector{F}_{\star}$ up to rotation—meaning that in the distance metric \eqref{eqn-defn-distance}, $\Vector{Q}$ is strictly required to be an orthogonal matrix. In this section, we ensure that the spectral norm, Frobenius norm, and $\ell_{2,\infty}$-norm of difference among $\Vector{F}_{k}$, $\Vector{F}_{k}^{(l)}$ and $\Vector{F}_{\star}$ remain bounded via the optimal rotation matrix, thereby proving that gradient descent achieves linear convergence in the sense of optimal alignment.

In the induction proof, we hypothesize that whenever $0\leq t\leq k$, the distance between $\Vector{F}_{t}$ and $\Vector{F}_{\star}$, $(\Vector{F}_{t}^{(l)})_{l,\cdot}$ and $(\Vector{F}_{\star})_{l,\cdot}$, $\Vector{F}_{t}$ and $\Vector{F}_{t}^{(l)}$, as well as $\Vector{Q}_{t}$ and $\Vector{O}_{t}$ are bounded by sufficiently small quantities under various norms, and $\Vector{F}_{t}$ converges to $\Vector{F}_{\star}$ linearly, as Hypothesis \ref{hypothesis-induction} shows.

\begin{hypothesis}[Induction Hypothesis]
\label{hypothesis-induction}
With high probability, the following statements hold for all $0\leq t\leq k$:
\begin{enumerate}[label=(\alph*)]
\item \label{induction-op-norm} $\Vector{F}_{t}$ satisfies
\begin{multline}
    \opnorm{\Vector{F}_{t}\Vector{O}_{t} - \Vector{F}_{\star}} \\ \leq \left(s\sigma_{\min} + \sqrt{\frac{\mu r\kappa^{6}\log d_{1}}{pd_{2}}}\right)\sqrt{\sigma_{\max}}~;
\end{multline}

\item \label{induction-l-2-norm} For $1\leq l\leq d_{1}+d_{2}$, $\Vector{F}_{t}^{(l)}$ satisfies
\begin{multline}
\twonorm{\left(\Vector{F}_{t}^{(l)}\Vector{O}_{t}^{(l)} - \Vector{F}_{\star}\right)_{l,\cdot}} \\ \leq \left(10^{3}s\kappa^{2}\sigma_{\min} + 10^{2}\sqrt{\frac{\mu^{2}r^{2}\kappa^{14}\log d_{1}}{pd_{2}}}\right)\sqrt{\frac{\mu r\sigma_{\max}}{d_{2}}}~;
\end{multline}

\item \label{induction-origin-loo-relation} $\Vector{F}_{t}$ and $\Vector{F}_{t}^{(l)}$ satisfy
\begin{multline}
\fnorm{\Vector{F}_{t}\Vector{O}_{t} - \Vector{F}_{t}^{(l)}\Vector{R}_{t}^{(l)}} \\ \leq \left(\frac{s\sigma_{\min}}{\kappa} + \sqrt{\frac{\mu^{2}r^{2}\kappa^{10}\log d_{1}}{pd_{2}^{2}}}\right)\sqrt{\sigma_{\max}}~;
\end{multline}
\item \label{induction-linear-convergence} $\Vector{F}_{t}$ converges linearly to $\Vector{F}_{\star}$, which satisfies
\begin{align}
\dist{\Vector{F}_{t}}{\Vector{F}_{\star}}  \leq \left(1 - \frac{s\sigma_{\min}}{100}\right)^{t}\dist{\Vector{F}_{0}}{\Vector{F}_{\star}}~;
\end{align}

\item \label{induction-bam-exist} The optimal alignment matrix $\Vector{Q}_{t}$ between $\Vector{F}_{t}$ and $\Vector{F}_{\star}$ exists and satisfies
\begin{equation}
\opnorm{\Vector{Q}_{t} - \Vector{O}_{t}} \leq \frac{1}{400\kappa}~.
\end{equation}
\end{enumerate}
\end{hypothesis}

Spectral initialization ensures that the initial matrix $\Vector{F}_{0}$ is sufficiently close to the target matrix $\Vector{F}_{\star}$. Consequently, this proximity enables Hypothesis \ref{hypothesis-induction}\ref{induction-op-norm}-\ref{induction-origin-loo-relation} to be satisfied at the initial iteration $k=0$, thereby guaranteeing that \ref{induction-bam-exist} also holds. As a result, Hypothesis \ref{hypothesis-induction} is valid at the initial point.
Building upon the induction hypothesis, we first establish the incoherence properties of $\Vector{X}_{k}$ and $\Vector{Y}_{k}$ in Lemma \ref{coro-bounded-2-inf-norm}. This subsequently ensures a small upper bound on the balancing term $\fnorm{\Vector{X}_{k}^{\top}\Vector{X}_{k} - \Vector{Y}_{k}^{\top}\Vector{Y}_{k}}$ in Lemma \ref{lemma-small-balance}, which is consistent with our observation of gradient flow \eqref{eqn-gradient-flow-mc}. These two properties collectively guarantee that the induction hypothesis remains valid at step $k+1$ with high probability. By combining the properties of the initial point with a union bound argument, we conclude that for all steps $k$ not exceeding a sufficiently large threshold dependent on $d_{1}$ and $d_{2}$, the linear convergence guarantee holds as stated in Theorem \ref{thm-linear-convergence}. 

\section{Simulations} \label{sec:simulations}

In this section, we compare the performance of the vanilla gradient descent algorithm VGD with two regularized algorithms RGD and BGD. Experiments were conducted on an Intel Core Ultra 5 125H processor with a base clock frequency of 1.2 GHz, accompanied by 32 GB of RAM.

The ground truth matrix $\Vector{X}^\star \in \mathbb{R}^{d_1 \times d_2}$ of rank $r$ is generated as follows: we first generate random matrices $\Vector{U}^\star \in \mathbb{R}^{d_1 \times r}, \Vector{V}^\star \in \mathbb{R}^{d_2 \times r}$ with orthonormal columns through QR decomposition of i.i.d. Bernoulli $\pm 1$ matrices.  The singular values $\{\sigma_i\}_{i=1}^r$ are linearly spaced between 1 and $1/\kappa$, yielding $\Vector{X}^\star = \Vector{U}^\star \mathrm{diag}(\Vector{\sigma}) (\Vector{V}^{\star})^\top$, where $\Vector{\sigma}=[\sigma_1,\sigma_2,\ldots,\sigma_r]^\top$. For each combination of sampling rate $p$  and rank $r$, we generate a binary sampling mask $\Vector{\Omega}$, where each entry is independently set to 1 with probability $p$. The observed matrix $ \Vector{Y} = \Vector{\Omega}\odot \Vector{X}^\star$ contains measurements of the ground truth at the sampled locations. For all gradient descent algorithms, we set the learning rate as $ s = 0.5 $. Relative error is used to compare the performance, which is defined as 
\begin{align}
    \mbox{Relative~error} = \frac{\fnorm{\Vector{M}_K-\Vector{M}_\star}}{\fnorm{\Vector{M}_\star}}.
\end{align}
And the algorithm stops when the relative error is below $10^{-14}$.

\begin{figure}
\centering
\subfloat[$\kappa=1$]{
    \includegraphics[scale=0.32]{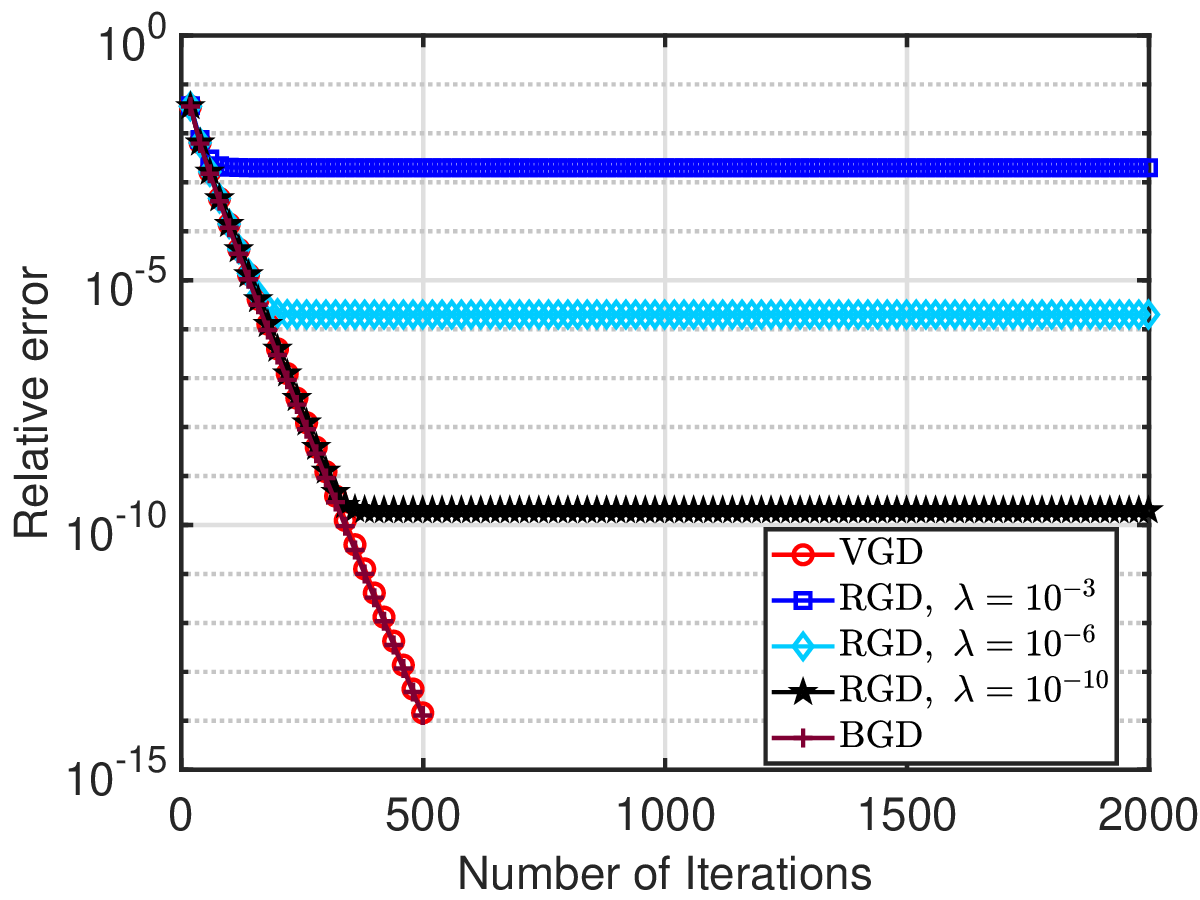}}
\hfill
\subfloat[$\kappa=3$]{
    \includegraphics[scale=0.32]{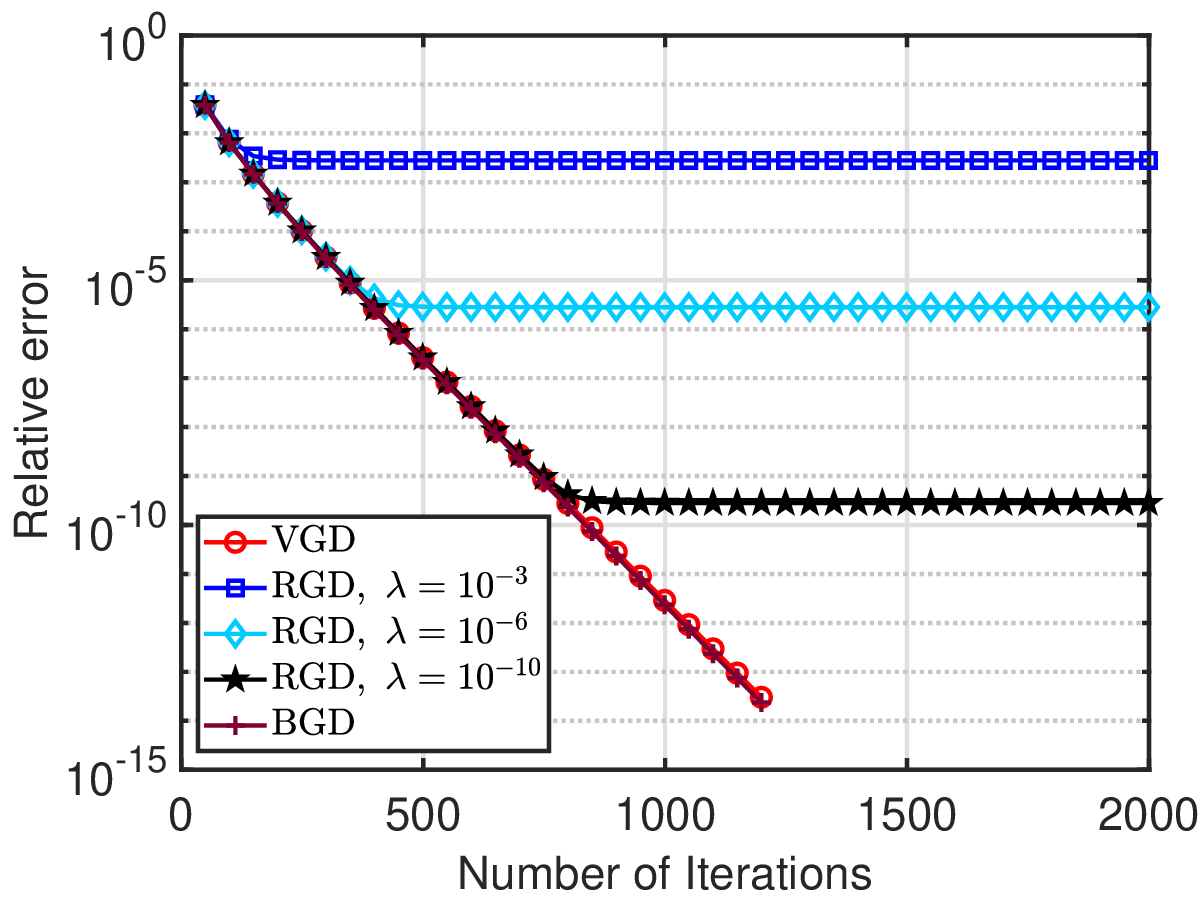}}
\hfill
\subfloat[$\kappa=5$]{
    \includegraphics[scale=0.32]{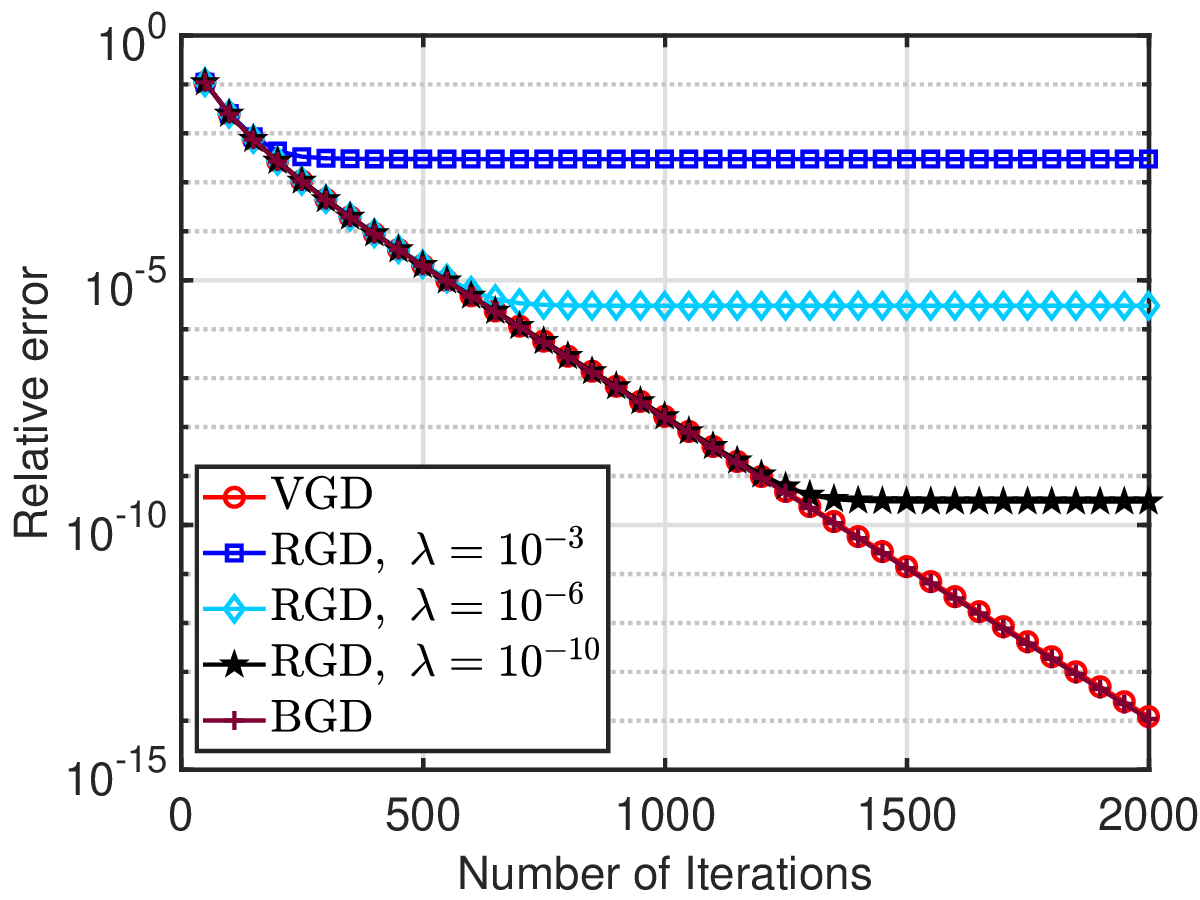}}
\caption{Convergence results for three gradient methods under $d_1=160, d_2=100$, $r=5$ and $p=0.2$.}
\label{fig: matrix-completion-via-gd_2}
\end{figure}

First, we choose a different kind of setting from Fig. \ref{fig: matrix-completion-via-gd} to present the convergence performance when $d_1$ and $d_2$ are relatively small. In Fig. \ref{fig: matrix-completion-via-gd_2}, we set $d_1=160$, $d_2=100$, $p=0.2$ and $r=5$. We vary $\kappa$ from $1$ to $5$ in steps of $2$. The results demonstrate that the convergence curves under the same $\kappa$ are almost the same for VGD and BGD, which exhibit linear convergence for all condition numbers. In addition, the curves of RGD converge to a constant error, and the error gets smaller as $\lambda$ decreases. Notice that RGD degrades to VGD when $\lambda=0$ and the performance becomes the best, which means VGD is a better choice to have a smaller relative error. Furthermore, the convergence speeds of VGD and BGD slow down as the condition number $\kappa$ increases, which coincides with Theorem \ref{thm-linear-convergence}.

Then we plot the phase transition of VGD, RGD and BGD for different $p$ and $r$ under $d_1=400$, $d_2=300$, and $\kappa=3$. We set $s=0.5$ for all algorithms and $\lambda=10^{-6}$ and $\lambda=10^{-10}$ for RGD, respectively. We increase the sampling rate $p$ from $0.05$ to $0.95$ in steps of $0.05$ and increase the rank $r$ from $20$ to $200$ in steps of $20$. We make 50 Monte Carlo trials for each pair of $p$ and $r$. A trial is successful if its relative error is less than $10^{-8}$. The empirical success probability is calculated and visualized as a 2D gray map, with the 50\% success contour extracted to demarcate the recovery boundary. As shown in Fig. \ref{fig:phase_transition}, the phase transition curves for VGD, BGD, and RGD with $\lambda=10^{-10}$ are the same, which also validates that regularization terms are not necessary for gradient descent algorithms with spectral initialization. However, Fig. 3(c) demonstrates that RGD with $\lambda=10^{-6}$ cannot complete the matrix for all pairs of $p$ and $r$, which means it is important for RGD to choose a suitable $\lambda$.

\begin{figure}
    \centering
    \subfloat[VGD]{
    \includegraphics[width=0.48\linewidth]{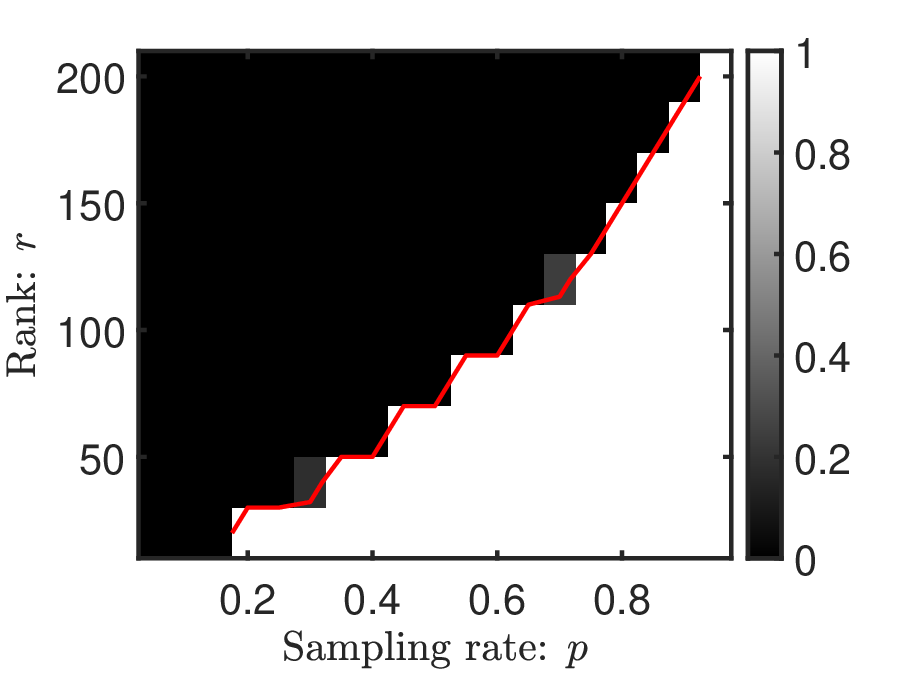}}
    \hfill
    \subfloat[BGD]{
    \includegraphics[width=0.48\linewidth]{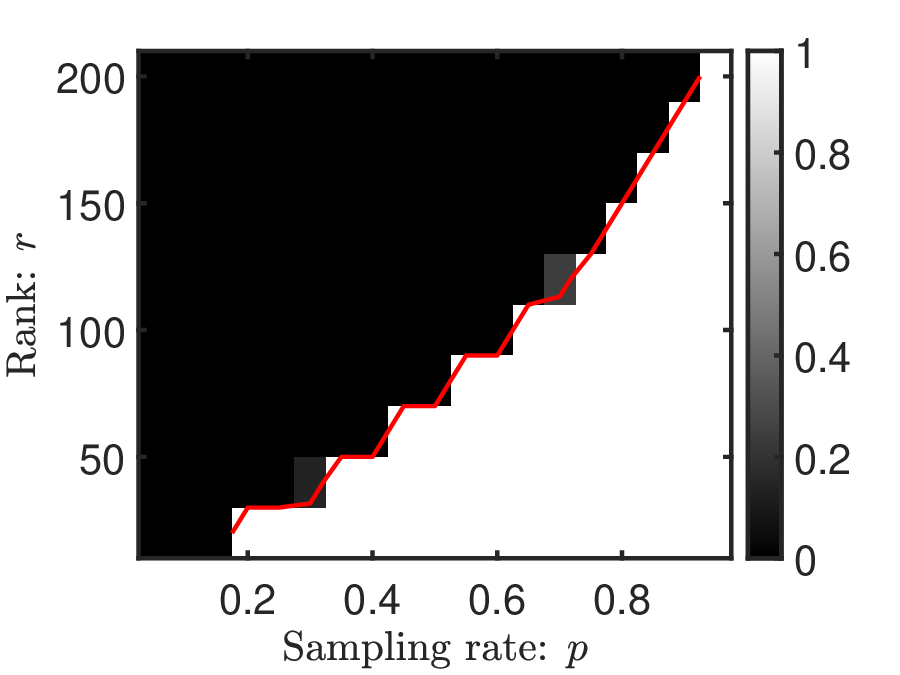}}
    \hfill
    \subfloat[RGD, $\lambda=10^{-6}$]{
    \includegraphics[width=0.48\linewidth]{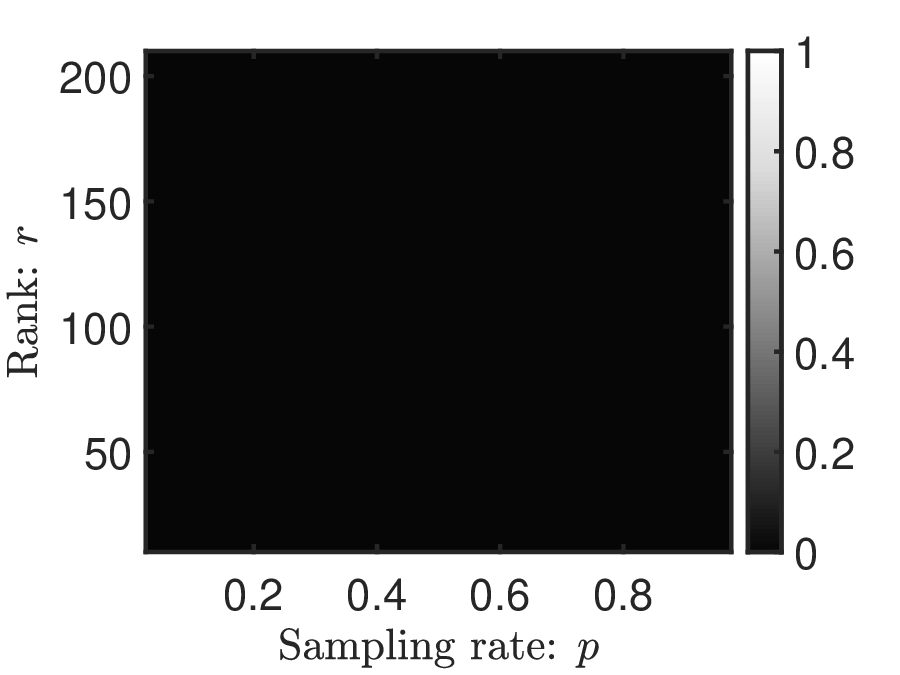}}
    \hfill
    \subfloat[RGD, $\lambda=10^{-10}$]{
    \includegraphics[width=0.48\linewidth]{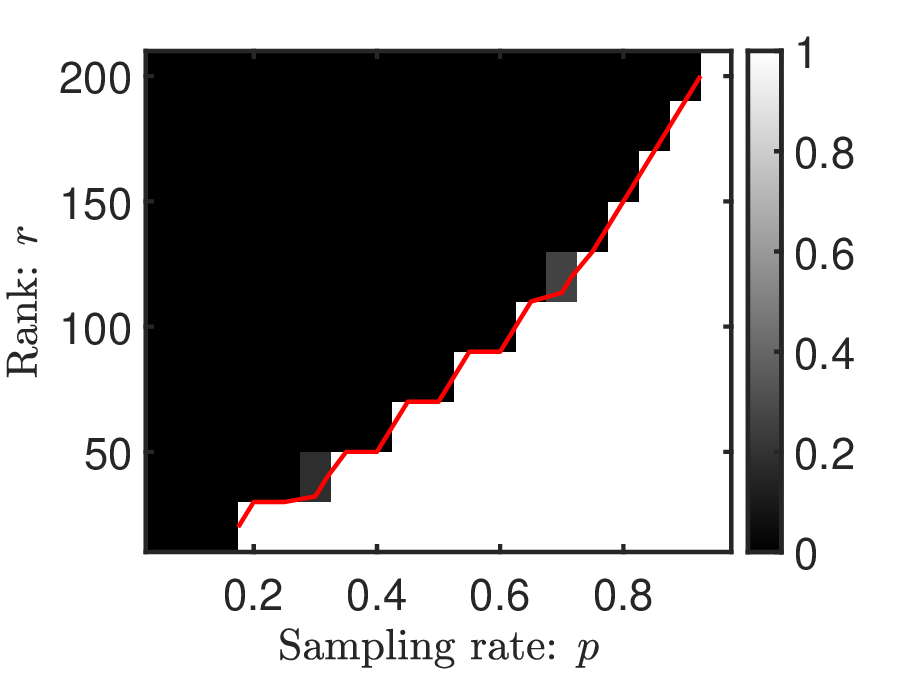}}
    \caption{The comparisons of phase transitions for VGD, BGD, and RGD. The red curve is the 50\% success rate curve.}
    \label{fig:phase_transition}
\end{figure}
\begin{table*}[h!]	
\begin{center}
\caption{The average running time (in seconds) for RGD, BGD, and VGD to reach a relative error $10^{-8}$. }
\begin{tabular}{c|c|c|c|c|c|c|c|c|c}
   \hline
\textbf{$(d_1,d_2)$} & \multicolumn{3}{c|}{$(160,100)$} &  \multicolumn{3}{c|}{$(1200,800)$}  &  \multicolumn{3}{c}{$(3000,2000)$}    \\
\hline			
$(r,p)$ &{$(3, 0.2)$} &{$(5,0.2)$} &{$(10,0.3)$}  &{$(10,0.2)$} &{$(20,0.2)$} &{$(50,0.3)$}  &{$(20,0.1)$} &{$(50,0.1)$} &{$(100,0.2)$} \\ 
\hline
\hline
RGD  &$0.0248$ & $0.0333$  & $0.0465$  & $0.696$ & $0.996$  & $1.625$  & $6.022$ & $18.313$  & $20.385$  \\
BGD  &$0.0396$  & $0.0511$  & $0.0717$  & $1.149$  & $1.619$  & $2.540$  & $9.430$ & $27.627$  & $30.961$ \\
VGD  &$\bm{0.0240}$  & $\bm{0.0330}$ &  $\bm{0.0457}$  & $\bm{0.689}$  & $\bm{0.992$}  & $\bm{1.620}$  & $\bm{6.014}$ & $\bm{18.152}$  & $\bm{20.187}$ \\
\hline
\end{tabular}
\label{tab:average_time}
\end{center}		
\end{table*}

Finally, we compare the computation time of the three gradient algorithms to show the computational efficiency of VGD. We set $\lambda=10^{-10}$ to avoid the running time of RGD being infinity. Additionally, we set the step size to $s=0.5$, the condition number $\kappa=3$, and perform $50$ Monte Carlo trials for all algorithms. Fig. \ref{fig:runningtime} provides the relative error as a function of computation time for two different settings: (a) $d_1=1200, d_2=800, r=10$; (b) $d_1=160, d_2=100, r=5$.  Table \ref{tab:average_time} includes more settings of parameters, which provides the average running time to achieve a relative error $10^{-8}$. The results in Fig. \ref{fig:runningtime}  and Table \ref{tab:average_time} present that VGD is the most computationally efficient method for achieving high-precision solutions, particularly in large-scale scenarios. RGD remains a competitive alternative with nearly identical performance characteristics, while BGD exhibits fundamental efficiency limitations that intensify with problem scale. These results indicate that VGD's architectural design leads to faster convergence in gradient computation.

\begin{figure}
    \centering
    \subfloat[$d_1=1200, d_2=800, r=10$]{
    \includegraphics[scale=0.35]{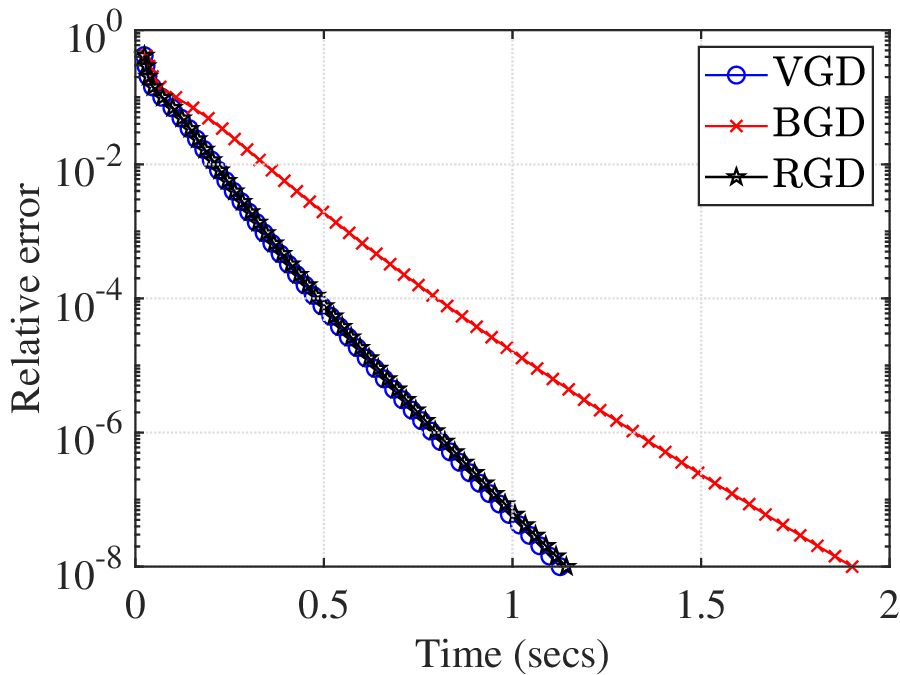}}
    \hfill
    \subfloat[$d_1=160, d_2=100, r=5$]{
    \includegraphics[scale=0.35]{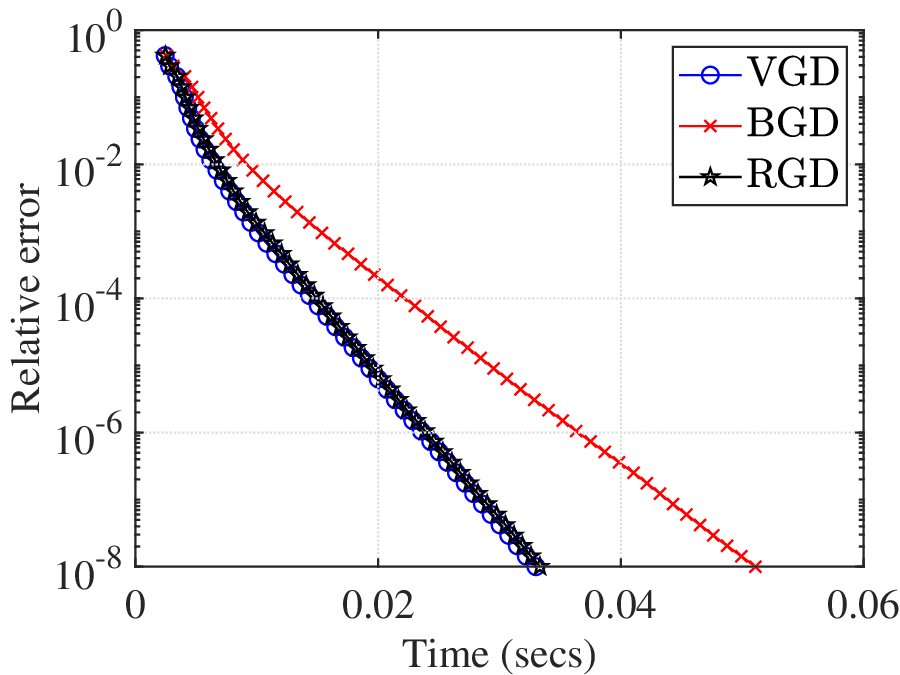}}
    \hfill
    \caption{The comparisons of computational time for VGD, BGD, and RGD. }
    \label{fig:runningtime}
\end{figure}



\section{Conclusion} \label{sec:conclusion}
This paper establishes that gradient descent (GD) with spectral initialization achieves linear convergence with high probability for asymmetric low-rank matrix completion, while eliminating the need for explicit regularization. We reveal GD’s intrinsic implicit regularization property through a novel leave-one-out sequence analysis, and we prove the balancing term maintains a bounded norm throughout iterations, inherently ensuring convergence without explicit regularization terms. Numerical results demonstrate that vanilla GD reduces computational costs by avoiding regularization-related overhead while matching the completion accuracy of regularized GD variants. 


\appendices

\section{Proof of Theorem \ref{thm-linear-convergence}}
This section demonstrates that Algorithm \ref{alg:alg1} achieves linear convergence with high probability by mathematical induction. Due to the limit of pages, we delay auxiliary lemmas (Lemmas \ref{lemma-rip-subspace}-\ref{lemma-bam-exist-sufficient}) and some proofs of lemmas in the supplementary material.

We first establish the incoherence property of $\Vector{X}_{k}$ and $\Vector{Y}_{k}$ through Lemmas \ref{lemma-relation-o-and-r}, \ref{lemma-incoherence}, and \ref{coro-bounded-2-inf-norm}, then derive the small upper bound of balancing term norm in Lemma \ref{lemma-small-balance}, which is a key result for proving Assumption \ref{hypothesis-induction}\ref{induction-op-norm}, \ref{induction-l-2-norm} and \ref{induction-origin-loo-relation} at $(k+1)$-th step. Subsequently, note that the expectation of the matrix completion problem \eqref{eqn-matrix-completion} is a low-rank matrix factorization problem, we reformulate the iteration for matrix completion as the combination of a gradient method for the matrix factorization problem and the perturbation term between these two iterations. Consequently, we prove the linear convergence induction hypothesis Assumption \ref{hypothesis-induction}\ref{induction-linear-convergence} by the existing convergence result for matrix factorization and the upper bound of the perturbation term. Finally, Hypothesis \ref{hypothesis-induction}\ref{induction-bam-exist} can be derived to hold at $(k+1)$-th step based on the previous result for Hypothesis \ref{hypothesis-induction}\ref{induction-op-norm}-\ref{induction-linear-convergence}.


Without loss of generality, we assume that $d_{1} \geq d_{2}$; otherwise, we can transpose the target matrix $\Vector{M}_{\star}$. We also assume $\log d_{1} \geq 1$, as the cases where $d_{1}=1$ or $2$ can be treated separately. 

Lemmas \ref{lemma-rip-subspace} and \ref{lemma-rip-all-space} show the RIP property of the matrix completion problem to some extent when incoherence condition is satisfied. In particular, Lemma \ref{lemma-rip-subspace} shows that in the subspace
\begin{multline}
   \big\{\Vector{M}\in\mathbb{R}^{d_{1}\times d_{2}}: \Vector{M} = \Vector{X}_{\star}\Vector{Y}^\top + \Vector{X}\Vector{Y}_{\star}^\top, \\ \forall\Vector{X}\in\mathbb{R}^{d_{1}\times r}, \Vector{Y}\in\mathbb{R}^{d_{2}\times r}\big\}, 
\end{multline}
the operator $p^{-1}\mathcal{P}_{\Omega}$ has RIP property. Lemma \ref{lemma-rip-all-space} shows that although $p^{-1}\mathcal{P}_{\Omega}$ doesn't satisfy the RIP property in the whole space, the distance between $p^{-1}\mathcal{P}_{\Omega}$ and $\mathcal{I}$ can be bounded. Define the event that both Lemmas \ref{lemma-rip-subspace} and \ref{lemma-rip-all-space} hold as 
$\mathrm{E}_{\mathrm{RIP}}$. According to \cite{chen2020nonconvex}, when $p$ satisfies the assumption in Eq. \eqref{eqn-assumption-p-s}, $\mathrm{E}_{\mathrm{RIP}}$ holds with probability at least $ 1 - \left(d_{1} + d_{2}\right)^{-11}$.

Let $\mathrm{E}_{k}$ denote the event that the Induction Hypothesis holds. As shown in Hypothesis \ref{hypothesis-induction}, the induction hypotheses \ref{induction-op-norm}-\ref{induction-origin-loo-relation} demonstrate that the iterative sequence remains bounded relative to the optimal solution up to rotation, while \ref{induction-linear-convergence}-\ref{induction-bam-exist} establish the linear convergence rate under optimal alignment.

Utilizing Lemma \ref{lemma-alignment-bound}, we obtain the following lemma, which shows that $\Vector{O}_{k}^{(l)}$ exhibits similar properties to $\Vector{R}_{k}^{(l)}$.

\begin{lemma}
\label{lemma-relation-o-and-r}
If Hypothesis \ref{hypothesis-induction} holds and the assumptions on $p$ and $s$ in \eqref{eqn-assumption-p-s} are satisfied, then
\begin{align}
\opnorm{\Vector{F}_{k}\Vector{O}_{k} - \Vector{F}_{k}^{(l)}\Vector{O}_{k}^{(l)}} \leq & 5\kappa\opnorm{\Vector{F}_{k}\Vector{O}_{k} - \Vector{F}_{k}^{(l)}\Vector{R}_{k}^{(l)}}, \\
\fnorm{\Vector{F}_{k}\Vector{O}_{k} - \Vector{F}_{k}^{(l)}\Vector{O}_{k}^{(l)}} \leq & 5\kappa\fnorm{\Vector{F}_{k}\Vector{O}_{k} - \Vector{F}_{k}^{(l)}\Vector{R}_{k}^{(l)}}.
\end{align}
\end{lemma}

\begin{proof}
See Appendix \ref{Appedd:relation-o-and-r} of the Supplementary Material.
\end{proof}

Then we establish that both $\Vector{X}_{k}$ and $\Vector{Y}_{k}$ satisfy the incoherence condition.

\begin{lemma}
\label{lemma-incoherence}
If Hypothesis \ref{hypothesis-induction} holds and the assumptions on $p$ and $s$ in \eqref{eqn-assumption-p-s} are satisfied, then
\begin{multline}
\twoinfnorm{\Vector{Y}_{k}\Vector{O}_{k} - \Vector{Y}_{\star}},~\twoinfnorm{\Vector{X}_{k}\Vector{O}_{k} - \Vector{X}_{\star}}\\ \leq \bigg((10^{3} + 5)s\kappa^{2}\sigma_{\min} + (10^{2} + 5)\sqrt{\frac{\mu^{2}r^{2}\kappa^{14}\log d_{1}}{pd_{2}}}\bigg)\\\times\sqrt{\frac{\mu r\sigma_{\max}}{d_{2}}}. \label{eqn:X-bound} 
\end{multline}
\end{lemma}
\begin{proof}
See Appendix \ref{Append:lemma-incoherence} of the Supplementary Material.
\end{proof}

\begin{lemma}
\label{coro-bounded-2-inf-norm}
If Hypothesis \ref{hypothesis-induction} holds and the assumptions on $p$ and $s$ in \eqref{eqn-assumption-p-s} are satisfied, the following inequalities hold
\begin{align}
\twoinfnorm{\Vector{X}_{k}} &\leq \frac{17}{16}\sqrt{\frac{\mu r\sigma_{\max}}{d_{1}}}, \\
\twoinfnorm{\Vector{Y}_{k}} &\leq \frac{17}{16}\sqrt{\frac{\mu r\sigma_{\max}}{d_{2}}}, \label{eqn:XY-norm-bound} \\
\twoinfnorm{\Vector{X}_{k}\Vector{Q}_{k} - \Vector{X}_{\star}} &\leq \frac{5}{2}\sqrt{\frac{\mu r\sigma_{\max}}{d_{1}}},\\ \twoinfnorm{\Vector{Y}_{k}\Vector{Q}_{k}^{-\top} - \Vector{Y}_{\star}} &\leq \frac{5}{2}\sqrt{\frac{\mu r\sigma_{\max}}{d_{2}}}. \label{eqn:XYQ-norm-bound}
\end{align}
\end{lemma}
\begin{proof}
See Appendix \ref{Append:coro-bounded-2-inf-norm} of the Supplementary Material.
\end{proof}

Next, we show that the balancing term is upper bounded by a small bound.

\begin{lemma}
\label{lemma-small-balance}
If Hypothesis \ref{hypothesis-induction} holds and the assumptions on $p$ and $s$ in \eqref{eqn-assumption-p-s} are satisfied, then the following inequality holds
\begin{align}
\fnorm{\Vector{X}_{k}^\top\Vector{X}_{k} - \Vector{Y}_{k}^\top\Vector{Y}_{k}} \leq \frac{s\sigma_{\min}^{2}}{10^{2}\kappa}.
\end{align}
\end{lemma}
\begin{proof}
    See Appendix \ref{Apend:lemma-small-balance} of the Supplementary Material.
\end{proof}


To establish that Hypothesis \ref{hypothesis-induction} holds at the initial point, we first refer to Lemma \ref{lemma-initial-property}, which demonstrates that Hypothesis \ref{hypothesis-induction}\ref{induction-op-norm}–\ref{induction-origin-loo-relation} of the hypothesis are satisfied with high probability. Additionally, Hypothesis \ref{hypothesis-induction}\ref{induction-linear-convergence} of the hypothesis is naturally fulfilled at iteration $k=0$.

Moreover, Lemma \ref{lemma-initial-property} provides the following bound:
\begin{equation}
\label{eqn-f-norm-initial}
\fnorm{\Vector{F}_0 \Vector{O}_0 - \Vector{F}_\star} \leq \sqrt{r} \opnorm{\Vector{F}_0 \Vector{O}_0 - \Vector{F}_\star} \leq \frac{c_0 \sqrt{\sigma_{\max}}}{\kappa^2},
\end{equation}  
where $c_{0}$ is a sufficiently small constant. By invoking Lemma \ref{lemma-bam-exist-sufficient} with $\Vector{P}=\Vector{O}_{0}$ and $\delta= \frac{c_0 \sqrt{\sigma_{\max}}}{\kappa^2}=\frac{c_{0}\sqrt{\sigma_{\min}}}{\kappa^{3/2}}$, we can conclude that Hypothesis \ref{hypothesis-induction}.\ref{induction-bam-exist} is also satisfied at $k=0$.

Armed with the above results, we proceed to establish the inductive step.

\subsection{Inductive Step for Hypothesis \ref{hypothesis-induction}\ref{induction-op-norm}}

We first verify that Hypothesis \ref{hypothesis-induction}\ref{induction-op-norm} holds at the $(k+1)$-th iteration.

\begin{lemma}
\label{lemma-op-norm}
If Hypothesis \ref{hypothesis-induction} holds and the assumptions on $p$ and $s$ in \eqref{eqn-assumption-p-s} are satisfied, then the following estimate holds
\begin{align}
\opnorm{\Vector{F}_{k+1}\Vector{O}_{k+1} - \Vector{F}_{\star}} \leq \left(s\sigma_{\min} + \sqrt{\frac{\mu r\kappa^{6}\log d_{1}}{pd_{2}}}\right)\sqrt{\sigma_{\max}}.
\end{align}
\end{lemma}

\begin{proof}
To prove this result, we introduce an auxiliary sequence $\widetilde{\Vector{F}}_{k+1} = [\widetilde{\Vector{X}}_{k+1}^\top, \widetilde{\Vector{Y}}_{k+1}^\top]^\top$, defined as
\begin{align}
\widetilde{\Vector{X}}_{k+1} &= \Vector{X}_{k}\Vector{O}_{k} - s\Big(p^{-1}\mathcal{P}_{\Omega}\left(\Vector{X}_{k}\Vector{Y}_{k}^\top- \Vector{M}_{\star}\right)\Vector{Y}_{\star} \nonumber \\
&\quad + \frac{1}{2}\Vector{X}_{\star}\Vector{O}_{k}^\top\left(\Vector{X}_{k}^\top\Vector{X}_{k} - \Vector{Y}_{k}^\top\Vector{Y}_{k}\right)\Vector{O}_{k}\Big), \\
\widetilde{\Vector{Y}}_{k+1} &= \Vector{Y}_{k}\Vector{O}_{k} - s\Big(p^{-1}\mathcal{P}_{\Omega}\left(\Vector{X}_{k}\Vector{Y}_{k}^\top- \Vector{M}_{\star}\right)^\top\Vector{X}_{\star} \nonumber\\
&\quad + \frac{1}{2}\Vector{Y}_{\star}\Vector{O}_{k}^\top\left(\Vector{Y}_{k}^\top\Vector{Y}_{k} - \Vector{X}_{k}^\top\Vector{X}_{k}\right)\Vector{O}_{k}\Big).
\end{align}

By the triangle inequality, we have
\begin{multline}
\opnorm{\Vector{F}_{k+1}\Vector{O}_{k+1} - \Vector{F}_{\star}} \leq \opnorm{\widetilde{\Vector{F}}_{k+1} - \Vector{F}_{\star}} \\ + \opnorm{\Vector{F}_{k+1}\Vector{O}_{k+1} - \widetilde{\Vector{F}}_{k+1}}.
\end{multline}

We first give the upper bound $\opnorm{\widetilde{\Vector{F}}_{k+1} - \Vector{F}_{\star}}$. From the definition of $\widetilde{\Vector{F}}_{k+1}$, we have \eqref{neq:lemma15_long}. For convenience, define
\begin{table*}
\begin{multline} \label{neq:lemma15_long}
\opnorm{\widetilde{\Vector{F}}_{k+1} - \Vector{F}_{\star}}
\leq \underbrace{\opnorm{\begin{bmatrix}
\Vector{X}_{k}\Vector{O}_{k} - \Vector{X}_{\star} - s\left(\left(\Vector{X}_{k}\Vector{Y}_{k}^\top- \Vector{M}_{\star}\right)\Vector{Y}_{\star} + \frac{1}{2}\Vector{X}_{\star}\Vector{O}_{k}^\top\left(\Vector{X}_{k}^\top\Vector{X}_{k} - \Vector{Y}_{k}^\top\Vector{Y}_{k}\right)\Vector{O}_{k}\right) \\
\Vector{Y}_{k}\Vector{O}_{k} - \Vector{Y}_{\star} - s\left(\left(\Vector{X}_{k}\Vector{Y}_{k}^\top- \Vector{M}_{\star}\right)^\top\Vector{X}_{\star} + \frac{1}{2}\Vector{Y}_{\star}\Vector{O}_{k}^\top\left(\Vector{Y}_{k}^\top\Vector{Y}_{k} - \Vector{X}_{k}^\top\Vector{X}_{k}\right)\Vector{O}_{k}\right)
\end{bmatrix}}}_{\eta_{1}}  \\
+ \underbrace{s \opnorm{\begin{bmatrix}
\left(p^{-1}\mathcal{P}_{\Omega} - \mathcal{I}\right)\left(\Vector{X}_{k}\Vector{Y}_{k}^\top- \Vector{M}_{\star}\right)\Vector{Y}_{\star} \\
\left(p^{-1}\mathcal{P}_{\Omega} - \mathcal{I}\right)\left(\Vector{X}_{k}\Vector{Y}_{k}^\top- \Vector{M}_{\star}\right)^\top\Vector{X}_{\star}
\end{bmatrix}}}_{\eta_{2}}.
\end{multline}
\end{table*}
\vspace{-0.2cm}

\begin{align}
\Matrix{\Delta}^k_{\Vector{X}} &= \Vector{X}_{k}\Vector{O}_{k} - \Vector{X}_{\star}, \quad \Matrix{\Delta}^k_{\Vector{Y}} = \Vector{Y}_{k}\Vector{O}_{k} - \Vector{Y}_{\star},\\
\Matrix{\Delta}^k &= \Vector{F}_{k}\Vector{O}_{k} - \Vector{F}_{\star}.
\end{align}

The form of $\eta_1$ is identical to $\alpha_2$ in \cite[Section 4.2]{chen2020nonconvex}. Therefore, from Hypothesis \ref{hypothesis-induction}\ref{induction-op-norm} and the assumptions on $p$ and $s$ in \eqref{eqn-assumption-p-s}, we have
\begin{multline}
\eta_{1} \leq \left(1 - s\sigma_{\min}\right)\opnorm{\Matrix{\Delta}^k} \\
+ 4s\opnorm{\Matrix{\Delta}^k}^{2}\max\left\{\opnorm{\Vector{X}_{\star}}, \opnorm{\Vector{Y}_{\star}}\right\} \\
\leq \left(1 - \frac{3s\sigma_{\min}}{4} \right)\opnorm{\Matrix{\Delta}^k},
\end{multline}
where the last inequality uses $\opnorm{\Vector{X}_{\star}}=\opnorm{\Vector{Y}_{\star}}=\sqrt{\sigma_{\max}}$.

The form of $\eta_2$  is identical to $\alpha_1$ in \cite[Section 4.2]{chen2020nonconvex}, so we have
\begin{align}
\eta_{2} &\leq \frac{2s}{p}\opnorm{\Vector{X}_{\star}}\opnorm{(\mathcal{P}_{\Omega} - \mathcal{I})(\Vector{1}\Vector{1}^\top)}\left(\twoinfnorm{\Matrix{\Delta}^k_{\Vector{X}}}\twoinfnorm{\Matrix{\Delta}^k_{\Vector{Y}}} \right. \nonumber \\
&\quad \left. + \twoinfnorm{\Matrix{\Delta}^k_{\Vector{X}}}\twoinfnorm{\Vector{Y}_{\star}} + \twoinfnorm{\Vector{X}_{\star}}\twoinfnorm{\Matrix{\Delta}^k_{\Vector{Y}}}\right).
\end{align}

From \cite[Lemma 3.2]{raghunandan2010matrix}, when $\mathrm{E}_{\mathrm{RIP}}$ holds, we have
\begin{align}
\opnorm{(\mathcal{P}_{\Omega} - \mathcal{I})(\Vector{1}\Vector{1}^\top)} \lesssim \sqrt{d_{1}p}.
\end{align}

From Lemma \ref{lemma-incoherence} and the assumptions on $p$ and $s$ in \eqref{eqn-assumption-p-s}, we obtain
\begin{align}
\twoinfnorm{\Matrix{\Delta}^k_{\Vector{X}}} &\leq \frac{\sqrt{\sigma_{\max}}}{10^{2}\kappa\sqrt{d_{1}}}, \quad \twoinfnorm{\Matrix{\Delta}^k_{\Vector{Y}}} \leq \frac{\sqrt{\sigma_{\max}}}{10^{2}\kappa\sqrt{d_{2}}}.
\end{align}

Combining these with the $\mu$-incoherence of $\Vector{X}_{\star}$ and $\Vector{Y}_{\star}$, when $p$ satisfies assumption \eqref{eqn-assumption-p-s}, we have
\begin{align}
\eta_{2} &\leq \frac{s\sigma_{\min}}{4}\sqrt{\frac{\mu r\sigma_{\max}}{pd_{2}}}.
\end{align}

On the other hand,
\begin{align}
\opnorm{\Vector{F}_{k+1}\Vector{O}_{k+1} - \widetilde{\Vector{F}}_{k+1}} &= \opnorm{\Vector{F}_{k+1}\Vector{O}_{k}\Vector{O}_{k}^\top\Vector{O}_{k+1} - \widetilde{\Vector{F}}_{k+1}}.
\end{align}

According to \cite[Assertion 4]{chen2020noisy}, the optimal rotation matrix between $\widetilde{\Vector{F}}_{k+1}$ and $\Vector{F}_{\star}$ is the identity matrix $\Vector{I}_{r}$, and we have
\begin{align}
\opnorm{\widetilde{\Vector{F}}_{k+1} - \Vector{F}_{\star}}\opnorm{\Vector{F}_{\star}} &\leq \left(1 - \frac{3s\sigma_{\min}}{4}\right)\opnorm{\Matrix{\Delta}^k_{\Vector{X}}}\sqrt{2\sigma_{\min}} \nonumber \\
&\leq \sigma_{\min} = \frac{\sigma_{\min}^{2}(\Vector{F}_{\star})}{2}.
\end{align}

Note that the optimal rotation matrix between $\Vector{F}_{k+1}\Vector{O}_{k}$ and $\Vector{F}_{\star}$ is $\Vector{O}_{k}^\top\Vector{O}_{k+1}$. By the triangle inequality, we have
\begin{align}
\opnorm{\Vector{F}_{k+1}\Vector{O}_{k} - \widetilde{\Vector{F}}_{k+1}}
\leq & \underbrace{\opnorm{\Vector{F}_{k}\Vector{O}_{k} - s\nabla \fbal(\Vector{F}_{k})\Vector{O}_{k} - \widetilde{\Vector{F}}_{k+1}}}_{\theta_{1}} \nonumber\\
&+ \underbrace{s\opnorm{\nabla f_{\text{diff}}(\Vector{F}_{k})\Vector{O}_{k}}}_{\theta_{2}}.
\end{align}

From \cite[(4.17)]{chen2020nonconvex} and \cite[Lemma 3.2]{raghunandan2010matrix}, we obtain
\begin{align}
\theta_{1} \lesssim& s\sqrt{\frac{d_{1}}{p}}\Big(\twoinfnorm{\Matrix{\Delta}^k_{\Vector{X}}}\twoinfnorm{\Matrix{\Delta}^k_{\Vector{Y}}} + \twoinfnorm{\Matrix{\Delta}^k_{\Vector{X}}}\twoinfnorm{\Vector{Y}_{\star}} \nonumber\\ &+ \twoinfnorm{\Vector{X}_{\star}}\twoinfnorm{\Matrix{\Delta}^k_{\Vector{Y}}}\Big)\opnorm{\Matrix{\Delta}^k} \nonumber \\
&+ s\Big(\opnorm{\Matrix{\Delta}^k_{\Vector{X}}}\opnorm{\Matrix{\Delta}^k_{\Vector{Y}}} + \opnorm{\Matrix{\Delta}^k_{\Vector{X}}}\opnorm{\Vector{Y}_{\star}}  \nonumber \\
&+\opnorm{\Vector{X}_{\star}}\opnorm{\Matrix{\Delta}^k_{\Vector{Y}}} + \opnorm{\Vector{X}_{\star}}\opnorm{\Matrix{\Delta}^k_{\Vector{X}}}  \nonumber \\
&+ \opnorm{\Vector{Y}_{\star}}\opnorm{\Matrix{\Delta}^k_{\Vector{Y}}} + \opnorm{\Matrix{\Delta}^k_{\Vector{X}}}^{2} + \opnorm{\Matrix{\Delta}^k_{\Vector{Y}}}^{2}\Big)\opnorm{\Matrix{\Delta}^k}.
\end{align}

From Hypothesis \ref{hypothesis-induction}\ref{induction-op-norm}, Lemma \ref{lemma-incoherence}, and the assumption on $p$ in \eqref{eqn-assumption-p-s}, we have
\begin{align}
\theta_{1} &\leq \frac{s\sigma_{\min}}{20\kappa}\opnorm{\Matrix{\Delta}^k}.
\end{align}

Combining Lemma \ref{lemma-initial-property}  and Eq. \eqref{eqn-assumption-p-s}, there exists a sufficiently small $ c_0 > 0 $ such that  
\begin{equation}
\label{eqn-f-norm-initial}
\fnorm{\Vector{F}_0 \Vector{O}_0 - \Vector{F}_\star} \leq \sqrt{r} \opnorm{\Vector{F}_0 \Vector{O}_0 - \Vector{F}_\star} \leq \frac{c_0 \sqrt{\sigma_{\max}}}{\kappa^2}.
\end{equation}  

Using inequality \eqref{eqn-f-norm-initial}, Lemma \ref{lemma-small-balance}, and the assumption on $s$ in \eqref{eqn-assumption-p-s}, we obtain for $\theta_2$
\begin{align}
\theta_{2} &\leq \frac{s\sigma_{\min}}{20\kappa}s\sigma_{\max}\sqrt{\sigma_{\max}} \leq \frac{s\sigma_{\min}}{20}s\sigma_{\min}\sqrt{\sigma_{\max}}.
\end{align}

Therefore, we have
\begin{align}
\opnorm{\Vector{F}_{k+1}\Vector{O}_{k} - \Vector{F}_{\star}}\opnorm{\Vector{F}_{\star}} &\leq \left(\theta_{1} + \theta_{2}\right)\sqrt{2\sigma_{\min}} \leq \frac{\sigma_{\min}^{2}(\Vector{F}_{\star})}{4}.
\end{align}

Finally, from Lemma \ref{lemma-alignment-bound}, we conclude
\begin{align}
\opnorm{\Vector{F}_{k+1}\Vector{O}_{k+1} - \Vector{F}_{\star}} \leq & \eta_{1} + \eta_{2} + 5\kappa\left(\theta_{1} + \theta_{2}\right) \nonumber \\
\leq & \left(\sqrt{\frac{\mu r\kappa^{6}\log d_{1}}{pd_{2}}} + s\sigma_{\min}\right)\sqrt{\sigma_{\max}},
\end{align}
which completes the proof of the lemma.
\end{proof}

\subsection{Inductive Step for Hypothesis \ref{hypothesis-induction}\ref{induction-l-2-norm}}

Lemma \ref{lemma-l-2-norm-loo} proves that Hypothesis \ref{hypothesis-induction}\ref{induction-l-2-norm} still holds at the $(k+1)$-th step.

\begin{lemma}
\label{lemma-l-2-norm-loo}
If Hypothesis \ref{hypothesis-induction} and the assumption \eqref{eqn-assumption-p-s} hold, then the following conclusions hold: For $1 \leq l \leq d_{1} + d_{2}$, we have
\begin{multline}
\twonorm{\left(\Vector{F}_{k+1}^{(l)}\Vector{O}_{k+1}^{(l)} - \Vector{F}_{\star}\right)_{l,\cdot}} \\ \leq \left(10^{3}s\kappa^{2}\sigma_{\min} + 50\sqrt{\frac{\mu^{2}r^{2}\kappa^{14}\log d_{1}}{pd_{2}}}\right)\sqrt{\frac{\mu r\sigma_{\max}}{d_{2}}};
\end{multline}
\end{lemma}
\begin{proof}
It suffices to prove the case for $1 \leq l \leq d_{1}$, as the case for $d_{1}+1 \leq l \leq d_{1} + d_{2}$ is entirely analogous.
According to the leave-one-out iteration rule \eqref{eqn-gd-mc-loo1}, we have 
\eqref{eq:long1}.
\begin{table*}
\begin{align} \label{eq:long1}
\left(\Vector{F}_{k+1}^{(l)}\Vector{O}_{k+1}^{(l)} - \Vector{F}_{\star}\right)_{l,\cdot} = &  \left(\Vector{X}_{k+1}^{(l)}\Vector{O}_{k+1}^{(l)} - \Vector{X}_{\star}\right)_{l,\cdot} \nonumber\\
= & \left(\Vector{X}_{k}^{(l)}\right)_{l,\cdot}\Vector{O}_{k+1}^{(l)} - \left(\Vector{X}_{\star}\right)_{l,\cdot} - s\left(\Vector{X}_{k}^{(l)}\left(\Vector{Y}_{k}^{(l)}\right)^\top- \Vector{M}_{\star}\right)_{l,\cdot}\Vector{O}_{k+1}^{(l)} 
- \frac{s}{2}\left(\Vector{X}_{k}^{(l)}\right)_{l,\cdot}\left(\left(\Vector{X}_{k}^{(l)}\right)^\top\Vector{X}_{k}^{(l)} - \left(\Vector{Y}_{k}^{(l)}\right)^\top\Vector{Y}_{k}^{(l)}\right)\Vector{O}_{k+1}^{(l)} \nonumber\\
= & \underbrace{\left(\Vector{X}_{k}^{(l)}\right)_{l,\cdot}\Vector{O}_{k}^{(l)} - \left(\Vector{X}_{\star}\right)_{l,\cdot} - s\left(\Vector{X}_{k}^{(l)}\left(\Vector{Y}_{k}^{(l)}\right)^\top- \Vector{M}_{\star}\right)_{l,\cdot}\Vector{O}_{k}^{(l)}}_{a_{1}} \nonumber\\
&+ \underbrace{\left(\left(\Vector{X}_{k}^{(l)}\right)_{l,\cdot}\Vector{O}_{k}^{(l)} - s\left(\Vector{X}_{k}^{(l)}\left(\Vector{Y}_{k}^{(l)}\right)^\top- \Vector{M}_{\star}\right)_{l,\cdot}\Vector{O}_{k}^{(l)}\right)\left(\left(\Vector{O}_{k}^{(l)}\right)^{-1}\Vector{O}_{k+1}^{(l)} - \Vector{I}_{r}\right)}_{a_{2}} \nonumber \\
& - \underbrace{\frac{s}{2}\left(\Vector{X}_{k}^{(l)}\right)_{l,\cdot}\left(\left(\Vector{X}_{k}^{(l)}\right)^\top\Vector{X}_{k}^{(l)} - \left(\Vector{Y}_{k}^{(l)}\right)^\top\Vector{Y}_{k}^{(l)}\right)\Vector{O}_{k+1}^{(l)}}_{a_{3}}.
\end{align}
\end{table*}

For convenience, let
\begin{align}
\overline{\Vector{X}}_k^{(l)} &= \Vector{X}_{k}^{(l)}\Vector{O}_{k}^{(l)}, \quad \overline{\Vector{Y}}_k^{(l)} = \Vector{Y}_{k}^{(l)}\Vector{O}_{k}^{(l)}, \\
\Matrix{\Delta}_{\Vector{X}}^{k,(l)} &= \overline{\Vector{X}}_k^{(l)} - \Vector{X}_{\star}, \quad
\Matrix{\Delta}_{\Vector{Y}}^{k,(l)} = \overline{\Vector{Y}}_k^{(l)} - \Vector{Y}_{\star}.
\end{align}
Then $a_{1}$ can be rewritten as
\begin{align*}
a_{1} = & \left(\Matrix{\Delta}_{\Vector{X}}^{k,(l)}\right)_{l,\cdot} \\
&- s\left(\left(\Matrix{\Delta}_{\Vector{X}}^{k,(l)}\right)_{l,\cdot}\left(\overline{\Vector{Y}}_k^{(l)}\right)^\top+ \left(\Vector{X}_{\star}\right)_{l,\cdot}\left(\Matrix{\Delta}_{\Vector{X}}^{k,(l)}\right)^\top\right)\overline{\Vector{Y}}_k^{(l)} \\
= & \left(\Matrix{\Delta}_{\Vector{X}}^{k,(l)}\right)_{l,\cdot}\left(\Vector{I}_{r} - s\left(\overline{\Vector{Y}}_k^{(l)}\right)^\top\left(\overline{\Vector{Y}}_k^{(l)}\right)\right) \nonumber\\
&- s\left(\Vector{X}_{\star}\right)_{l,\cdot}\left(\Matrix{\Delta}_{\Vector{X}}^{k,(l)}\right)^\top\overline{\Vector{Y}}_k^{(l)}.
\end{align*}
Thus, we have
\begin{multline}
\twonorm{a_{1}} \leq \opnorm{\Vector{I}_{r} - s\left(\overline{\Vector{Y}}_k^{(l)}\right)^\top\left(\overline{\Vector{Y}}_k^{(l)}\right)}\twonorm{\left(\Matrix{\Delta}_{\Vector{X}}^{k,(l)}\right)_{l,\cdot}} \\+ s\opnorm{\Matrix{\Delta}_{\Vector{X}}^{k,(l)}}\opnorm{\overline{\Vector{Y}}_k^{(l)}}\twonorm{\left(\Vector{X}_{\star}\right)_{l,\cdot}}.
\end{multline}
By Hypothesis \ref{hypothesis-induction}\ref{induction-op-norm}, \ref{induction-origin-loo-relation} and Lemma \ref{lemma-relation-o-and-r}, we have
\begin{align}
&\opnorm{\overline{\Vector{Y}}_k^{(l)} - \Vector{Y}_{\star}} \nonumber \\
\leq & \opnorm{\overline{\Vector{Y}}_k^{(l)} - \Vector{Y}_{k}\Vector{O}_{k}} + \opnorm{\Vector{Y}_{k}\Vector{O}_{k} - \Vector{Y}_{\star}} \nonumber\\
\leq & \fnorm{\Vector{F}^{(l)} - \Vector{F}_{k}\Vector{O}_{k}} + \opnorm{\Vector{Y}_{k}\Vector{O}_{k} - \Vector{Y}_{\star}} \nonumber\\
\leq & 5\kappa\fnorm{\Vector{F}_{k}^{(l)}\Vector{R}_{k}^{(l)} - \Vector{F}_{k}\Vector{O}_{k}} + \opnorm{\Vector{F}_{k}\Vector{O}_{k} - \Vector{F}_{\star}} \nonumber\\
\leq & \left(6s\sigma_{\min} + 2\sqrt{\frac{\mu^{2}r^{2}\kappa^{10}\log d_{1}}{pd_{2}}}\right)\sqrt{\sigma_{\max}}. \label{eqn-loo-op-norm-y-diff}
\end{align}
Therefore we obtain
\begin{equation}
\label{eqn-loo-op-norm-y}
\frac{9\sqrt{\sigma_{\min}}}{10} \leq \sigma_{\min}\left(\overline{\Vector{Y}}_k^{(l)}\right) \leq \sigma_{\max}\left(\overline{\Vector{Y}}_k^{(l)}\right) \leq 2\sqrt{\sigma_{\max}}.
\end{equation}
Similarly, we have
\begin{align}
& \opnorm{\overline{\Vector{X}}_k^{(l)} - \Vector{X}_{\star}} \leq \left(6s\sigma_{\min} + 2\sqrt{\frac{\mu^{2}r^{2}\kappa^{10}\log d_{1}}{pd_{2}}}\right)\sqrt{\sigma_{\max}}, \nonumber \\
& \frac{9\sqrt{\sigma_{\min}}}{10} \leq \sigma_{\min}\left(\overline{\Vector{X}}_k^{(l)}\right) \leq \sigma_{\max}\left(\overline{\Vector{X}}_k^{(l)}\right) \leq 2\sqrt{\sigma_{\max}}. \label{eqn-loo-op-norm-x}
\end{align}
Based on inequalities \eqref{eqn-loo-op-norm-y} and \eqref{eqn-loo-op-norm-x}, we obtain
\begin{multline}
\twonorm{a_{1}} \leq \left(1 - \frac{81s\sigma_{\min}}{10^{2}}\right)\twonorm{\left(\Matrix{\Delta}_{\Vector{X}}^{k,(l)}\right)_{l,\cdot}} \\
+ \frac{s\sigma_{\min}}{10}\left(120s\kappa\sigma_{\min} + 40\sqrt{\frac{\mu^{2}r^{2}\kappa^{10}\log d_{1}}{pd_{2}}}\right)\sqrt{\frac{\mu r\sigma_{\max}}{d_{1}}}. \label{eqn-a-1}
\end{multline}

On the other hand, for $a_{2}$, we have
\begin{multline}
a_{2} \leq \opnorm{\left(\Vector{O}_{k}^{(l)}\right)^{-1}\Vector{O}_{k+1}^{(l)} - \Vector{I}_{r}}\left(\twonorm{a_{1}} + \twonorm{\left(\Vector{X}_{\star}\right)_{l,\cdot}}\right).
\end{multline}
Consider the auxiliary sequence $\widetilde{\Vector{F}}_{k+1}$ defined in the proof of Lemma \ref{lemma-op-norm}. Then, according to \cite[(125)]{chen2020noisy}, we have
\begin{multline}
\opnorm{\left(\Vector{O}_{k}^{(l)}\right)^{-1}\Vector{O}_{k+1}^{(l)} - \Vector{I}_{r}} \\ \leq \frac{2}{\sigma_{\min}}\opnorm{\Vector{F}_{k+1}^{(l)}\Vector{O}_{k}^{(l)} - \widetilde{\Vector{F}}_{k+1}}\opnorm{\Vector{F}_{\star}}.
\end{multline}
From their respective iteration schemes, we can compute
\begin{multline}
\Vector{F}_{k+1}^{(l)}\Vector{O}_{k}^{(l)} - \widetilde{\Vector{F}}_{k+1}  = s\begin{bmatrix}
\Vector{D}^{(l)} & 0 \\
0 & \left(\Vector{D}^{(l)}\right)^\top
\end{bmatrix}\begin{bmatrix}
\Matrix{\Delta}_{\Vector{X}}^{k,(l)} \\ \Matrix{\Delta}_{\Vector{Y}}^{k,(l)}
\end{bmatrix} \\ + \frac{s}{2}\begin{bmatrix}
\Vector{X}_{\star} \\ \Vector{Y}_{\star}
\end{bmatrix}\left(\Vector{O}_{k}^{(l)}\right)^\top\Vector{B}^{(l)}\Vector{O}_{k}^{(l)},
\end{multline}
where
\begin{align*}
\Vector{D}^{(l)} = & -\left(p^{-1}\mathcal{P}_{\Omega_{-l,\cdot}} + \mathcal{P}_{l,\cdot}\right)\left(\Vector{X}\left(\Vector{Y}^{(l)}\right)^\top- \Vector{M}_{\star}\right), \\
\Vector{B}^{(l)} = & \left(\overline{\Vector{X}}_k^{(l)}\right)^\top\overline{\Vector{X}}_k^{(l)} - \left(\Vector{Y}^{(l)}\right)^\top\Vector{Y}^{(l)}.
\end{align*}
Thus, we have
\begin{multline} \label{eq:lemma16_total}
\opnorm{\Vector{F}_{k+1}^{(l)}\Vector{O}_{k}^{(l)} - \widetilde{\Vector{F}}_{k+1}} \\ \leq s\opnorm{\Vector{D}^{(l)}}\opnorm{\Matrix{\Delta}^{(l)}} + \frac{s}{2}\fnorm{\Vector{B}^{(l)}}\opnorm{\Vector{F}_{\star}}.
\end{multline}
From the discussion in \cite[D.6]{chen2020noisy}, we have
\begin{multline}
\opnorm{\Vector{D}^{(l)}} \lesssim \sqrt{\frac{d_{1}}{p}}\twoinfnorm{\Vector{F}_{k}^{(l)}\Vector{O}_{k}^{(l)} - \Vector{F}_{\star}}\twoinfnorm{\Vector{F}_{\star}} \\+ \opnorm{\Vector{F}_{k}^{(l)}\Vector{O}_{k}^{(l)} - \Vector{F}_{\star}}\opnorm{\Vector{F}_{\star}}.
\end{multline}
By Hypothesis \ref{hypothesis-induction}\ref{induction-op-norm}, \ref{induction-origin-loo-relation} and Lemma \ref{lemma-incoherence}, we have
\begin{align*}
&\twoinfnorm{\Vector{F}_{k}^{(l)}\Vector{O}_{k}^{(l)} - \Vector{F}_{\star}} \\
\leq & \twoinfnorm{\Vector{F}_{k}^{(l)}\Vector{O}_{k}^{(l)} - \Vector{F}_{k}\Vector{O}_{k}} + \twoinfnorm{\Vector{F}_{k}\Vector{O}_{k} - \Vector{F}_{\star}} \\
\leq & \fnorm{\Vector{F}_{k}^{(l)}\Vector{O}_{k}^{(l)} - \Vector{F}_{k}\Vector{O}_{k}} + \twoinfnorm{\Vector{F}_{k}\Vector{O}_{k} - \Vector{F}_{\star}} \\
\leq & 5\kappa\fnorm{\Vector{F}_{k}^{(l)}\Vector{R}_{k}^{(l)} - \Vector{F}_{k}\Vector{O}_{k}} + \twoinfnorm{\Vector{F}_{k}\Vector{O}_{k} - \Vector{F}_{\star}} \\
\leq & \sqrt{\frac{\sigma_{\max}}{d_{1}}}, 
\end{align*}
and
\begin{align*}
&\opnorm{\Vector{F}_{k}^{(l)}\Vector{O}_{k}^{(l)} - \Vector{F}_{\star}} \\\leq & 5\kappa\fnorm{\Vector{F}_{k}^{(l)}\Vector{R}_{k}^{(l)} - \Vector{F}_{k}\Vector{O}_{k}} + \opnorm{\Vector{F}_{k}\Vector{O}_{k} - \Vector{F}_{\star}} \\
\leq & \left(6s\sigma_{\min} + 2\sqrt{\frac{\mu^{2}r^{2}\kappa^{10}\log d_{1}}{pd_{2}}}\right)\sqrt{\sigma_{\max}}.
\end{align*}
Therefore, for $\opnorm{\Vector{D}^{(l)}}$, we have
\begin{equation}
\label{eqn-opnorm-d-l}
\opnorm{\Vector{D}^{(l)}} \lesssim \left(12s\sigma_{\min} + 5\sqrt{\frac{\mu^{2}r^{2}\kappa^{10}\log d_{1}}{pd_{2}}}\right)\sigma_{\max}.
\end{equation}
On the other hand, by the triangle inequality, $\fnorm{\Vector{B}^{(l)}}$ can be rewritten as
\begin{align*}
&\fnorm{\Vector{B}^{(l)}} \nonumber
\\= & \fnorm{\left(\overline{\Vector{X}}_k^{(l)}\Vector{R}_{t}^{(l)}\right)^\top\overline{\Vector{X}}_k^{(l)}\Vector{R}_{t}^{(l)} - \left(\Vector{Y}^{(l)}\Vector{R}_{t}^{(l)}\right)^\top\Vector{Y}^{(l)}\Vector{R}_{t}^{(l)}} \\
\leq & \fnorm{\left(\Vector{X}_{k}\Vector{O}_{k}\right)^\top\Vector{X}_{k}\Vector{O}_{k} - \left(\Vector{Y}_{k}\Vector{O}_{k}\right)^\top\Vector{Y}_{k}\Vector{O}_{k}} \\
& + \fnorm{\left(\Vector{X}_{k}^{(l)}\Vector{R}_{t}^{(l)}\right)^\top\Vector{X}_{k}^{(l)}\Vector{R}_{t}^{(l)} - \left(\Vector{X}_{k}\Vector{O}_{k}\right)^\top\Vector{X}_{k}\Vector{O}_{k}} \\
& +  \fnorm{\left(\Vector{Y}_{k}^{(l)}\Vector{R}_{t}^{(l)}\right)^\top\Vector{Y}_{k}^{(l)}\Vector{R}_{t}^{(l)} - \left(\Vector{Y}_{k}\Vector{O}_{k}\right)^\top\Vector{Y}_{k}\Vector{O}_{k}}.
\end{align*}
By Hypothesis \ref{hypothesis-induction}\ref{induction-origin-loo-relation}, we have
\begin{align*}
& \fnorm{\left(\Vector{X}_{k}^{(l)}\Vector{R}_{t}^{(l)}\right)^\top\Vector{X}_{k}^{(l)}\Vector{R}_{t}^{(l)} - \left(\Vector{X}_{k}\Vector{O}_{k}\right)^\top\Vector{X}_{k}\Vector{O}_{k}} \\
\leq & \left(\opnorm{\Vector{X}_{k}^{(l)}\Vector{R}_{t}^{(l)}} + \opnorm{\Vector{X}_{k}\Vector{O}_{k}}\right)\fnorm{\Vector{X}_{k}^{(l)}\Vector{R}_{t}^{(l)} - \Vector{X}_{k}\Vector{O}_{k}} \\
\leq & 4\left(\frac{s\sigma_{\min}}{\kappa} + \sqrt{\frac{\mu^{2}r^{2}\kappa^{10}\log d_{1}}{pd_{2}^{2}}}\right)\sigma_{\max}.
\end{align*}
This estimate also holds for 
\begin{equation}
\fnorm{\left(\Vector{Y}_{k}^{(l)}\Vector{R}_{t}^{(l)}\right)^\top\Vector{Y}_{k}^{(l)}\Vector{R}_{t}^{(l)} - \left(\Vector{Y}_{k}\Vector{O}_{k}\right)^\top\Vector{Y}_{k}\Vector{O}_{k}}.    
\end{equation}
Together with Lemma \ref{lemma-small-balance}, we obtain
\begin{equation}
\label{eqn-fnorm-b-l}
\fnorm{\Vector{B}^{(l)}}
\leq \frac{s\sigma_{\min}^{2}}{10^{2}\kappa} + 4\left(\frac{s\sigma_{\min}}{\kappa} + \sqrt{\frac{\mu^{2}r^{2}\kappa^{10}\log d_{1}}{pd_{2}^{2}}}\right)\sigma_{\max}.
\end{equation}
Combining the assumptions on $p$ and $s$ in \eqref{eqn-assumption-p-s} with inequalities \eqref{eq:lemma16_total}, \eqref{eqn-opnorm-d-l}, \eqref{eqn-fnorm-b-l}, we have
\begin{align*}
& \opnorm{\Vector{F}_{k+1}^{(l)}\Vector{O}_{k}^{(l)} - \widetilde{\Vector{F}}_{k+1}} \\
\leq & s\left(12s\sigma_{\min} + 5\sqrt{\frac{\mu^{2}r^{2}\kappa^{10}\log d_{1}}{pd_{2}}}\right)\sigma_{\max} \nonumber \\
&\cdot \left(s\sigma_{\min} + \sqrt{\frac{\mu r}{pd_{2}}}\right)\sqrt{\sigma_{\max}} \\
& + \frac{s\sqrt{\sigma_{\max}}}{\sqrt{2}}\left(\frac{s\sigma_{\min}^{2}}{10^{2}\kappa} + 4\left(\frac{s\sigma_{\min}}{\kappa} + \sqrt{\frac{\mu^{2}r^{2}\kappa^{10}\log d_{1}}{pd_{2}^{2}}}\right)\sigma_{\max}\right) \\
\leq & \frac{s\sigma_{\min}}{5}\left(120s\kappa\sigma_{\min} + 50\sqrt{\frac{\mu^{2}r^{2}\kappa^{12}\log d_{1}}{pd_{2}}}\right)\sqrt{\sigma_{\max}}.
\end{align*}
Thus, for $a_{2}$, we have the following upper bound
\begin{align}
\twonorm{a_{2}} \leq & \frac{2\sqrt{2\sigma_{\max}}}{\sigma_{\min}}\opnorm{\Vector{F}_{k+1}^{(l)}\Vector{O}_{k}^{(l)} - \widetilde{\Vector{F}}_{k+1}}\nonumber \\
&\cdot \left(\twonorm{a_{1}} + \twonorm{\left(\Vector{X}_{\star}\right)_{l,\cdot}}\right) \nonumber\\
\leq & \frac{s\sigma_{\min}}{5}\left(10^{3}s\kappa^{2}\sigma_{\min} + 50\sqrt{\frac{\mu^{2}r^{2}\kappa^{14}\log d_{1}}{pd_{2}}}\right)\nonumber \\
&\cdot\sqrt{\frac{\mu r\sigma_{\max}}{d_{1}}}. \label{eqn-a-2}
\end{align}

Finally, note that 
\begin{align*}
\twonorm{\left(\Vector{X}_{k}^{(l)}\right)_{l,\cdot}} &\leq \twonorm{\left(\Vector{X}_{k}^{(l)}\Vector{O}_{k}^{(l)} - \Vector{X}_{\star}\right)_{l,\cdot}} + \twonorm{\left(\Vector{X}_{\star}\right)_{l,\cdot}} \nonumber\\
&\leq 2\sqrt{\frac{\mu r\sigma_{\max}}{d_{2}}},
\end{align*}
so we obtain
\begin{align}
\twonorm{a_{3}} \leq & \frac{s}{2}\fnorm{\Vector{B}^{(l)}}\twonorm{\left(\Vector{X}_{k}^{(l)}\right)_{l,\cdot}} \nonumber \\
\leq & \frac{s}{2}\left(\frac{s\sigma_{\min}^{2}}{10^{2}\kappa} + 4\left(\frac{s\sigma_{\min}}{\kappa} + \sqrt{\frac{\mu^{2}r^{2}\kappa^{10}\log d_{1}}{pd_{2}^{2}}}\right)\sigma_{\max}\right)\nonumber \\
&\cdot 2\sqrt{\frac{\mu r\sigma_{\max}}{d_{1}}} \nonumber \\
\leq & \frac{s\sigma_{\min}}{10}\left(10^{3}s\kappa^{2}\sigma_{\min} + 50\sqrt{\frac{\mu^{2}r^{2}\kappa^{14}\log d_{1}}{pd_{2}}}\right) \nonumber \\
&\cdot\sqrt{\frac{\mu r\sigma_{\max}}{d_{2}}}. \label{eqn-a-3}
\end{align}
Combining inequalities \eqref{eqn-a-1}, \eqref{eqn-a-2}, \eqref{eqn-a-3} and $d_{1} \geq d_{2}$ yields the conclusion.

\end{proof}

\subsection{Inductive Step for Hypothesis \ref{hypothesis-induction}\ref{induction-origin-loo-relation}}

Lemma \ref{lemma-loo-perturbation} proves that the induction Hypothesis \ref{hypothesis-induction}\ref{induction-origin-loo-relation} remains valid at the $(k+1)$-th iteration.

\begin{lemma}
\label{lemma-loo-perturbation}
If Hypothesis \ref{hypothesis-induction} holds, and the assumptions on $p$ and $s$ in \eqref{eqn-assumption-p-s} are satisfied, then the following inequality holds with probability at least $1 - \left(d_{1} + d_{2}\right)^{-10}$
\begin{multline}
\fnorm{\Vector{F}_{k+1}\Vector{O}_{k+1} - \Vector{F}_{k+1}^{(l)}\Vector{R}_{k+1}^{(l)}} \\\leq \left(\frac{s\sigma_{\min}}{\kappa} + \sqrt{\frac{\mu^{2}r^{2}\kappa^{10}\log d_{1}}{pd_{2}^{2}}}\right)\sqrt{\sigma_{\max}}~.
\end{multline}
\end{lemma}

\begin{proof}
By the definition of $\Vector{R}_{k}^{(l)}$, we have
\begin{align*}
&\fnorm{\Vector{F}_{k+1}\Vector{O}_{k+1} - \Vector{F}_{k+1}^{(l)}\Vector{R}_{k+1}^{(l)}} \nonumber \\
\leq & \fnorm{\Vector{F}_{k+1}\Vector{O}_{k}\Vector{O}_{k}^\top\Vector{O}_{k+1} - \Vector{F}_{k+1}^{(l)}\Vector{R}_{k}^{(l)}\Vector{O}_{k}^\top\Vector{O}_{k+1}} \\
\leq & \fnorm{\Vector{F}_{k+1}\Vector{O}_{k} - \Vector{F}_{k+1}^{(l)}\Vector{R}_{k}^{(l)}}.
\end{align*}
From the iterative formulas of $\Vector{F}_{k+1}$ and $\Vector{F}_{k+1}^{(l)}$, it follows that
\begin{align*}
&\Vector{F}_{k+1}\Vector{O}_{k} - \Vector{F}_{k+1}^{(l)}\Vector{R}_{k}^{(l)} \nonumber\\
= & \left(\Vector{F}_{k} - s\nabla f(\Vector{F}_{k})\right)\Vector{O}_{k} - \left(\Vector{F}_{k}^{(l)} - s\nabla f^{(l)}(\Vector{F}_{k})\right)\Vector{R}_{k}^{(l)} \\
= & \underbrace{\Vector{F}_{k}\Vector{O}_{k} - \Vector{F}_{k}^{(l)}\Vector{R}_{k}^{(l)} - s\left(\nabla \fbal(\Vector{F}_{k}\Vector{O}_{k}) - \nabla \fbal(\Vector{F}_{k}^{(l)}\Vector{R}_{k}^{(l)})\right)}_{\Vector{A}_{1}} \\
& + \underbrace{s\nabla f_{\text{diff}}\left(\Vector{F}_{k}\Vector{O}_{k}\right)}_{\Vector{A}_{2}} - \underbrace{s\left(\nabla \fbal(\Vector{F}_{k}^{(l)}\Vector{R}_{k}^{(l)}) - \nabla \fbal^{(l)}(\Vector{F}_{k}^{(l)}\Vector{R}_{k}^{(l)})\right)}_{\Vector{A}_{3}},
\end{align*}
where the second equality holds because we have $\nabla f(\Vector{F})\Vector{O}=\nabla f(\Vector{F}\Vector{O})$ for any $\Vector{F}\in\mathbb{R}^{(d_{1} + d_{2})\times r}$ and $\Vector{O}\in\mathcal{O}_r$,, and similarly for $\fbal$ and $\fbal^{(l)}$.

By the Newton-Leibniz theorem, we obtain
\begin{align*}
\mathrm{vec}(\Vector{A}_{1}) = & \mathrm{vec}\left(\Vector{F}_{k}\Vector{O}_{k} - \Vector{F}_{k}^{(l)}\Vector{R}_{k}^{(l)}\right) \nonumber\\
& - s\cdot\mathrm{vec}\left(\nabla \fbal(\Vector{F}_{k}\Vector{O}_{k}) - \nabla \fbal(\Vector{F}_{k}^{(l)}\Vector{R}_{k}^{(l)})\right) \\
= & \left(\Vector{I}_{(d_{1}+d_{2})r} - s\int_{0}^{1}\nabla \fbal\left(\Vector{F}(\tau)\right)\mathrm{d}\tau\right) \nonumber\\
&\cdot \mathrm{vec}\left(\Vector{F}_{k}\Vector{O}_{k} - \Vector{F}_{k}^{(l)}\Vector{R}_{k}^{(l)}\right),
\end{align*}
where
\begin{equation}
    \Vector{F}(\tau) = \Vector{F}_{k}^{(l)}\Vector{R}_{k}^{(l)} + \tau\left( \Vector{F}_{k}\Vector{O}_{k} - \Vector{F}_{k}^{(l)}\Vector{R}_{k}^{(l)}\right).
\end{equation}

Let $\Vector{J}=\int_{0}^{1}\nabla \fbal\left(\Vector{F}(\tau)\right)\mathrm{d}\tau$. Then we get
\begin{align*}
\fnorm{\Vector{A}_{1}}^{2} = & \left(\mathrm{vec}\left(\Vector{F}_{k}\Vector{O}_{k} - \Vector{F}_{k}^{(l)}\Vector{R}_{k}^{(l)}\right)\right)^\top\left(\Vector{I}_{(d_{1} + d_{2})r} - s\Vector{J}\right)^{2} \nonumber\\
&\cdot \mathrm{vec}\left(\Vector{F}_{k}\Vector{O}_{k} - \Vector{F}_{k}^{(l)}\Vector{R}_{k}^{(l)}\right) \\
\leq & \left(1 + s^{2}\opnorm{\Vector{J}}^{2}\right)\fnorm{\Vector{F}_{k}\Vector{O}_{k} - \Vector{F}_{k}^{(l)}\Vector{R}_{k}^{(l)}}^{2} \\
& - 2s\left(\mathrm{vec}\left(\Vector{F}_{k}\Vector{O}_{k} - \Vector{F}_{k}^{(l)}\Vector{R}_{k}^{(l)}\right)\right)^\top\nonumber\\
&\cdot \Vector{J}\mathrm{vec}\left(\Vector{F}_{k}\Vector{O}_{k} - \Vector{F}_{k}^{(l)}\Vector{R}_{k}^{(l)}\right).
\end{align*}
Note that by Lemma \ref{lemma-incoherence}, Hypothesis \ref{hypothesis-induction}\ref{induction-origin-loo-relation}, and the conditions on $p$ and $s$ in \eqref{eqn-assumption-p-s}, we have
\begin{align*}
&\twoinfnorm{\Vector{F}(\tau) - \Vector{F}_{\star}} \nonumber\\
\leq & \tau\twoinfnorm{\Vector{F}_{k}\Vector{O}_{k} - \Vector{F}_{\star}} + (1-\tau)\twoinfnorm{\Vector{F}_{k}^{(l)}\Vector{R}_{k}^{(l)} - \Vector{F}_{\star}} \\
\leq & \twoinfnorm{\Vector{F}_{k}\Vector{O}_{k} - \Vector{F}_{\star}} + \fnorm{\Vector{F}_{k}^{(l)}\Vector{R}_{k}^{(l)} - \Vector{F}_{\star}} \\
\leq & \frac{\sqrt{\sigma_{\max}}}{500\kappa\sqrt{d_{1} + d_{2}}}~,
\end{align*}
Thus, $\Vector{F}(\tau)$ and $\Vector{D}_{\Vector{F}}\triangleq\Vector{F}_{k}^{(l)}\Vector{R}_{k}^{(l)} - \Vector{F}_{k}\Vector{O}_{k}$ satisfy the conditions of Lemma \ref{lemma-hessian}. Therefore, $\opnorm{\Vector{J}}\leq 5\sigma_{\max}$, and
\begin{multline*}
\left(\mathrm{vec}\left(\Vector{F}_{k}\Vector{O}_{k} - \Vector{F}_{k}^{(l)}\Vector{R}_{k}^{(l)}\right)\right)^\top\Vector{J}\mathrm{vec}\left(\Vector{F}_{k}\Vector{O}_{k} - \Vector{F}_{k}^{(l)}\Vector{R}_{k}^{(l)}\right) \\ \geq \frac{\sigma_{\min}}{10}\fnorm{\Vector{F}_{k}\Vector{O}_{k} - \Vector{F}_{k}^{(l)}\Vector{R}_{k}^{(l)}}^{2}.
\end{multline*}
Hence, when $s\leq \frac{1}{250\kappa\sigma_{\max}}$, we have
\begin{align}
\fnorm{\Vector{A}_{1}} \leq & \left(1 + 25s^{2}\sigma_{\max}^{2} -\frac{s\sigma_{\min}}{5} \right)\fnorm{\Vector{F}_{k}\Vector{O}_{k} - \Vector{F}_{k}^{(l)}\Vector{R}_{k}^{(l)}}^{2} \nonumber \\
\leq & \left(1 -\frac{s\sigma_{\min}}{10} \right)\fnorm{\Vector{F}_{k}\Vector{O}_{k} - \Vector{F}_{k}^{(l)}\Vector{R}_{k}^{(l)}}^{2}. \label{eqn-fnorm-a-1}
\end{align}
By Lemma \ref{lemma-small-balance} and inequality \eqref{eqn-f-norm-initial}, we obtain
\begin{equation}
\label{eqn-fnorm-a-2}
\fnorm{\Vector{A}_{2}} \leq \frac{s}{2}\sqrt{2\sigma_{\max}}\frac{s\sigma_{\min}^{2}}{10^{2}\kappa} \leq \frac{s\sigma_{\min}}{20}\frac{s\sigma_{\min}\sqrt{\sigma_{\max}}}{\kappa}~.
\end{equation}
Finally, according to \cite[Assertion 5, Assertion 6]{chen2020noisy}, the following inequality holds with probability at least $1 - \left(d_{1} + d_{2}\right)^{-10}$:
\begin{align}
\fnorm{\Vector{A}_{3}} \lesssim & s\sqrt{\frac{\mu^{2}r^{2}\log d_{1}}{pd_{2}}}\twoinfnorm{\Vector{F}_{k}^{(l)}\Vector{R}_{k}^{(l)} - \Vector{F}_{\star}}\sigma_{\max} \nonumber \\
\leq & s\sqrt{\frac{\mu^{2}r^{2}\log d_{1}}{pd_{2}}}\sigma_{\max}\bigg(\twoinfnorm{\Vector{F}_{k}\Vector{O}_{k} - \Vector{F}_{\star}} \nonumber\\
&+ \fnorm{\Vector{F}_{k}\Vector{O}_{k} - \Vector{F}_{k}^{(l)}\Vector{R}_{k}^{(l)}}\bigg) \nonumber \\
\leq & \frac{s\sigma_{\min}}{20}\sqrt{\frac{\mu^{2}r^{2}\sigma_{\max}\log d_{1}}{pd_{2}^{2}}} \nonumber\\
&+ \frac{s\sigma_{\min}}{20}\fnorm{\Vector{F}_{k}\Vector{O}_{k} - \Vector{F}_{k}^{(l)}\Vector{R}_{k}^{(l)}}. \label{eqn-fnorm-a-3}
\end{align}
Combining inequalities \eqref{eqn-fnorm-a-1}, \eqref{eqn-fnorm-a-2}, and \eqref{eqn-fnorm-a-3} yields the desired conclusion.
\end{proof}

\subsection{Inductive Step for Hypothesis \ref{hypothesis-induction}\ref{induction-linear-convergence}}

This subsection analyzes the Hypothesis \ref{hypothesis-induction}\ref{induction-linear-convergence}. It can be easily verified that by taking expectation over the observable index set $\Omega$, we have
\begin{multline*}
\mathbb{E}\left[\frac{1}{2p}\fnorm{\mathcal{P}_{\Omega}\left(\Vector{X}\Vector{Y}^\top- \Vector{M}_{\star}\right)}^{2}\right] = \frac{1}{2}\fnorm{\Vector{X}\Vector{Y}^\top- \Vector{M}_{\star}}^{2}.
\end{multline*}
This indicates that in expectation, the matrix completion problem \eqref{eqn-matrix-completion} reduces to a low-rank matrix factorization problem. The gradient descent iteration for solving this problem is given by
\begin{equation}
\label{eq-mf-gd}
\begin{cases}
\Vector{X}_{k+1} = \Vector{X}_{k} - s\left(\Vector{X}_{k}\Vector{Y}_{k}^\top- \Vector{M}_{\star}\right)\Vector{Y}_{k}~, \\
\Vector{Y}_{k+1} = \Vector{Y}_{k} - s\left(\Vector{X}_{k}\Vector{Y}_{k}^\top- \Vector{M}_{\star}\right)^\top\Vector{X}_{k}~.
\end{cases}
\end{equation}
Lemma \ref{lemma-mf-linear-convergence} establishes the local linear convergence rate of \eqref{eq-mf-gd}.

\begin{lemma}[\cite{ma20121beyond}]
\label{lemma-mf-linear-convergence}
If there exists a sufficiently small $c_{0}>0$ such that the initial point $\Vector{F}_{0}=[\Vector{X}_{0}^\top, \Vector{Y}_{0}^\top]^\top$ satisfies
\begin{equation}
\min_{\Vector{O}\in\mathcal{O}_r}\fnorm{\Vector{F}_{0}\Vector{O} - \Vector{F}_{\star}} \leq c_{0}\frac{1}{\kappa^{3/2}}\sqrt{\sigma_{\min}}~;
\end{equation}
and the optimal alignment matrix $\Vector{Q}_{k}$ between $\Vector{F}_{k}$ and $\Vector{F}_{\star}$ exists with some orthogonal matrix $\widehat{\Vector{O}}\in\mathcal{O}_r$ satisfying
\begin{equation}
\opnorm{\Vector{Q}_{k} - \widehat{\Vector{O}}} \leq \frac{1}{400\sqrt{\kappa}}~;
\end{equation}
then under the step size condition $0<s\leq\frac{1}{24\sigma_{\max}}$, the following inequality holds for $\Vector{F}_{k+1}$
\begin{multline*}
\fnorm{\Vector{X}_{k+1}\Vector{Q}_{k} - \Vector{X}_{\star}}^{2} + \fnorm{\Vector{Y}_{k+1}\Vector{Q}_{k}^{-\top} - \Vector{Y}_{\star}}^{2} \\ \leq \left(1 - \frac{s\sigma_{\min}}{24}\right)\dist{\Vector{F}_{k}}{\Vector{F}_{\star}}.
\end{multline*}
\end{lemma}

Using Lemma \ref{coro-bounded-2-inf-norm}, we can prove that the hypothesis \ref{induction-linear-convergence} holds at the $(k+1)$-th iteration with high probability.

\begin{lemma}
\label{lemma-linear-convergence}
If Hypothesis \ref{hypothesis-induction} and the assumptions on $p$ and $s$ in \eqref{eqn-assumption-p-s} hold, then $\Vector{F}_{k+1}$ satisfies
\begin{align}
\dist{\Vector{F}_{k+1}}{\Vector{F}_{\star}} \leq \left(1 - \frac{s\sigma_{\min}}{100}\right)^{k+1}\dist{\Vector{F}_{0}}{\Vector{F}_{\star}}.
\end{align}
\end{lemma}
\begin{proof}
    See Appendix \ref{append:lemma-linear-convergence} of the supplementary material. 
\end{proof}

\subsection{Inductive Step for Hypothesis \ref{hypothesis-induction}\ref{induction-bam-exist}}
Finally, we analyze the existence and spectral properties of the optimal alignment matrix $\Vector{Q}_{k+1}$ for Hypothesis \ref{hypothesis-induction}\ref{induction-bam-exist}.

\begin{lemma}
\label{lemma-alignment-matrix}
If Hypothesis \ref{hypothesis-induction}\ref{induction-bam-exist}, Lemma \ref{lemma-linear-convergence}, and assumptions on $p$, $s$ in \eqref{eqn-assumption-p-s} hold, then the optimal transport matrix $\Vector{Q}_{k+1}$ between $\Vector{F}_{k+1}$ and $\Vector{F}_{\star}$ exists with
\begin{equation}
\label{eqn-alignment-matrix-op-norm}
\opnorm{\Vector{Q}_{k+1} - \Vector{O}_{k+1}} \leq \frac{1}{400\kappa}~.
\end{equation}
\end{lemma}

\begin{proof}
Combining the spectral bound $\sigma_{\min}(\Vector{X}_{k+1})\geq\frac{\sqrt{\sigma_{\min}}}{2}$ from Lemma \ref{lemma-op-norm} with the convergence results in Lemma \ref{lemma-linear-convergence}, we derive through perturbation analysis
\begin{multline*}
\opnorm{\Vector{Q}_{k+1} - \Vector{O}_{k+1}}  \\
\leq \frac{1}{\sigma_{\min}(\Vector{X}_{k+1})}\opnorm{\Vector{X}_{k+1}\Vector{Q}_{k+1} - \Vector{X}_{k+1}\Vector{O}_{k+1}} \\
\leq \frac{2}{\sqrt{\sigma_{\min}}}\left(\opnorm{\Vector{X}_{k+1}\Vector{Q}_{k+1} - \Vector{X}_{\star}} + \opnorm{\Vector{X}_{k+1}\Vector{O}_{k+1} - \Vector{X}_{\star}}\right).
\end{multline*}

On the other hand, by Lemma \ref{lemma-linear-convergence} we have
\begin{multline*}
\opnorm{\Vector{X}_{k+1}\Vector{Q}_{k+1} - \Vector{X}_{\star}} \leq  \fnorm{\Vector{X}_{k+1}\Vector{Q}_{k+1} - \Vector{X}_{\star}}
\\ \leq  \dist{\Vector{F}_{k+1}}{\Vector{F}_{\star}} \leq \dist{\Vector{F}_{0}}{\Vector{F}_{\star}} \leq \frac{c_{0}\sqrt{\sigma_{\min}}}{\kappa^{3/2}}.
\end{multline*}
According to Lemma \ref{lemma-op-norm}, we get
\begin{multline*}
\opnorm{\Vector{X}_{k+1}\Vector{O}_{k+1} - \Vector{X}_{\star}} \leq \opnorm{\Vector{F}_{k+1}\Vector{O}_{k+1} - \Vector{F}_{\star}} \\
\leq \left(s\sigma_{\min} + \sqrt{\frac{\mu r\kappa^{6}\log d_{1}}{pd_{2}}}\right)\sqrt{\sigma_{\max}}~.
\end{multline*}

The conclusion follows from combining the step size condition $s\leq\frac{1}{24\sigma_{\max}}$, sampling requirement $p\geq\frac{\mu r^{2}\kappa^{10}\log d_{1}}{d_{2}}$ and the above inequalities.
\end{proof}

\section{Auxiliary Lemmas}

\begin{lemma}[\cite{zheng2016convergence}]
\label{lemma-rip-subspace}
If the matrix $\Vector{M}_{\star}$ is $\mu$-incoherent, and the sampling rate satisfies $p\gtrsim\frac{\mu r\log (\max\{d_{1}, d_{2}\})}{\min\{d_{1}, d_{2}\}}$, then the following inequality holds with high probability
\begin{multline}
\left\vert\left\langle\left(p^{-1}\mathcal{P}_\Omega - \mathcal{I}\right)\left(\Vector{X}_{\star}\Vector{Y}_{A}^\top + \Vector{X}_{A}\Vector{Y}_{\star}^\top\right), \Vector{X}_{\star}\Vector{Y}_{B}^\top + \Vector{X}_{B}\Vector{Y}_{\star}^\top\right\rangle\right\vert \leq \\
C_{1}\sqrt{\frac{\mu r\log (\max\{d_{1}, d_{2}\})}{p\min\{d_{1}, d_{2}\}}} \fnorm{\Vector{X}_{\star}\Vector{Y}_{A}^\top + \Vector{X}_{A}\Vector{Y}_{\star}^\top}\\
 \cdot \fnorm{\Vector{X}_{\star}\Vector{Y}_{B}^\top + \Vector{X}_{B}\Vector{Y}_{\star}^\top},
\end{multline}
where $\Vector{X}_{A},\Vector{X}_{B}\in\mathbb{R}^{d_{1}\times r}$, $\Vector{Y}_{A},\Vector{Y}_{B}\in\mathbb{R}^{d_{2}\times r}$ and $C_{1}>0$ is a constant.
\end{lemma}

\begin{lemma}[\cite{chen2019model, chen2020nonconvex}]
\label{lemma-rip-all-space}
If the matrix $\Vector{M}_{\star}$ is $\mu$-incoherent, and the sampling rate satisfies $p\gtrsim\frac{\log (\max\{d_{1}, d_{2}\})}{\min\{d_{1}, d_{2}\}}~$, then the following inequality holds with high probability
\begin{multline}
\left\vert\left\langle\left(p^{-1}\mathcal{P}_\Omega - \mathcal{I}\right)\left(\Vector{X}_{A}\Vector{Y}_{A}^\top\right), \Vector{X}_{B}\Vector{Y}_{B}^\top \right\rangle\right\vert \leq 
C_{2}\sqrt{\frac{\max\{d_{1}, d_{2}\}}{p}} \\ \cdot \min\left\{\fnorm{\Vector{X}_{A}}\twoinfnorm{\Vector{X}_{B}}, \twoinfnorm{\Vector{X}_{A}}\fnorm{\Vector{X}_{B}}\right\}\\
\cdot \min\left\{\fnorm{\Vector{Y}_{A}}\twoinfnorm{\Vector{Y}_{B}}, \twoinfnorm{\Vector{Y}_{A}}\fnorm{\Vector{Y}_{B}}\right\},
\end{multline}
where $\Vector{X}_{A},\Vector{X}_{B}\in\mathbb{R}^{d_{1}\times r}$, $\Vector{Y}_{A},\Vector{Y}_{B}\in\mathbb{R}^{d_{2}\times r}$ and $C_{2}>0$ is a constant.
\end{lemma}

\begin{lemma}[\cite{chen2020nonconvex}]
\label{lemma-hessian}
If there exists a suitable constant $C_{3}>0$ such that the sampling probability $p$ satisfies
\[
p\geq \frac{C_{3}\mu r\kappa\log d_{1}}{d_{2}}~,
\]
then when the random event $\mathrm{E}_{\mathrm{RIP}}$ holds, the following inequalities concerning $\nabla^{2} f_{\mathrm{bal}}(\Vector{X}, \Vector{Y})$ are valid
\begin{align}
\mathrm{vec}\left(\begin{bmatrix}
\Vector{D}_{\Vector{X}} \\ \Vector{D}_{\Vector{Y}}
\end{bmatrix}\right)^\top\nabla^{2}f_{\mathrm{bal}}(\Vector{X}, \Vector{Y})\mathrm{vec}\left(\begin{bmatrix}
\Vector{D}_{\Vector{X}} \\ \Vector{D}_{\Vector{Y}}
\end{bmatrix}\right) \nonumber\\
\geq \frac{\sigma_{\min}}{5}\fnorm{\begin{bmatrix}
\Vector{D}_{\Vector{X}} \\ \Vector{D}_{\Vector{Y}}
\end{bmatrix}}^{2}, \\
\opnorm{\nabla^{2}f_{\mathrm{bal}}(\Vector{X}, \Vector{Y})} \leq 5\sigma_{\max},
\end{align}
where $\Vector{X}$ and $\Vector{Y}$ satisfy
\begin{equation}
 \twoinfnorm{\begin{bmatrix}
\Vector{X} - \Vector{X}_{\star} \\ \Vector{Y} - \Vector{Y}_{\star}
\end{bmatrix}} \leq \frac{1}{500\kappa\sqrt{d_{1}+d_{2}}}\sqrt{\sigma_{\max}};
\end{equation}
and $\Vector{D}_{\Vector{X}}$, $\Vector{D}_{\Vector{Y}}$ belong to the following set
\begin{multline}
    \Bigg\{\begin{bmatrix}
\widetilde{\Vector{X}}_{1} \\ \widetilde{\Vector{Y}}_{1}
\end{bmatrix}\widetilde{\Vector{O}} - \begin{bmatrix}
\widetilde{\Vector{X}}_{2} \\ \widetilde{\Vector{Y}}_{2}
\end{bmatrix}: \opnorm{\begin{bmatrix}
\widetilde{\Vector{X}}_{2} - \Vector{X}_{\star}\\ \widetilde{\Vector{Y}}_{2} - \Vector{Y}_{\star}
\end{bmatrix}}\leq \frac{\sqrt{\sigma_{\max}}}{500\kappa}, \\ \widetilde{\Vector{O}}=\mathop{\arg\min}_{\Vector{O}\in\mathcal{O}_r}\fnorm{\begin{bmatrix}
\widetilde{\Vector{X}}_{1} \\ \widetilde{\Vector{Y}}_{1}
\end{bmatrix}\Vector{O} - \begin{bmatrix}
\widetilde{\Vector{X}}_{2} \\ \widetilde{\Vector{Y}}_{2}
\end{bmatrix}}\Bigg\}.
\end{multline}
\end{lemma}

\begin{lemma}[\cite{chen2020noisy,ma2020implicit}]
\label{lemma-alignment-bound}
Let $\Vector{T}_{1}$ and $\Vector{T}_{2}$ the optimal rotation matrices between $\Vector{A}_{1} \in\mathbb{R}^{d\times r}$ and $\Vector{A}_{0} \in\mathbb{R}^{d\times r}$, and between $\Vector{A}_{2} \in\mathbb{R}^{d\times r} $ and $\Vector{A}_{0}$ respectively, i.e.,
\begin{align}
\Vector{T}_{1}\triangleq\mathop{\arg\min}_{\Vector{O}\in\mathcal{O}_r}\fnorm{\Vector{A}_{1}\Vector{O} - \Vector{A}_{0}}, \\ \Vector{T}_{2}\triangleq\mathop{\arg\min}_{\Vector{O}\in\mathcal{O}_r}\fnorm{\Vector{A}_{2}\Vector{O} - \Vector{A}_{0}}.
\end{align}
If $\Vector{A}_{0}, \Vector{A}_{1}$ and $\Vector{A}_{2}$ satisfy
\begin{align}
    \opnorm{\Vector{A}_{1} - \Vector{A}_{2}}\opnorm{\Vector{A}_{0}} &\leq \frac{\sigma_{r}^{2}(\Vector{A}_{0})}{4},\\ 
    \opnorm{\Vector{A}_{1} - \Vector{A}_{0}}\opnorm{\Vector{A}_{0}} &\leq \frac{\sigma_{r}^{2}(\Vector{A}_{0})}{2}~, 
\end{align}
then the following inequalities hold
\begin{align}
    \fnorm{\Vector{A}_{1}\Vector{T}_{1} - \Vector{A}_{2}\Vector{T}_{2}} &\leq 5\kappa\fnorm{\Vector{A}_{1} - \Vector{A}_{2}},\\ \opnorm{\Vector{A}_{1}\Vector{T}_{1} - \Vector{A}_{2}\Vector{T}_{2}} &\leq 5\kappa\opnorm{\Vector{A}_{1} - \Vector{A}_{2}}.
\end{align}
\end{lemma}

\begin{lemma}[\cite{chen2020nonconvex}] 
\label{lemma-initial-property}  
If there exists a suitbale constant $ C_3 > 0 $ such that the sampling probability $ p $ satisfies  
\begin{equation}
    p \geq C_3 \frac{\mu^2 r^2 \kappa^6 \log d_1}{d_2},
\end{equation}
and we define the random event $ \mathrm{E}_{\mathrm{init}} $ as the occurrence of the following inequalities:  
\begin{align}
\opnorm{\Vector{F}_0 \Vector{O}_0 - \Vector{F}_\star} &\leq C_6 \sqrt{\frac{\mu r \kappa^6 \log d_1}{p d_2}} \sqrt{\sigma_{\max}}, \label{eqn-op-norm-initial} \\
\twonorm{\left(\Vector{F}_0^{(l)} \Vector{O}_0^{(l)} - \Vector{F}_\star\right)_{l, \cdot}} &\leq 10^2 C_6 \sqrt{\frac{\mu^2 r^2 \kappa^7 \log d_1}{p d_2^2}} \sqrt{\sigma_{\max}}, \nonumber\\
&\qquad  \qquad \forall 1 \leq l \leq d_1 + d_2, \label{eqn-l-2-norm-initial} \\
\fnorm{\Vector{F}_0 \Vector{O}_0 - \Vector{F}_0^{(l)} \Vector{R}_0^{(l)}} &\leq C_6 \sqrt{\frac{\mu^2 r^2 \kappa^{10} \log d_1}{p d_2^2}} \sqrt{\sigma_{\max}}, \nonumber\\
&\qquad \qquad \forall 1 \leq l \leq d_1 + d_2, \label{eqn-f-norm-loo-initial}
\end{align}  
where $ C_6 $ is a fixed constant, then $ \mathrm{E}_{\mathrm{init}} \subset \mathrm{E}_{\mathrm{RIP}} $, and  
\begin{align}
    \mathbb{P}\left[\mathrm{E}_{\mathrm{init}}\right] \geq 1 - (d_1 + d_2)^{-10}.
\end{align}
\end{lemma}
 
\begin{lemma}[\cite{ma20121beyond}]
\label{lemma-bam-exist-sufficient}  
For a matrix $ \Vector{F} = [\Vector{X}^\top, \Vector{Y}^\top]^\top $, if there exists an invertible matrix $ \Vector{P} \in \mathbb{R}^{r \times r} $ satisfying $ \frac{1}{2} \leq \sigma_{\min}(\Vector{P}) \leq \sigma_{\max}(\Vector{P}) \leq \frac{3}{2} $, and a $ \delta > 0 $ such that  
\begin{equation}
    \max\left\{\fnorm{\Vector{X} \Vector{P} - \Vector{X}_\star}, \fnorm{\Vector{Y} \Vector{P}^{-\top} - \Vector{Y}_\star}\right\} \leq \delta \leq \frac{\sqrt{\sigma_{\min}}}{80},
\end{equation}
then the optimal alignment matrix $ \Vector{Q} $ between $ \Vector{F} $ and $ \Vector{F}_\star $ exists, and  
\begin{equation}
    \opnorm{\Vector{Q} - \Vector{P}} \leq \fnorm{\Vector{Q} - \Vector{P}} \leq \frac{5 \delta}{\sqrt{\sigma_{\min}}}.
\end{equation}
\end{lemma}

\section{Proofs of Lemmas in Theorem \ref{thm-linear-convergence}}

\subsection{Proof of Lemma \ref{lemma-relation-o-and-r}} \label{Appedd:relation-o-and-r}
Let $\Vector{A}_{0} = \Vector{F}_{\star}$, $\Vector{A}_{1} = \Vector{F}_{k}\Vector{O}_{k}$, and $\Vector{A}_{2}=\Vector{F}_{k}^{(l)}\Vector{R}_{k}^{(l)}$. By the definitions of $\Vector{O}_{k}$ and $\Vector{O}_{k}^{(l)}$, we have
\begin{align}
    \Vector{T}_{1} &= \mathop{\arg\min}_{\Vector{O}\in\mathcal{O}_r}\fnorm{\Vector{A}_{1}\Vector{O} - \Vector{A}_{0}} = \Vector{I}_{r},\\
    \Vector{T}_{2}& =\mathop{\arg\min}_{\Vector{O}\in\mathcal{O}_r}\fnorm{\Vector{A}_{2}\Vector{O} - \Vector{A}_{0}}=\left(\Vector{R}_{k}^{(l)}\right)^{-1}\Vector{O}_{k}^{(l)},
\end{align}
where $\Vector{I}_{r}$ denotes the $r\times r$ identity matrix. Furthermore, from the definition of $\Vector{F}_{\star}$, it follows that $\opnorm{\Vector{A}_{0}}=\sqrt{2\sigma_{\max}}$ and $\sigma_{r}(\Vector{A}_{0})=\sqrt{2\sigma_{\min}}$. Combining the induction hypotheses \ref{induction-op-norm} and \ref{induction-origin-loo-relation} with assumption \eqref{eqn-assumption-p-s}, we obtain
\begin{align}
&\opnorm{\Vector{A}_{1} - \Vector{A}_{0}}\opnorm{\Vector{A}_{0}} \nonumber \\ \leq &\left(s\sigma_{\min} + \sqrt{\frac{\mu r\kappa^{6}\log d_{1}}{pd_{2}}}\right)\sqrt{\sigma_{\max}}\sqrt{2\sigma_{\max}} \nonumber\\
\leq& \sigma_{\min} = \frac{\sigma_{r}^{2}(\Vector{A}_{0})}{2}, \\
&\opnorm{\Vector{A}_{1} - \Vector{A}_{2}}\opnorm{\Vector{A}_{0}} 
\leq \fnorm{\Vector{A}_{1} - \Vector{A}_{2}}\opnorm{\Vector{A}_{0}} \\
\leq& \left(\frac{s\sigma_{\min}}{\kappa} + \sqrt{\frac{\mu^{2}r^{2}\kappa^{10}\log d_{1}}{pd_{2}^{2}}}\right)\sqrt{\sigma_{\max}}\sqrt{2\sigma_{\max}} \nonumber\\
\leq& \frac{\sigma_{\min}}{2} = \frac{\sigma_{r}^{2}(\Vector{A}_{0})}{4}.
\end{align}
The conclusion then follows directly from Lemma \ref{lemma-alignment-bound}.

\subsection{Proof of Lemma \ref{lemma-incoherence}} \label{Append:lemma-incoherence}

For $1\leq l\leq d_{1}$, by the triangle inequality we have
\begin{multline}
\twonorm{\left(\Vector{X}_{k}\Vector{O}_{k} - \Vector{X}_{\star}\right)_{l,\cdot}} \leq \twonorm{\left(\Vector{X}_{k}\Vector{O}_{k} - \Vector{X}_{k}^{(l)}\Vector{O}_{k}^{(l)}\right)_{l,\cdot}} \\+ \twonorm{\left(\Vector{X}_{k}^{(l)}\Vector{O}_{k}^{(l)} - \Vector{X}_{\star}\right)_{l,\cdot}}. \label{eqn:triangle-ineq}
\end{multline}
Moreover, Lemma \ref{lemma-relation-o-and-r} yields
\begin{multline}
\twonorm{\left(\Vector{X}_{k}\Vector{O}_{k} - \Vector{X}_{k}^{(l)}\Vector{O}_{k}^{(l)}\right)_{l,\cdot}} \leq \fnorm{\Vector{F}_{k}\Vector{O}_{k} - \Vector{F}_{k}^{(l)}\Vector{O}_{k}^{(l)}} \\ \leq 5\kappa\fnorm{\Vector{F}_{k}\Vector{O}_{k} - \Vector{F}_{k}^{(l)}\Vector{R}_{k}^{(l)}}. \label{eqn:lemma-bound}
\end{multline}

Therefore, combining induction hypotheses \ref{induction-l-2-norm} and \ref{induction-origin-loo-relation} with $\mu r\geq 1$, $\kappa \geq 1$ and $d_{1}\geq d_{2}$, we obtain
\begin{align}
&\twonorm{\left(\Vector{X}_{k}\Vector{O}_{k} - \Vector{X}_{\star}\right)_{l,\cdot}} \nonumber\\
\leq &5\kappa\fnorm{\Vector{F}_{k}\Vector{O}_{k} - \Vector{F}_{k}^{(l)}\Vector{R}_{k}^{(l)}} + \twonorm{\left(\Vector{X}_{k}^{(l)}\Vector{O}_{k}^{(l)} - \Vector{X}_{\star}\right)_{l,\cdot}} \nonumber\\
\leq &\bigg((10^{3} + 5)s\kappa^{2}\sigma_{\min} + \nonumber\\
&~~~~~~~(10^{2} + 5)\sqrt{\frac{\mu^{2}r^{2}\kappa^{14}\log d_{1}}{pd_{2}}}\bigg)\sqrt{\frac{\mu r\sigma_{\max}}{d_{2}}}, \label{eqn:final-X-bound}
\end{align}
which holds for all $1\leq l\leq d_{1}$. The upper bound for $\twoinfnorm{\Vector{Y}_{k}\Vector{O}_{k} - \Vector{Y}_{\star}}$ can be derived similarly.

\subsection{Proof of Lemma \ref{coro-bounded-2-inf-norm}} \label{Append:coro-bounded-2-inf-norm}

First observe that
\begin{align}
\twoinfnorm{\Vector{X}_{k}} &\leq \twoinfnorm{\Vector{X}_{k}\Vector{O}_{k}}\opnorm{\Vector{O}_{k}^\top}  \nonumber\\
&\leq \twoinfnorm{\Vector{X}_{k}\Vector{O}_{k} - \Vector{X}_{\star}} + \twoinfnorm{\Vector{X}_{\star}}. \label{eqn:X-decomp}
\end{align}
From the definition of $\Vector{X}_{\star}$, we have
\begin{align}
\twoinfnorm{\Vector{X}_{\star}} &\leq \twoinfnorm{\Vector{U}_{\star}}\opnorm{\Vector{\Sigma}_{\star}^{\frac{1}{2}}} \leq \sqrt{\frac{\mu r\sigma_{\max}}{d_{1}}}. \label{eqn:Xstar-bound}
\end{align}
Under suitable $C_3$ and $C_4$ in Eq. \eqref{eqn-assumption-p-s}, combining these inequalities with Lemma \ref{lemma-incoherence} and Eq. \eqref{eqn-assumption-p-s} yields
\begin{align}
\twoinfnorm{\Vector{X}_{k}} &\leq \frac{17}{16}\sqrt{\frac{\mu r\sigma_{\max}}{d_{1}}}. \label{eqn:X-final}
\end{align}

On the other hand, the triangle inequality gives
\begin{align}
\twoinfnorm{\Vector{X}_{k}\Vector{Q}_{k} - \Vector{X}_{\star}} &\leq \twoinfnorm{\Vector{X}_{k}}\opnorm{\Vector{Q}_{k}} + \twoinfnorm{\Vector{X}_{\star}}. \label{eqn:XQ-decomp}
\end{align}
From induction hypothesis \ref{induction-bam-exist}, we obtain
\begin{align}
\opnorm{\Vector{Q}_{k}} &\leq \opnorm{\Vector{O}_{k}} + \opnorm{\Vector{Q}_{k} - \Vector{O}_{k}} \leq 1 + \frac{1}{400}, \label{eqn:Q-bound}
\end{align}
and consequently
\begin{align}
\twoinfnorm{\Vector{X}_{k}\Vector{Q}_{k} - \Vector{X}_{\star}} &\leq \frac{401}{400}\frac{17}{16}\sqrt{\frac{\mu r\sigma_{\max}}{d_{1}}} + \sqrt{\frac{\mu r\sigma_{\max}}{d_{1}}} \nonumber \\
&\leq \frac{5}{2}\sqrt{\frac{\mu r\sigma_{\max}}{d_{1}}}. \label{eqn:XQ-final}
\end{align}

Repeating this derivation and noting that
\begin{align}
\opnorm{\Vector{Q}_{k}^{-\top}} &= \opnorm{\Vector{Q}_{k}^{-1}} = \frac{1}{\sigma_{\min}(\Vector{Q_{k}})}, \label{eqn:Qinv-norm} \\
\sigma_{\min}(\Vector{Q}_{k}) &\geq \sigma_{\min}(\Vector{O}_{k}) - \opnorm{\Vector{Q}_{k} - \Vector{O}_{k}} \geq 1 - \frac{1}{400}, \label{eqn:Q-sigma-min}
\end{align}
we obtain the corresponding upper bounds for $\twoinfnorm{\Vector{Y}_{k}}$ and $\twoinfnorm{\Vector{Y}_{k}\Vector{Q}_{k}^{-\top} - \Vector{Y}_{\star}}$.

\subsection{Proof of Lemma \ref{lemma-small-balance}} \label{Apend:lemma-small-balance}

Let $\Vector{B}_{k} \triangleq \Vector{X}_{k}^\top\Vector{X}_{k} - \Vector{Y}_{k}^\top\Vector{Y}_{k}$. From the iteration formulas \eqref{eqn-gd-mc_1} and \eqref{eqn-gd-mc_2}, we have
\begin{align}
 &\Vector{X}_{k+1}^\top \Vector{X}_{k+1} = \Vector{X}_{k}^\top\Vector{X}_{k} + s^{2}\nabla_{\Vector{X}}f(\Vector{X}_{k}, \Vector{Y}_{k})^\top\nabla_{\Vector{X}}f(\Vector{X}_{k}, \Vector{Y}_{k}) \nonumber\\
 & \qquad- s(\Vector{X}_{k}^\top\nabla_{\Vector{X}}f(\Vector{X}_{k}, \Vector{Y}_{k}) + \nabla_{\Vector{X}}f(\Vector{X}_{k}, \Vector{Y}_{k})^\top\Vector{X}_{k}), \\
&\Vector{Y}_{k+1}^\top \Vector{Y}_{k+1} = \Vector{Y}_{k}^\top\Vector{Y}_{k} + s^{2}\nabla_{\Vector{Y}}f(\Vector{X}_{k}, \Vector{Y}_{k})^\top\nabla_{\Vector{Y}}f(\Vector{X}_{k}, \Vector{Y}_{k}) \nonumber\\
& \qquad- s\left(\Vector{Y}_{k}^\top\nabla_{\Vector{Y}}f(\Vector{X}_{k}, \Vector{Y}_{k}) + \nabla_{\Vector{Y}}f(\Vector{X}_{k}, \Vector{Y}_{k})^\top\Vector{Y}_{k}\right).
\end{align}

Thus, the relationship between $\Vector{B}_{k+1}$ and $\Vector{B}_{k}$ is
\begin{align}
\Vector{B}_{k+1} = \Vector{B}_{k} - s \Vector{C}_{k} + s^{2}\Vector{D}_{k},
\end{align}
where
\begin{align}
\Vector{C}_{k} &= \Vector{X}_{k}^\top\nabla_{\Vector{X}}f(\Vector{X}_{k}, \Vector{Y}_{k}) + \nabla_{\Vector{X}}f(\Vector{X}_{k}, \Vector{Y}_{k})^\top\Vector{X}_{k} \nonumber \\
&\quad + \Vector{Y}_{k}^\top\nabla_{\Vector{Y}}f(\Vector{X}_{k}, \Vector{Y}_{k}) + \nabla_{\Vector{Y}}f(\Vector{X}_{k}, \Vector{Y}_{k})^\top\Vector{Y}_{k}, \\
\Vector{D}_{k} &= \nabla_{\Vector{X}}f(\Vector{X}_{k}, \Vector{Y}_{k})^\top\nabla_{\Vector{X}}f(\Vector{X}_{k}, \Vector{Y}_{k}) \nonumber \\
&\quad + \nabla_{\Vector{Y}}f(\Vector{X}_{k}, \Vector{Y}_{k})^\top\nabla_{\Vector{Y}}f(\Vector{X}_{k}, \Vector{Y}_{k}).
\end{align}

Substituting $\nabla f(\Vector{X}_{k}, \Vector{Y}_{k})$ into $\Vector{C}_{k}$ verifies that $\Vector{C}_{k}\equiv0$. 

By the triangle inequality, we obtain
\begin{align}
\fnorm{\Vector{D}_{k}} &\leq \fnorm{p^{-1}\mathcal{P}_{\Omega}(\Vector{X}_{k}\Vector{Y}_{k}^\top- \Vector{M}_{\star})\Vector{Y}_{k}}^{2} \nonumber \\
&\quad + \fnorm{p^{-1}\mathcal{P}_{\Omega}(\Vector{X}_{k}\Vector{Y}_{k}^\top- \Vector{M}_{\star})^\top\Vector{X}_{k}}^{2}.
\end{align}

Note that
\begin{align}
&\fnorm{p^{-1}\mathcal{P}_{\Omega}(\Vector{X}_{k}\Vector{Y}_{k}^\top- \Vector{M}_{\star})\Vector{Y}_{k}}^{2} \nonumber \\
&\leq 2\underbrace{\fnorm{\left(p^{-1}\mathcal{P}_{\Omega} - \mathcal{I}\right)(\Vector{X}_{k}\Vector{Y}_{k}^\top- \Vector{M}_{\star})\Vector{Y}_{k}}^{2}}_{\gamma_{1}} \nonumber \\
&\quad + 2\underbrace{\fnorm{(\Vector{X}_{k}\Vector{Y}_{k}^\top- \Vector{M}_{\star})\Vector{Y}_{k}}^{2}}_{\gamma_{2}}.
\end{align}

For $\gamma_{1}$, we have
\begin{align}
\sqrt{\gamma_{1}} &=\fnorm{\Vector{\Lambda}_k}
= \left\langle \Vector{\Lambda}_k, \widehat{\Vector{X}}_k\right\rangle,
\end{align}
where
\begin{align}
\Vector{\Lambda}_k &= \left(p^{-1}\mathcal{P}_{\Omega} - \mathcal{I}\right)\Big(\big(\Vector{X}_{k}\Vector{Q}_{k}\big)\left(\Vector{Y}_{k}\Vector{Q}_{k}^{-\top}\right)^\top- \Vector{M}_{\star}\Big) \nonumber\\
&\qquad \cdot \Vector{Y}_{k}\Vector{Q}_{k}^{-\top}\left(\Vector{Q}_{k}^\top\Vector{Q}_{k}\right),\\
\widehat{\Vector{X}}_k &= \frac{\Vector{\Lambda}_k}{\fnorm{\Vector{\Lambda}_k}}.
\end{align}
So $\|\widehat{\Vector{X}}_k\|_{\mathrm{F}}=1$. For convenience, let
\begin{align}
\overline{\Vector{X}}_k&=\Vector{X}_{k}\Vector{Q}_{k}, \quad \overline{\Vector{Y}}_k=\Vector{Y}_{k}\Vector{Q}_{k}^{-\top}, \quad \Vector{\Gamma}_k=\Vector{Q}_{k}^\top\Vector{Q}_{k}, \nonumber \\
\Matrix{\Pi}^k_{\Vector{X}} &=\overline{\Vector{X}}_k-\Vector{X}_{\star}, \quad \Matrix{\Pi}^k_{\Vector{Y}} = \overline{\Vector{Y}}_k-\Vector{Y}_{\star}.
\end{align}

From Hypothesis \ref{hypothesis-induction}\ref{induction-bam-exist}, we have
\begin{align}
\opnorm{\Vector{\Gamma}_k - \Vector{I}_{r}} &\leq \opnorm{\Vector{Q}_{k}^\top\Vector{Q}_{k} - \Vector{Q}_{k}^\top\Vector{O}_{k}} + \opnorm{\Vector{Q}_{k}^\top\Vector{O}_{k} - \Vector{O}_{k}^\top\Vector{O}_{k}} \nonumber \\
&\leq \frac{3}{400}.
\end{align}
Thus, we get
\begin{align} \label{eqn-op-gamma}
\opnorm{\Vector{\Gamma}_k} \leq 1 + \frac{3}{400} \leq \frac{3}{2}.
\end{align}

By using the fact that $\overline{\Vector{X}}_k \overline{\Vector{Y}}_k^\top- \Vector{M}_{\star}=\overline{\Vector{X}}_k(\Matrix{\Pi}^k_{\Vector{Y}})^\top+\Matrix{\Pi}^k_{\Vector{X}}\Vector{Y}_{\star}^\top$, we decompose $\sqrt{\gamma_{1}}$ as follows
{\begin{align} 
\sqrt{\gamma_{1}} &= \left\langle\left(p^{-1}\mathcal{P}_{\Omega} - \mathcal{I}\right)\left(\overline{\Vector{X}}_k \overline{\Vector{Y}}_k^\top- \Vector{M}_{\star}\right)\Vector{Y}\Vector{\Gamma}_k, \widehat{\Vector{X}}_k\right\rangle \nonumber \\
&\leq \underbrace{\left\vert\left\langle\left(p^{-1}\mathcal{P}_{\Omega} - \mathcal{I}\right)\left(\Matrix{\Pi}^k_{\Vector{X}}\Vector{Y}_{\star}^\top\right), \widehat{\Vector{X}}_k\Vector{\Gamma}_k\Vector{Y}_{\star}^\top\right\vert\right\rangle}_{\gamma_{11}} \nonumber \\
&\quad + \underbrace{\left\vert\left\langle\left(p^{-1}\mathcal{P}_{\Omega} - \mathcal{I}\right)\left(\Matrix{\Pi}^k_{\Vector{X}}\Vector{Y}_{\star}^\top\right), \widehat{\Vector{X}}_k\Vector{\Gamma}_k(\Matrix{\Pi}^k_{\Vector{Y}})^\top\right\vert\right\rangle}_{\gamma_{12}} \nonumber \\
&\quad + \underbrace{\left\vert\left\langle\left(p^{-1}\mathcal{P}_{\Omega} - \mathcal{I}\right)\left(\overline{\Vector{X}}_k(\Matrix{\Pi}^k_{\Vector{Y}})^\top\right), \widehat{\Vector{X}}_k\Vector{\Gamma}_k\overline{\Vector{Y}}_k^\top\right\vert\right\rangle}_{\gamma_{13}}.
\end{align}

From Lemma \ref{lemma-rip-subspace}, we have
\begin{align} \label{eqn-gamma-11}
\gamma_{11} &\leq C_{1}\sqrt{\frac{\mu r\log d_{1}}{pd_{2}}}\fnorm{\Matrix{\Pi}^k_{\Vector{X}}\Vector{Y}_{\star}^\top}\fnorm{\widehat{\Vector{X}}_k\Vector{\Gamma}_k\Vector{Y}_{\star}^\top} \nonumber \\
&\leq C_{1}\sqrt{\frac{\mu r\log d_{1}}{pd_{2}}}\opnorm{\Vector{Y}_{\star}}^{2}\opnorm{\Vector{\Gamma}_k}\fnorm{\Matrix{\Pi}^k_{\Vector{X}}} \nonumber \\
&\leq \frac{3C_{1}\sigma_{\max}}{2}\sqrt{\frac{\mu r\log d_{1}}{pd_{2}}}\fnorm{\Matrix{\Pi}^k_{\Vector{X}}},
\end{align}
where the last inequality follows from \eqref{eqn-op-gamma}. From Lemma \ref{lemma-rip-all-space}, we obtain
\begin{align} \label{eqn-gamma-12}
\gamma_{12} &\leq C_{2}\sqrt{\frac{d_{1}}{p}}\twoinfnorm{\Matrix{\Pi}^k_{\Vector{X}}}\fnorm{\widehat{\Vector{X}}_k\Vector{\Gamma}_k}\fnorm{\Matrix{\Pi}^k_{\Vector{Y}}}\twoinfnorm{\Vector{Y}_{\star}} \nonumber \\
&\leq \frac{15C_{2}\mu r\sigma_{\max}}{4\sqrt{pd_{2}}}\fnorm{\Matrix{\Pi}^k_{\Vector{Y}}},
\end{align}
where the second inequality follows from Lemma \ref{coro-bounded-2-inf-norm} and \eqref{eqn-op-gamma}. Similarly for $\gamma_{13}$, we get
\begin{align}
\gamma_{13} &\leq C_{2}\sqrt{\frac{d_{1}}{p}}\twoinfnorm{\Vector{X}} \fnorm{\widehat{\Vector{X}}_k\Vector{\Gamma}_k}\fnorm{\Matrix{\Pi}^k_{\Vector{Y}}}\twoinfnorm{\overline{\Vector{Y}}_k} \nonumber \\
&\leq \frac{27C_{2}\mu r\sigma_{\max}}{8\sqrt{pd_{2}}}\fnorm{\Matrix{\Pi}^k_{\Vector{Y}}}.
\label{eqn-gamma-13}
\end{align}

Combining inequalities \eqref{eqn-gamma-11}, \eqref{eqn-gamma-12}, and \eqref{eqn-gamma-13} yields
\begin{align}
\gamma_{1} &\leq \Bigg(\frac{3C_{1}\sigma_{\max}}{2}\sqrt{\frac{\mu r\log d_{1}}{pd_{2}}}\fnorm{\Matrix{\Pi}^k_{\Vector{X}}} \nonumber \\
&\qquad \qquad \qquad \qquad \quad + \frac{57C_{2}\mu r\sigma_{\max}}{8\sqrt{pd_{2}}}\fnorm{\Matrix{\Pi}^k_{\Vector{Y}}}\Bigg)^{2} \nonumber\\
&\leq \frac{9C_{1}^{2}\sigma_{\max}^{2}\mu r\log d_{1}}{2pd_{2}}\fnorm{\Matrix{\Pi}^k_{\Vector{X}}}^{2} + \frac{57^{2}C_{2}^{2}\mu^{2} r^{2}\sigma_{\max}^{2}}{32pd_{2}}\fnorm{\Matrix{\Pi}^k_{\Vector{Y}}}^{2}.
\end{align}

Thus, from the assumption on $p$ in \eqref{eqn-assumption-p-s}, we have
\begin{align}
\gamma_{1} \leq \sigma_{\max}^{2}\left(\fnorm{\Matrix{\Pi}^k_{\Vector{X}}}^{2} + \fnorm{\Matrix{\Pi}^k_{\Vector{Y}}}^{2}\right).
\end{align}

From the definition of $\Vector{Y}_{\star}$, we have $\opnorm{\Vector{Y}_{\star}}=\sqrt{\sigma_{\max}}$. From Hypothesis \ref{hypothesis-induction}\ref{induction-op-norm}, we get
\begin{align}
\opnorm{\Vector{Y}_{k}} &\leq \opnorm{\Vector{Y}_{k} - \Vector{Y}_{\star}} + \opnorm{\Vector{Y}_{\star}} \nonumber \\
&\leq \opnorm{\Vector{F}_{k} - \Vector{F}_{\star}} + \opnorm{\Vector{Y}_{\star}} \leq \frac{5\sqrt{\sigma_{\max}}}{4}.
\end{align}
Thus we have
\begin{align}
\opnorm{\overline{\Vector{X}}_k} \leq \opnorm{\Vector{Y}_{k}}\opnorm{\Vector{Q}_{k}^\top} \leq 2\sqrt{\sigma_{\max}}.
\end{align}
Similarly, $\opnorm{X}\leq2\sqrt{\sigma_{\max}}$. For $\gamma_{2}$, we have
\begin{align}
\gamma_{2} &= \fnorm{\left(\overline{\Vector{X}}_k\overline{\Vector{Y}}_k^\top- \Vector{M}_{\star}\right)\overline{\Vector{Y}}_k\Vector{\Gamma}_k}^{2} \nonumber \\
&\leq \fnorm{\overline{\Vector{X}}_k (\Matrix{\Pi}^k_{\Vector{Y}})^\top+ \Matrix{\Pi}^k_{\Vector{X}}\Vector{Y}_{\star}^\top}^{2}\opnorm{\overline{\Vector{Y}}_k}^{2}\opnorm{\Vector{\Gamma}_k}^{2} \nonumber \\
&\leq 36\sigma_{\max}^{2}\left(\fnorm{\Matrix{\Pi}^k_{\Vector{X}}}^{2} + \fnorm{\Matrix{\Pi}^k_{\Vector{Y}}}^{2}\right).
\end{align}

Therefore we obtain
\begin{multline}
\fnorm{p^{-1}\mathcal{P}_{\Omega}(\Vector{X}_{k}\Vector{Y}_{k}^\top- \Vector{M}_{\star})\Vector{Y}_{k}}^{2} \\\leq 37\sigma_{\max}^{2}\left(\fnorm{\Matrix{\Pi}^k_{\Vector{X}}}^{2} + \fnorm{\Matrix{\Pi}^k_{\Vector{Y}}}^{2}\right), 
\end{multline}
\begin{multline}
\fnorm{p^{-1}\mathcal{P}_{\Omega}(\Vector{X}_{k}\Vector{Y}_{k}^\top- \Vector{M}_{\star})^\top\Vector{X}_{k}}^{2} \\\leq 37\sigma_{\max}^{2}\left(\fnorm{\Matrix{\Pi}^k_{\Vector{X}}}^{2} + \fnorm{\Matrix{\Pi}^k_{\Vector{Y}}}^{2}\right).
\end{multline}

Thus, we have
\begin{align}
\fnorm{\Vector{B}_{k}} &\leq s^{2}\sum_{t=0}^{k-1}\fnorm{\Vector{D}_{t}} \nonumber \\
&\leq 74s^{2}\sigma_{\max}^{2}\sum_{t=0}^{k-1}\left(1 - \frac{s\sigma_{\min}}{100}\right)^{2t}\dist{\Vector{F}_{0}}{\Vector{F}_{\star}}^{2} \nonumber \\
&\leq 7400\kappa s\sigma_{\max}\dist{\Vector{F}_{0}}{\Vector{F}_{\star}}^{2} \leq \frac{s\sigma_{\min}^{2}}{10^{2}\kappa},
\end{align}
where the first inequality holds because the spectral initialization leads to zero initial balancing term $\Vector{B}_{0} = \Vector{X}_{0}^\top\Vector{X}_{0} - \Vector{Y}_{0}^\top\Vector{Y}_{0} = \Vector{\Sigma}_{0} - \Vector{\Sigma}_{0} = 0$, 
and the last inequality follows from \eqref{eqn-f-norm-initial}. Therefore, the conclusion holds.

\subsection{Proof of Lemma \ref{lemma-linear-convergence}} \label{append:lemma-linear-convergence}

By the definition of $\dist{\Vector{F}_{k+1}}{\Vector{F}_{\star}}$, we have
\begin{multline}
\dist{\Vector{F}_{k+1}}{\Vector{F}_{\star}} \\ \leq \fnorm{\Vector{X}_{k+1}\Vector{Q}_{k} - \Vector{X}_{\star}}^{2} + \fnorm{\Vector{Y}_{k+1}\Vector{Q}_{k}^{-\top} - \Vector{Y}_{\star}}^{2}.
\end{multline}
From the update rules \eqref{eqn-gd-mc_1} and \eqref{eqn-gd-mc_2}, it follows that
\begin{align}
&\fnorm{\Vector{X}_{k+1}\Vector{Q}_{k} - \Vector{X}_{\star}}^{2} \nonumber\\
= & \fnorm{\left(\Vector{X}_{k} - \frac{s}{p}\mathcal{P}_{\Omega}\left(\Vector{X}_{k}\Vector{Y}_{k}^\top- \Vector{M}_{\star}\right)\Vector{Y}_{k}\right)\Vector{Q}_{k} - \Vector{X}_{\star}}^{2} \nonumber \\
= & \Big{\|}\Vector{X}_{k}\Vector{Q}_{k} - \Vector{X}_{\star} - s\big(\overline{\Vector{X}}_k \overline{\Vector{Y}}_k^\top- \Vector{M}_{\star}\big)\overline{\Vector{Y}}_k \Vector{\Gamma}_k \nonumber\\
& \qquad - s\big(p^{-1}\mathcal{P}_{\Omega} - \mathcal{I}\big)\big(\overline{\Vector{X}}_k \overline{\Vector{Y}}_k^\top-\Vector{M}_{\star} \big) \overline{\Vector{Y}}_k \Vector{\Gamma}_k \Big{\|}_{\mathrm{F}}^{2},
\end{align}
where
\begin{align*}
\overline{\Vector{X}}_k&=\Vector{X}_{k}\Vector{Q}_{k}, \quad \overline{\Vector{Y}}_k=\Vector{Y}_{k}\Vector{Q}_{k}^{-\top}, \quad \Vector{\Gamma}_k=\Vector{Q}_{k}^\top\Vector{Q}_{k}, \nonumber\\
\Matrix{\Delta}^k_{\Vector{X}}&= \overline{\Vector{X}}_k-\Vector{X}_{\star}, \quad\Matrix{\Delta}^k_{\Vector{Y}}= \overline{\Vector{Y}}_k-\Vector{Y}_{\star}.
\end{align*}
Using these notations, we derive 
\begin{align*}
&\fnorm{\Vector{X}_{k+1}\Vector{Q}_{k} - \Vector{X}_{\star}}^{2} \nonumber\\
= & \fnorm{\Matrix{\Delta}_{\Vector{X}} - s\left(\overline{\Vector{X}}_k\overline{\Vector{Y}}_k^\top- \Vector{M}_{\star}\right)\overline{\Vector{Y}}_k\Vector{\Gamma}_k}^{2} \\
& - 2s \big{\langle} \Matrix{\Delta}_{\Vector{X}} - s\left(\overline{\Vector{X}}_k\overline{\Vector{Y}}_k^\top- \Vector{M}_{\star}\right)\overline{\Vector{Y}}_k\Vector{\Gamma}_k, \nonumber \\
&\qquad \qquad \left(p^{-1}\mathcal{P}_{\Omega} - \mathcal{I}\right)\left(\overline{\Vector{X}}_k\overline{\Vector{Y}}_k^\top- \Vector{M}_{\star}\right)\overline{\Vector{Y}}_k\Vector{\Gamma}_k \big{\rangle} \\
& + s^{2}\fnorm{\left(p^{-1}\mathcal{P}_{\Omega} - \mathcal{I}\right)\left(\overline{\Vector{X}}_k\overline{\Vector{Y}}_k^\top- \Vector{M}_{\star}\right)\overline{\Vector{Y}}_k\Vector{\Gamma}_k}^{2}.
\end{align*}
Noting that
\begin{multline*}
\overline{\Vector{X}}_k\overline{\Vector{Y}}_k^\top- \Vector{M}_{\star} = \Matrix{\Delta}_{\Vector{X}}\overline{\Vector{Y}}_k^\top+ \Vector{X}_{\star}\Matrix{\Delta}_{\Vector{Y}}^\top= \Matrix{\Delta}_{\Vector{X}}\Vector{Y}_{\star}^\top+ \overline{\Vector{X}}_k\Matrix{\Delta}_{\Vector{Y}}^\top,
\end{multline*}
we decompose the expression into Eq. \eqref{eq:x_update_alpha}. Similarly, for the $\Vector{Y}$-update, we have Eq. \eqref{eq:y_update_beta}.
\begin{table*}
\begin{align} \label{eq:x_update_alpha}
\fnorm{\Vector{X}_{k+1}\Vector{Q}_{k} - \Vector{X}_{\star}}^{2} 
= & \underbrace{\fnorm{\Matrix{\Delta}_{\Vector{X}} - s\left(\Vector{X}\Vector{Y}^\top- \Vector{M}_{\star}\right)\Vector{Y}\Vector{\Gamma}_k}^{2}}_{\alpha_{1}}  - 2s\underbrace{\left\langle\Matrix{\Delta}_{\Vector{X}}\left(\Vector{I}_{r} - s \Vector{Y}^\top\Vector{Y}\Vector{\Gamma}_k\right), \left(p^{-1}\mathcal{P}_{\Omega} - \mathcal{I}\right)\left(\Vector{X}\Vector{Y}^\top- \Vector{M}_{\star}\right)\Vector{Y}\Vector{\Gamma}_k\right\rangle}_{\alpha_{2}} \nonumber\\
& + 2s^{2}\underbrace{\left\langle \Vector{X}_{\star}\Matrix{\Delta}_{\Vector{Y}}^\top\Vector{Y}\Vector{\Gamma}_k, \left(p^{-1}\mathcal{P}_{\Omega} - \mathcal{I}\right)\left(\Vector{X}\Vector{Y}^\top- \Vector{M}_{\star}\right)\Vector{Y}\Vector{\Gamma}_k\right\rangle}_{\alpha_{3}}  + s^{2}\underbrace{\fnorm{\left(p^{-1}\mathcal{P}_{\Omega} - \mathcal{I}\right)\left(\Vector{X}\Vector{Y}^\top- \Vector{M}_{\star}\right)\Vector{Y}\Vector{\Gamma}_k}^{2}}_{\alpha_{4}}.
\end{align}
\vspace{-0.1in}
\begin{align} \label{eq:y_update_beta}
&\fnorm{\Vector{Y}_{k+1}\Vector{Q}_{k}^{-\top} - \Vector{Y}_{\star}}^{2} 
=  \underbrace{\fnorm{\Matrix{\Delta}_{\Vector{Y}} - s\left(\Vector{X}\Vector{Y}^\top- \Vector{M}_{\star}\right)^\top\Vector{X}\Vector{\Gamma}_k^{-1}}^{2}}_{\beta_{1}} - 2s\underbrace{\left\langle\Matrix{\Delta}_{\Vector{Y}}\left(\Vector{I}_{r} - s \Vector{X}^\top\Vector{X}\Vector{\Gamma}_k^{-1}\right), \left(p^{-1}\mathcal{P}_{\Omega} - \mathcal{I}\right)\left(\Vector{X}\Vector{Y}^\top- \Vector{M}_{\star}\right)^{-\top}\Vector{X}\Vector{\Gamma}_k^{-1}\right\rangle}_{\beta_{2}} \nonumber\\
& \qquad \qquad \qquad \qquad + 2s^{2}\underbrace{\left\langle \Vector{Y}_{\star}\Matrix{\Delta}_{\Vector{X}}^\top\Vector{X}\Vector{\Gamma}_k^{-1}, \left(p^{-1}\mathcal{P}_{\Omega} - \mathcal{I}\right)\left(\Vector{X}\Vector{Y}^\top- \Vector{M}_{\star}\right)^\top\Vector{X}\Vector{\Gamma}_k^{-1}\right\rangle}_{\beta_{3}} + s^{2}\underbrace{\fnorm{\left(p^{-1}\mathcal{P}_{\Omega} - \mathcal{I}\right)\left(\Vector{X}\Vector{Y}^\top- \Vector{M}_{\star}\right)^\top\Vector{X}\Vector{\Gamma}_k^{-1}}^{2}}_{\beta_{4}}.
\end{align}
\vspace{-0.1in}
\begin{multline}
\label{eq-alpha-2}
|\alpha_{2}| = \left\vert\left\langle\Matrix{\Delta}_{\Vector{X}}\left(\Vector{I}_{r} - s \overline{\Vector{Y}}_k^{\top}\overline{\Vector{Y}}_k\Vector{\Gamma}\right), \left(p^{-1}\mathcal{P}_{\Omega} - \mathcal{I}\right)\left(\Matrix{\Delta}_{\Vector{X}}\Vector{Y}_{\star}^{\top} + \overline{\Vector{X}}_k\Matrix{\Delta}_{\Vector{Y}}^{\top}\right)\overline{\Vector{Y}}_k\Vector{\Gamma}\right\rangle\right\vert
\leq \underbrace{\left\vert\left\langle\Matrix{\Delta}_{\Vector{X}}\left(\Vector{I}_{r} - s \overline{\Vector{Y}}_k^{\top}\overline{\Vector{Y}}_k\Vector{\Gamma}\right), \left(p^{-1}\mathcal{P}_{\Omega} - \mathcal{I}\right)\left(\Matrix{\Delta}_{\Vector{X}}\Vector{Y}_{\star}^{\top} \right)\Vector{Y}_{\star}\Vector{\Gamma}\right\rangle\right\vert}_{\alpha_{21}} \\
+ \underbrace{\left\vert\left\langle\Matrix{\Delta}_{\Vector{X}}\left(\Vector{I}_{r} - s \overline{\Vector{Y}}_k^{\top}\overline{\Vector{Y}}_k\Vector{\Gamma}\right), \left(p^{-1}\mathcal{P}_{\Omega} - \mathcal{I}\right)\left(\Matrix{\Delta}_{\Vector{X}}\Vector{Y}_{\star}^{\top} \right)\Matrix{\Delta}_{\Vector{Y}}\Vector{\Gamma}\right\rangle\right\vert}_{\alpha_{22}} + \underbrace{\left\vert\left\langle\Matrix{\Delta}_{\Vector{X}}\left(\Vector{I}_{r} - s \overline{\Vector{Y}}_k^{\top}\overline{\Vector{Y}}_k\Vector{\Gamma}\right), \left(p^{-1}\mathcal{P}_{\Omega} - \mathcal{I}\right)\left(\overline{\Vector{X}}_k\Matrix{\Delta}_{\Vector{Y}}^{\top}\right)\overline{\Vector{Y}}_k\Vector{\Gamma}\right\rangle\right\vert}_{\alpha_{23}}.
\end{multline}
\vspace{-0.1in}
\end{table*}



By Lemma \ref{lemma-initial-property} and induction hypotheses \ref{induction-linear-convergence}, \ref{induction-bam-exist}, there exists sufficiently large $C_{1}$ such that when $p\geq\frac{\mu r^{2}\kappa^{10}\log d_{1}}{d_{2}}$, the conditions of Lemma \ref{lemma-mf-linear-convergence} hold with high probability. Thus for $0<s\leq\frac{1}{24\sigma_{\max}}$, we have
\begin{equation}
\label{eqn-alpha-1}
\alpha_{1} + \beta_{1} \leq \left(1 - \frac{s\sigma_{\min}}{24}\right)\dist{\Vector{F}_{k}}{F_{\star}}^{2}.
\end{equation}
For $\alpha_{2}$, it can be split as \eqref{eq-alpha-2} shows. 

By Lemma \ref{lemma-rip-subspace}, it holds that
\begin{align*}
\alpha_{21} \leq & C_{1}\sqrt{\frac{\mu r\log d_{1}}{d_{2}}}\fnorm{\Matrix{\Delta}_{\Vector{X}}\Vector{Y}_{\star}}\fnorm{\Matrix{\Delta}_{\Vector{X}}\left(\Vector{I}_{r} - s \overline{\Vector{Y}}_k^{\top}\overline{\Vector{Y}}_k\Vector{\Gamma}\right)\Vector{\Gamma}\Vector{Y}_{\star}} \\
\leq & C_{1}\sqrt{\frac{\mu r\log d_{1}}{d_{2}}}\opnorm{\Vector{Y}_{\star}}^{2}\opnorm{\Vector{\Gamma}}\opnorm{\Vector{I}_{r} - s \overline{\Vector{Y}}_k^{\top}\overline{\Vector{Y}}_k\Vector{\Gamma}}\fnorm{\Matrix{\Delta}_{\Vector{X}}}^{2}.
\end{align*}
By induction hypothesis \ref{induction-op-norm}, we have 
\begin{multline*}
\opnorm{\Vector{Y}_{k}} \leq \opnorm{\Vector{Y}_{k} - \Vector{Y}_{\star}} + \opnorm{\Vector{Y}_{\star}} \\
\leq \opnorm{\Vector{F}_{k} - \Vector{F}_{\star}} + \opnorm{\Vector{Y}_{\star}} \leq \frac{5\sqrt{\sigma_{\max}}}{4}~.
\end{multline*}
Hence
\begin{equation}
\label{eqn-bounded-op-norm}
\opnorm{\overline{\Vector{Y}}_k} \leq \opnorm{\Vector{Y}_{k}}\opnorm{\Vector{Q}_{k}^{\top}} \leq 2\sqrt{\sigma_{\max}}~.
\end{equation}
Similarly we can know $\opnorm{\overline{\Vector{X}}_k}\leq2\sqrt{\sigma_{\max}}$.
When $0<s\leq\frac{8}{27\sigma_{\max}}$, we have the upper bound of $\alpha_{21}$ by \eqref{eqn-bounded-op-norm}:
\begin{equation}
\label{eqn-alpha-21}
\alpha_{21} \leq \frac{3C_{1}}{2}\sigma_{\max}\sqrt{\frac{\mu r\log d_{1}}{pd_{2}}}\fnorm{\Matrix{\Delta}_{\Vector{X}}}^{2}.
\end{equation}
By Lemma \ref{lemma-rip-all-space}, we have the following inequality for $\alpha_{22}$
\begin{align}
\alpha_{22} \leq & C_{2}\sqrt{\frac{d_{1}}{p}}\fnorm{\Matrix{\Delta}_{\Vector{X}}}\twoinfnorm{\Matrix{\Delta}_{\Vector{X}}}\twoinfnorm{\Vector{Y}_{\star}}\fnorm{\Matrix{\Delta}_{\Vector{Y}}} \nonumber\\
\leq & \frac{5C_{2}\mu r\sigma_{\max}}{2\sqrt{pd_{2}}}\fnorm{\Matrix{\Delta}_{\Vector{X}}}\fnorm{\Matrix{\Delta}_{\Vector{Y}}}, \label{eqn-alpha-22}
\end{align}
The second inequality is due to Lemma \ref{coro-bounded-2-inf-norm} and $\mu$-incoherence of $\Vector{M}_{\star}$. Utilizing Lemma \ref{lemma-rip-all-space} and Lemma \ref{coro-bounded-2-inf-norm}, for $\alpha_{23}$, we have 
\begin{align}
\alpha_{23} \leq & C_{2}\sqrt{\frac{d_{1}}{p}}\twoinfnorm{\overline{\Vector{X}}_k}\fnorm{\Matrix{\Delta}_{\Vector{X}}}\fnorm{\Matrix{\Delta}_{\Vector{Y}}} \nonumber\\
& \cdot \twoinfnorm{\overline{\Vector{Y}}_k\Vector{\Gamma}\left(\Vector{I}_{r} - s\overline{\Vector{Y}}_k^{\top}\overline{\Vector{Y}}_k\Vector{\Gamma}\right)} \nonumber\\
\leq & \frac{27C_{2}\mu r\sigma_{\max}}{8\sqrt{pd_{2}}}\fnorm{\Matrix{\Delta}_{\Vector{X}}}\fnorm{\Matrix{\Delta}_{\Vector{Y}}}. \label{eqn-alpha-23}
\end{align}
Combining \eqref{eqn-alpha-21}, \eqref{eqn-alpha-22} and \eqref{eqn-alpha-23}, we get
\begin{align*}
\alpha_{2} \leq & \frac{3C_{1}}{2}\sigma_{\max}\sqrt{\frac{\mu r\log d_{1}}{pd_{2}}}\fnorm{\Matrix{\Delta}_{\Vector{X}}}^{2} \\
&+ \frac{47C_{2}\mu r\sigma_{\max}}{8\sqrt{pd_{2}}}\fnorm{\Matrix{\Delta}_{\Vector{X}}}\fnorm{\Matrix{\Delta}_{\Vector{Y}}} \\
\leq & \left(\frac{3C_{1}}{2}\sigma_{\max}\sqrt{\frac{\mu r\log d_{1}}{pd_{2}}} + \frac{47C_{2}\mu r\sigma_{\max}}{18\sqrt{pd_{2}}}\right)\fnorm{\Matrix{\Delta}_{\Vector{X}}}^{2} \\
& + \frac{47C_{2}\mu r\sigma_{\max}}{8\sqrt{pd_{2}}}\fnorm{\Matrix{\Delta}_{\Vector{Y}}}^{2}.
\end{align*}
The upper bound of $\beta_{2}$ can be derived by the same method. Combining the estimation of $\alpha_{2}$ and $\beta_2$, we have
\begin{multline}
\label{eqn-alpha-beta-2}
\alpha_{2} + \beta_{2} \leq \left(\frac{3C_{1}}{2}\sigma_{\max}\sqrt{\frac{\mu r\log d_{1}}{pd_{2}}} + \frac{47C_{2}\mu r\sigma_{\max}}{18\sqrt{pd_{2}}}\right)\\
\times\left(\fnorm{\Matrix{\Delta}_{\Vector{X}}}^{2} + \fnorm{\Matrix{\Delta}_{\Vector{Y}}}^{2}\right).
\end{multline}

Using the similar method to split $\alpha_{3}$, we get
\begin{align*}
\vert\alpha_{3}\vert \leq & \underbrace{\left\vert\left\langle\Vector{X}_{\star}\Matrix{\Delta}_{\Vector{Y}}^{\top}\overline{\Vector{Y}}_k\Vector{\Gamma}^{2}\overline{\Vector{Y}}_k^{\top}, \left(p^{-1}\mathcal{P}_{\Omega}\left(\Vector{X}_{\star}\Matrix{\Delta}_{\Vector{Y}}^{\top}\right)\right)\right\rangle\right\vert}_{\alpha_{31}} \\
& + \underbrace{\left\vert\left\langle\Vector{X}_{\star}\Matrix{\Delta}_{\Vector{Y}}^{\top}\overline{\Vector{Y}}_k\Vector{\Gamma}^{2}\overline{\Vector{Y}}_k^{\top}, \left(p^{-1}\mathcal{P}_{\Omega}\left(\Matrix{\Delta}_{\Vector{X}}\overline{\Vector{Y}}_k^{\top}\right)\right)\right\rangle\right\vert}_{\alpha_{32}}.
\end{align*}
By Lemma \ref{lemma-rip-subspace}, we have
\begin{align}
\alpha_{31} \leq & C_{1}\sqrt{\frac{\mu r\log d_{1}}{pd_{2}}} \fnorm{\Vector{X}_{\star}\Matrix{\Delta}_{\Vector{Y}}^{\top}}\fnorm{\Vector{X}_{\star}\Matrix{\Delta}_{\Vector{Y}}^{\top}\overline{\Vector{Y}}_k\Vector{\Gamma}^{2}\overline{\Vector{Y}}_k^{\top}} \nonumber\\
\leq & C_{1}\sqrt{\frac{\mu r\log d_{1}}{pd_{2}}}\opnorm{\Vector{X}_{\star}}^{2}\opnorm{\Vector{\Gamma}}^{2}\opnorm{\overline{\Vector{Y}}_k}^{2}\fnorm{\Matrix{\Delta}_{\Vector{Y}}}^{2} \nonumber\\
\leq & \frac{81C_{1}\sigma_{\max}^{2}}{16}\sqrt{\frac{\mu r\log d_{1}}{pd_{2}}}\fnorm{\Matrix{\Delta}_{\Vector{Y}}}^{2}, \label{eqn-alpha-31}
\end{align}
The last inequality is due to \eqref{eqn-bounded-op-norm}. According to Lemma \ref{lemma-rip-all-space}, for $\alpha_{32}$ we have
{\small \begin{align} 
\alpha_{32} \leq & C_{2}\sqrt{\frac{d_{1}}{p}}\fnorm{\Matrix{\Delta}_{\Vector{X}}}\twoinfnorm{\Vector{X}_{\star}}\twoinfnorm{\overline{\Vector{Y}}_k}\fnorm{\overline{\Vector{Y}}_k\Vector{\Gamma}\overline{\Vector{Y}}_k^{\top}\Matrix{\Delta}_{\Vector{Y}}} \nonumber\\
\leq & C_{2}\sqrt{\frac{d_{1}}{p}}\fnorm{\Matrix{\Delta}_{\Vector{X}}}\twoinfnorm{\Vector{X}_{\star}}\twoinfnorm{\overline{\Vector{Y}}_k}\opnorm{\Vector{\Gamma}}\opnorm{\overline{\Vector{Y}}_k}^{2}\fnorm{\Matrix{\Delta}_{\Vector{Y}}} \nonumber\\
\leq & \frac{243C_{2}\mu r\sigma_{\max}^{2}}{32\sqrt{pd_{2}}}\fnorm{\Matrix{\Delta}_{\Vector{X}}}\fnorm{\Matrix{\Delta}_{\Vector{Y}}}, \label{eqn-alpha-32}
\end{align}}

\noindent where the last inequality is by Lemma \ref{coro-bounded-2-inf-norm}, \eqref{eqn-bounded-op-norm} and \eqref{eqn-op-gamma}. Repeating the process for $\beta_{3}$ and utilizing mean value inequality, we establish
\begin{multline}
\label{eqn-alpha-beta-3}
\alpha_{3} + \beta_{3} \leq \left(\frac{81C_{1}\sigma_{\max}^{2}}{16}\sqrt{\frac{\mu r\log d_{1}}{pd_{2}}} + \frac{243C_{2}\mu r\sigma_{\max}^{2}}{64\sqrt{pd_{2}}} \right)\\
\times\left(\fnorm{\Matrix{\Delta}_{\Vector{X}}}^{2} + \fnorm{\Matrix{\Delta}_{\Vector{Y}}}^{2}\right).
\end{multline}

Finally using the same method of estimating $\gamma_{1}$ in Lemma \ref{lemma-small-balance}, we have
\begin{align*}
\alpha_{4} \leq & \left(\frac{3C_{1}\sigma_{\max}}{2}\sqrt{\frac{\mu r\log d_{1}}{pd_{2}}}\fnorm{\Matrix{\Delta}_{\Vector{X}}} + \frac{27C_{2}\mu r\sigma_{\max}}{8\sqrt{pd_{2}}}\fnorm{\Matrix{\Delta}_{\Vector{R}}}\right)^{2} \\
\leq & \frac{9C_{1}^{2}\sigma_{\max}^{2}\mu r\log d_{1}}{pd_{2}}\fnorm{\Matrix{\Delta}_{\Vector{X}}}^{2} + \frac{27^{2}C_{2}^{2}\mu^{2} r^{2}\sigma_{\max}^{2}}{32pd_{2}}\fnorm{\Matrix{\Delta}_{\Vector{R}}}^{2}.
\end{align*}
The upper bound of $\beta_{4}$ can also be derived. Combining $\alpha_{4}$ and $\beta_{4}$, we have
\begin{multline}
\label{eqn-alpha-beta-4}
\alpha_{4} + \beta_{4} \leq \left(\frac{9C_{1}^{2}\sigma_{\max}^{2}\mu r\log d_{1}}{pd_{2}} + \frac{27^{2}C_{2}^{2}\mu^{2} r^{2}\sigma_{\max}^{2}}{32pd_{2}}\right)\\
\cdot \left(\fnorm{\Matrix{\Delta}_{\Vector{X}}}^{2} + \fnorm{\Matrix{\Delta}_{\Vector{R}}}^{2}\right).
\end{multline}

Combining \eqref{eqn-alpha-1}, \eqref{eqn-alpha-beta-2}, \eqref{eqn-alpha-beta-3} and \eqref{eqn-alpha-beta-4}, we establish
\begin{multline*}
\fnorm{\Vector{X}_{k+1}\Vector{Q}_{k} - \Vector{X}_{\star}}^{2} + \fnorm{\Vector{Y}_{k+1}\Vector{Q}_{k}^{-\top} - \Vector{Y}_{\star}}^{2} \\
\leq \left(1 - C(p, s)s\sigma_{\min}\right)\left(\fnorm{\Matrix{\Delta}_{\Vector{X}}}^{2} + \fnorm{\Matrix{\Delta}_{\Vector{R}}}^{2}\right),
\end{multline*}
where $C(p, s)$ is a constant depending on $p$ and $s$:
\begin{multline*}
C(p, s) = \frac{1}{24} - \bigg(3C_{1}\kappa\sqrt{\frac{\mu r\log d_{1}}{pd_{2}}} + \frac{47C_{2}\mu r\kappa}{9\sqrt{pd_{2}}}  \\
+ \frac{81C_{1}\kappa s\sigma_{\max}}{8}\sqrt{\frac{\mu r\log d_{1}}{pd_{2}}} + \frac{243C_{2}\mu r\kappa s\sigma_{\max}}{32\sqrt{pd_{2}}}  \\
 + \frac{9C_{1}^{2}\mu r\kappa s\sigma_{\max}\log d_{1}}{pd_{2}} + \frac{27^{2}C_{2}^{2}\mu^{2}r^{2}\kappa s\sigma_{\max}}{32pd_{2}} \bigg).
\end{multline*}
Since $p$ and $s$ satisfy \eqref{eqn-assumption-p-s}, we have \[
C(p, s) \geq \frac{1}{50}~.
\]
Consequently, we get
\begin{align*}
&\dist{\Vector{F}_{k+1}}{\Vector{F}_{\star}}^{2} \\
\leq & \fnorm{\Vector{X}_{k+1}\Vector{Q}_{k} - \Vector{X}_{\star}}^{2} + \fnorm{\Vector{Y}_{k+1}\Vector{Q}_{k}^{-\top} - \Vector{Y}_{\star}}^{2} \\
\leq & \left(1 - \frac{s\sigma_{\min}}{50}\right)\left(\fnorm{\Matrix{\Delta}_{\Vector{X}}}^{2} + \fnorm{\Matrix{\Delta}_{\Vector{R}}}^{2}\right) \\
\leq & \left(1 - \frac{s\sigma_{\min}}{100}\right)^{2}\dist{\Vector{F}_{k}}{\Vector{F}_{\star}}^{2}.
\end{align*}


\end{document}